%% file: neurips_2026.tex
\documentclass{article}

\usepackage[main,final]{neurips_2026}

\usepackage[utf8]{inputenc} 
\usepackage[T1]{fontenc}    
\usepackage{hyperref}       
\usepackage{url}            
\usepackage{booktabs}       
\hypersetup{
  hidelinks,
  pdftitle={NAGO: Noise-Aware Generative Operator via Flow Matching},
  pdfauthor={Guanyu Chen, Shengze Xu, Pengwei Liu, Xingyu Ren, Pengkai Wang, Antonios Armaou, Tieyong Zeng, Dong Ni}
}
\usepackage{xspace}
\usepackage{graphicx}
\usepackage{amsmath}
\usepackage{amssymb}
\usepackage{amsthm}
\usepackage{subcaption}
\usepackage{float}
\usepackage{algorithm}
\usepackage{wrapfig}
\usepackage{algorithmic}
\usepackage{multirow}
\usepackage{nicefrac}       
\usepackage{microtype}      
\usepackage{xcolor}         
\newcommand{\themodel}{\textbf{\texttt{NAGO}}\xspace}
\usepackage{natbib}
\usepackage{makecell}
\usepackage{titletoc}
\newtheorem{theorem}{Theorem}[section]

\newtheorem{definition}[theorem]{Definition}

\newtheorem{remark}[theorem]{Remark}
\newtheorem*{remark*}{\textbf{Remark}}
\usepackage[textsize=tiny]{todonotes}
\usepackage{enumitem}

\title{NAGO: Noise-Aware Generative Operator via Flow Matching}

\author{
  {\bfseries Guanyu Chen\textsuperscript{1}\thanks{These authors contributed equally to this work.} \quad
  Shengze Xu\textsuperscript{2}\footnotemark[1] \quad
  Pengwei Liu\textsuperscript{1} \quad
  Xingyu Ren\textsuperscript{1}} \\
  {\bfseries Pengkai Wang\textsuperscript{3} \quad
  Antonios Armaou\textsuperscript{4} \quad
  Tieyong Zeng\textsuperscript{2} \quad
  Dong Ni\textsuperscript{1}\thanks{Corresponding author: \texttt{dni@zju.edu.cn}.}} \\
  \textsuperscript{1}Zhejiang University \quad
  \textsuperscript{2}The Chinese University of Hong Kong \\
  \textsuperscript{3}The Hong Kong Polytechnic University \quad
  \textsuperscript{4}The Pennsylvania State University
}

\begin{document}

\maketitle

\input{src/0_Abstract}
\input{src/1_Introduction}
\input{src/3_Preliminary}
\input{src/4_Methodology}
\input{src/5_Experiment}
\input{src/6_Conclusion}

\begin{ack}
This work was supported by the National Key Research and Development Project of China under Grant No.~2025YFE0204600 and the Key Program of the National Natural Science Foundation of China under Grant No.~62533019.
The authors declare no competing financial interests.
\end{ack}

\newpage
\bibliographystyle{plainnat}
\bibliography{reference}

\input{src/7_Appendix_1}

\input{src/7_Appendix_2}

\input{src/7_Appendix_3}

\input{src/7_Appendix_4}

\stopcontents[appendix]

\newpage
\input{checklist.tex}

\end{document}

%% file: src/0_Abstract.tex

\begin{abstract}
We introduce \themodel, a noise-aware generative operator for efficient sampling from invariant measures of stochastic partial differential equations (SPDEs) in function space. Our central idea is to treat the source measure as a design variable rather than a fixed Gaussian default. We establish admissibility conditions for well-defined infinite-dimensional flows and show that the finite-band source--target covariance ratio controls the spectral stiffness of the marginal velocity field. This analysis yields an admissibility-constrained spectral-matching rule: choose a source that is sufficiently rough for well-posedness, then minimize covariance mismatch over the resolved Fourier modes. On three representative stochastic-field benchmarks, this source design improves few-step spectral accuracy and several complementary distributional and physical statistics. In a controlled Mat\'ern-field ablation, it reduces the accurate RK4 budget from about 30 steps for Gaussian functional flow matching to 3--5 steps, a $6$--$10\times$ reduction in neural ODE evaluations. These gains are strongest for stationary multiscale fields and depend on the metric, resolution, and guidance protocol.
\end{abstract}

%% file: src/1_Introduction.tex
\section{Introduction}
Stochastic partial differential equations (SPDEs) constitute a fundamental mathematical framework for modeling the evolution of infinite-dimensional random systems under uncertainty and provide a rigorous description of diverse spatiotemporal phenomena, including statistical physics \citep{huang2009introduction} and composite materials \citep{hornung1992applications}. However, in many practical applications, such as geophysical modeling \citep{lindgren2011explicit}, turbulent flow \citep{kida1987kolmogorov, hairer2006ergodicity, carmona1999stochastic}, and thin film growth \citep{hairer2013solving}, the exact evolution of a single realization is often less important than the long-term statistical behavior of the system \citep{Eckmann_2000}. 

A central problem in this context is the characterization of invariant measures \citep{da1996ergodicity, Da_Prato_Zabczyk_2014}, which represent the most compressed yet complete description of long-term stochastic behavior, as discussed in Appendix~\ref{apdx:intro}. 
For many SPDEs of physical and statistical interest, these invariant measures take the form of Gaussian or non-Gaussian probability measures on suitable function spaces, with representative examples including the Ornstein--Uhlenbeck equations \citep{sarkka2019applied} and nonlinear models such as the stochastic Allen--Cahn equation \citep{xu2025interface} and the stochastic Navier--Stokes equation \citep{bensoussan1995stochastic}, which may admit unique ergodic invariant measures under appropriate conditions.
Despite their importance, sampling invariant measures of SPDEs remains challenging due to the infinite-dimensional nature of the dynamics, the absence of closed-form solutions in nonlinear regimes, and the intricate interaction between noise, dissipation, and spatial operators. Classical numerical approaches, including long-time integration of discretized SPDEs \citep{Lord_Powell_Shardlow_2014, Da_Prato_Zabczyk_2014} and Markov chain Monte Carlo (MCMC) sampling in function space \citep{chen2025new}, may approximate invariant measures but are often computationally prohibitive, particularly in high-dimensional or stiff systems exhibiting pronounced multiscale structure.

These challenges have motivated growing interest in data-driven and generative modeling approaches that aim to learn invariant measures directly, bypassing explicit long-time simulation. Recent advances in generative modeling \citep{ho2020denoising, songdenoising, songscore, lipmanflow} and operator-learning frameworks \citep{kovachki2023neural, li2021fourierneuraloperatorparametric} have shown considerable promise: generative models provide expressive distributional approximations, while neural operators enable mappings between functional distributions. 
Recent progress has been made toward generative modeling directly in Hilbert or Banach spaces.
\citet{rahman2022generative, seidman2023variational} proposed GANO and VANO, while
\citet{pidstrigach2024infinite, lim2023score, kerrigan2022diffusion, kerrigan2023functional, franzese2023continuous}
develop the theoretical foundations for diffusion models and flow matching in infinite-dimensional settings. More details about related work are reviewed in Appendix~\ref{sec_related_work}.

Nevertheless, an important gap remains between the theory of functional generative modeling and its practical use for multiscale SPDE data.
Most implementations still adopt standard Gaussian noise as the source distribution by default.
For invariant measures with nontrivial spatial regularity, this choice can be poorly conditioned: the discretized Gaussian noise injects excessive high-frequency energy, forcing the learned velocity field to suppress unresolved modes and thereby making training and inference resolution-dependent.
Existing remedies, such as interpolation schedules \citep{chen2025scale}, Schr{\"o}dinger bridge formulations \citep{shi2023diffusion}, and adaptive numerical solvers \citep{chen2024adaptive}, can improve the transport dynamics or numerical integration, but they do not directly modify the source measure that induces the stiffness.
We address this issue by treating the source measure as a design variable in functional flow matching.
This perspective leads to an admissibility-constrained spectral-matching principle: the source must first be rough enough for function-space well-posedness, and should then be chosen to reduce covariance mismatch over the frequencies resolved by the model.

Based on this principle, we propose \themodel, a noise-designed generative operator for sampling from multiscale SPDE invariant measures. By selecting an admissible Mat\'ern source with a finite-band spectrum aligned to the target, \themodel reduces high-frequency mismatch and yields better-conditioned marginal velocity fields. We support this approach with an infinite-dimensional analysis, a quantitative link between source--target covariance ratios and spectral stiffness, and an operational source-selection rule. Our main contributions are summarized as follows:
\vspace{-0.6em}
\begin{itemize}[leftmargin=1.2em]
    \item We identify a function-space mechanism behind source design: source regularity determines admissibility, while the resolved source--target covariance ratio governs spectral stiffness.

    \item We develop an admissibility-constrained spectral-matching rule for selecting Mat\'ern sources from target samples, including both smoothness and correlation-length selection.

    \item Across three stochastic-field benchmarks, we demonstrate improved few-step generation and complementary distributional statistics. A controlled ablation reduces the accurate RK4 budget by $6$--$10\times$ relative to a Gaussian-source model; a separate post-training benchmark shows an $89.1\times$ raw-throughput ratio to the SNS solver.
\end{itemize}

\begin{figure}[t]
    \centering
    \includegraphics[width=\textwidth]{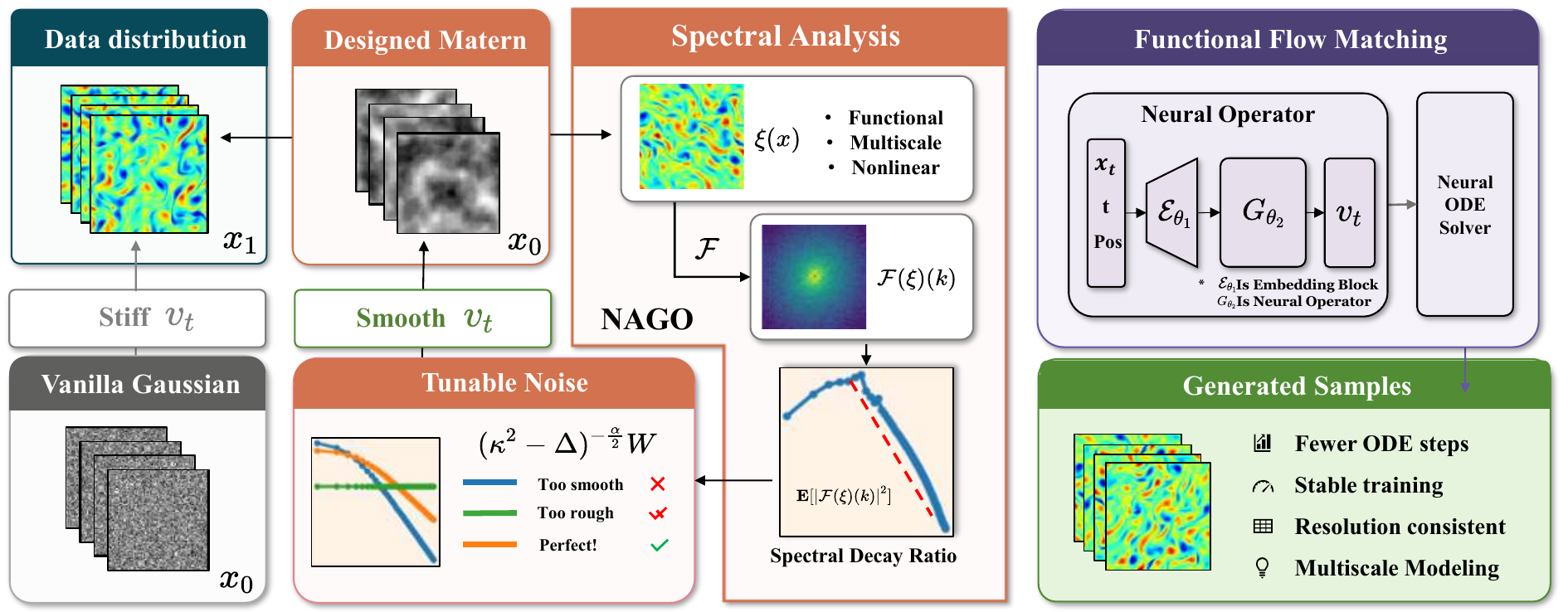}
    \caption{The structure of the \themodel model.}
    \label{fig:NAGO}
    \vspace{-1.5em}
\end{figure}

\paragraph{Code availability.} The code release is available at \url{https://github.com/cgymmy/NAGO}.

%% file: src/3_Preliminary.tex
\section{Preliminary}

In this section, the fundamental concepts of random fields and SPDEs are first introduced, followed by a discussion of invariant measures of SPDEs, and concluded with a brief overview of the basic theory underlying functional flow-matching generative models.
\subsection{Random Field and Stochastic PDEs}\label{3_1}
To formalize invariant measures and noise distributions in function space, we briefly recall the notion of random fields. Usually, a random field is regarded as a random variable taking values in a function space; we formalize this as follows:
\begin{definition}[Random Field]
For a set $D \subset \mathbb{R}^d$, a (real-valued) random field $\{u(x): x \in D\}$ is a collection of real-valued random variables defined on a probability space $(\Omega, \mathcal{F}, P)$, where $\Omega$ denotes the sample functional space, $\mathcal{F}$ is a $\sigma$-algebra of subsets of $\Omega$, and $P$ is a probability measure on $(\Omega, \mathcal{F})$.
\end{definition}

The class of random fields is broad and contains many important special cases with distinct structural properties as discussed in Appendix~\ref{apdx:srf}. Among them, the Matérn field \citep{lindgren2011explicit} plays a particularly prominent role, as it defines a Gaussian random field with tunable smoothness and correlation length. These properties make it both mathematically tractable and practically useful, especially for studying the effect of regularity in function-space probabilistic modeling. We therefore introduce the Matérn field as a representative example.
\begin{definition}[Mat\'ern Field]
A Gaussian random field $u(x)$ with Mat\'ern covariance is called a Mat\'ern field, which is a solution to the linear fractional stochastic partial differential equation:
\begin{equation}\label{matern_field}
    \mathcal{L}^{\alpha/2}u(x) = (\kappa^2 - \Delta)^{\alpha/2} u(x) = W(x),
\end{equation}
where $x\in D\subset \mathbb{R}^d$, $\nu = \alpha - d/2>0$ describes the smoothness, $\rho = \frac{\sqrt{2\nu}}{\kappa}>0$ for the correlation length, $\Delta$ is the Laplacian operator, and $W(x)$ is a spatial Gaussian white noise. 
\end{definition}


The Mat\'ern field provides a canonical example of Gaussian measures.
On $\mathbb{R}^d$, or under periodic boundary conditions, the spectral
solution of $u$ defined by Eq.~\eqref{matern_field} is given by
\begin{equation}\label{eq_matern}
    u(x)
    =
    \mathcal{F}^{-1}
    \left[
        G(k)\widehat W(k)
    \right](x),
    \quad
    G(k)
    =
    (\kappa^2+\|k\|^2)^{-\alpha/2},
\end{equation}
where $\mathcal{F}$ is the Fourier transform defined in
Appendix~\ref{FourierTransform}, and $G(k)$ is the Fourier multiplier
of the solution operator. The corresponding spectral density of $u$ is
    $S_u(k)\propto |G(k)|^2=(\kappa^2+\|k\|^2)^{-\alpha}$.
The covariance function of $u$ is then given by:
\begin{equation}\label{eq:cov_matern}
    c(x) = \frac{\sigma^2}{2^{\nu -1}\Gamma(\nu)}(\kappa \|x\|)^\nu K_\nu (\kappa \|x\|),
\end{equation}
where $\nu = \alpha - d/2$, $\rho = \frac{\sqrt{2\nu}}{\kappa}, \sigma^2 = \frac{\Gamma(\nu)}{(4\pi)^{d/2}\kappa^{2\nu}\Gamma(\alpha)}$ and $K_\nu$ is the modified Bessel function of the second kind. 
Detailed derivations are provided in Appendix~\ref{apdx_sec_matern}. In addition, Gaussian noise, white noise, and general Mat\'ern fields are closely related, as summarized in Table~\ref{tab_noise}.
\begin{wraptable}{r}{0.58\textwidth}
\centering
\vspace{-5pt}   
\caption{Comparison of different random fields.}
\label{tab_noise}
\begin{tabular}{lccc}
\toprule
 & Mat\'ern & Gaussian & White \\
\midrule
Smoothness $\alpha$ 
& $> d/2$ 
& $0$ 
& $0$ \\
Range parameter $\kappa$ 
& $>0$ 
& -- 
& -- \\
Normalized 
& -- 
& Yes 
& No \\
\bottomrule
\end{tabular}
\vspace{-10pt}   
\end{wraptable}
Here, ``Normalized'' indicates whether the corresponding noise has been rescaled to have unit marginal variance in the discretized setting. In the remainder of this paper, we use the term \textbf{Mat\'ern noise} to denote noise generated from a Mat\'ern field.

Having introduced random fields and the Matérn field, we now turn to the SPDEs of interest in this work. In general, SPDEs can be divided into two categories \citep{Lord_Powell_Shardlow_2014}: stationary SPDEs, which describe mappings between probability measures as introduced in Appendix~{\ref{apdx_spde}}, and semilinear evolution equations, which describe the temporal evolution of systems perturbed by noise. 
Analogous to ordinary time-dependent PDEs, semilinear evolution SPDEs are treated as semilinear stochastic ordinary differential equations (SODEs) on a Hilbert space $\mathcal{H}_U$, taking the form:
\begin{equation}\label{SPDE}
\left\{
\begin{aligned}
&du = [\mathcal{L}u+F(u)]dt + G(u)dW(t),\\
&u(0, x) = u_0, \quad x \in D.
\end{aligned}
\right.
\end{equation}
Here, $D$ is a bounded spatial domain in $\mathbb{R}^d$, $\mathcal{L}$ denotes a linear operator that generates the semigroup $P_t=e^{t\mathcal{L}}$; $W(t)$ is an $\mathcal{H}_\xi$-valued $Q$-Wiener process, where $\mathcal{H}_\xi$ is a separable Hilbert space; $F: \mathcal{H}_U \rightarrow \mathcal{H}_U$ denotes a nonlinear operator; and $G: \mathcal{H}_U \rightarrow \mathcal{L}\left(\mathcal{H}_{0,\xi}, \mathcal{H}_U\right)$ is a Hilbert--Schmidt operator associated with the Cameron--Martin space $\mathcal{H}_{0,\xi}$ of $\mathcal{H}_\xi$.
Under appropriate dissipativity and ergodicity conditions (see Thm.~\ref{condition_invariant}), the SPDE admits a unique invariant measure $\mu$, which is defined as follows:
\begin{definition}[Invariant Measure]\label{def_invariant_measure}
A Borel probability measure $\mu$ on $(H,\mathcal{B}(H))$ is called an \emph{invariant measure} for Equ.\eqref{SPDE} if it is invariant under the semigroup $(P_t)_{t\ge 0}$, i.e. $\forall\, t\ge 0,\ \forall\, \varphi\in B_b(H),$
\[
\int_H (P_t\varphi)(x)\,\mu(dx) \;=\; \int_H \varphi(x)\,\mu(dx),
\]
Equivalently, if $u(0)\sim \mu$, then $u(t)\sim \mu$ for every $t\ge 0$.
\end{definition}
Moreover, we say the SPDE is \emph{ergodic} (or \emph{mixing}) with invariant measure $\mu$ if for every initial condition $u_0\in H, Law(u^x(t)) \rightarrow \mu$ as $t\to\infty$, where $\rightarrow$ denotes weak convergence of probability measures on $H$. Further details on the existence and properties of invariant measures for SPDEs are provided in Appendix~\ref{apdx_spde}.

\subsection{Infinite-dimensional Flow Matching}
Compared with finite-dimensional generative models, the marginal velocity field in infinite-dimensional generative modeling is well-defined only under rather stringent theoretical conditions \citep{pidstrigach2024infinite, kerrigan2023functional}.
Here, we briefly review the theoretical framework for infinite-dimensional generative models and the conditions required for well-posedness.

Assume that the initial noise $g \in \mathcal{U}$ follows a Gaussian measure $\mu_0$ with covariance operator $Q:\mathcal{U}\to\mathcal{U}$, where $Q$ is self-adjoint, nonnegative, and trace-class, and $(\mathcal{U},\langle\cdot,\cdot\rangle)$ is a separable Hilbert space.
Data samples are denoted by $f$, distributed according to the target measure $\mu_1$.
Our objective is to construct an infinite-dimensional flow $\psi_t:[0,1]\times\mathcal{U}\to\mathcal{U}$ that pushes $\mu_0$ to $\mu_1$.
Then the marginal probability measure at time $t$ is defined as $\mu_t := [\psi_t]_\# \mu_0$, which admits the representation $\mu_t(A)=\int_{\mathcal{U}}\mu_t(A|f)\,d\mu_1(f),
    \quad \forall A\in \mathcal{B}(\mathcal{U}),$ 
where $\mu_t(\cdot|f)$ denotes the conditional measure and $\mathcal{B}(\mathcal{U})$ is the Borel $\sigma$-algebra on $\mathcal{U}$. Then we define the map $\psi_t$ which is associated with a time-dependent velocity field $v_t:[0,1]\times\mathcal{U}\to\mathcal{U}$ through
\begin{equation}
    \frac{d}{dt}\psi_t(g)=v_t(\psi_t(g)),\qquad \psi_0(g)=g\sim\mu_0.
\end{equation}
When the target measure $\mu_1$ is fixed, the central object of interest is the \emph{marginal velocity field} $v_t$, whose regularity and well-posedness directly determine the stability and learnability of the generative flow.
The corresponding infinite-dimensional flow matching objective is $    \mathcal{L}_{\mathrm{FFM}}
    = \mathbb{E}_{t,\,h\sim\mu_t}\!\left[\|v_t(h)-v_\theta(h)\|^2\right].$ 
As in finite-dimensional settings, the marginal velocity field is generally intractable.
We therefore consider conditional velocity fields $v_t(\cdot\mid f)$ associated with conditional measures $\mu_t(\cdot\mid f)$. We recall two standard results from functional flow matching that express the marginal field through conditional dynamics and characterize the required absolute continuity; our new source-design analysis begins in Sec.~\ref{observation}.
\begin{theorem}[Recalled marginalization result]\label{thm:3_1}
Given
$\int_0^1\!\int\!\!\int \|v_t(h|f)\|^2\,d\mu_t(h|f)\,d\mu_1(f)\,dt$ is finite,
and that for a.e. $f\sim\mu_1$ and $t\in[0,1]$, the conditional measure $\mu_t(\cdot|f)$ is absolutely continuous with respect to $\mu_t$ ($\mu_t(\cdot|f)\ll\mu_t$).
Then the marginal velocity field admits the representation
\begin{equation}
    v_t(h)
    = \int_{\mathcal{U}} v_t(h|f)\,
    \frac{d\mu_t(h|f)}{d\mu_t(h)}\,d\mu_1(f).
\end{equation}
\end{theorem}
The absolute continuity condition $\mu_t(\cdot|f)\ll\mu_t$ is therefore a key requirement for the existence of a well-defined marginal velocity field.
We adopt affine conditional flows of the form
\begin{equation}\label{interpolation}
    f_t=\psi_t(g|f)=\alpha_t g+\beta_t f,
\end{equation}
which push $g\sim\mu_0=\mathcal{N}(0,Q)$ toward a given $f$.
The corresponding conditional measure is $\mu_t(\cdot|f)=\mathcal{N}(\beta_t f,\alpha_t^2 Q)$.
\begin{theorem}[Recalled absolute-continuity criterion]\label{thm:3_2}
    Consider a given measure $\mu$ on a Hilbert space $\mathcal{U}$ and the corresponding conditional probability measure $\mu_t(\cdot|f), f\sim \mu$. If $a.e. f,g\sim \mu$, $\mu_t(\cdot|f)$ is equivalant to $\mu_t(\cdot|g)$, then the marginal probability measure $\mu_t$ satisfies $\mu_t(\cdot|f)\ll\mu_t$ for $a.e. f\sim \mu$. Proof is shown in Appendix~{\ref{pf:3.2}}.
\end{theorem}
By the Feldman--Hájek theorem, this equivalence holds if and only if $f-g\in \alpha_t \mathcal{Q}^{1/2}(\mathcal{U})$, implying that the data distribution must be supported on the Cameron--Martin space associated with the noise covariance. In other words, the noise distribution must be rougher than the data distribution, an observation that underlies the conclusions of most existing theoretical works.

%% file: src/4_Methodology.tex
\section{Methodology} 
The framework of the proposed Noise-Aware Generative Operator (\themodel) is presented in this section.
We first introduce the problem setting and establish the notation used throughout the paper.
We then extend the theoretical framework for infinite-dimensional generative models, focusing on how different noise measures affect the marginal velocity field in flow matching.
Finally, we describe the complete \themodel framework together with the associated algorithmic strategy.

\subsection{Problem Setting}
Our goal is to generate independent functional samples from the invariant distribution $\mu$ defined in Def.~(\ref{def_invariant_measure}), without explicitly simulating the underlying SPDE over long time horizons.
In addition to unconditional generation, a learned generative model should support reconstruction under partial observations and enable principled uncertainty quantification through guided sampling.

\subsection{Noise and Marginal Velocity Fields}
\begin{figure}[t]
\centering
\includegraphics[width=\linewidth]{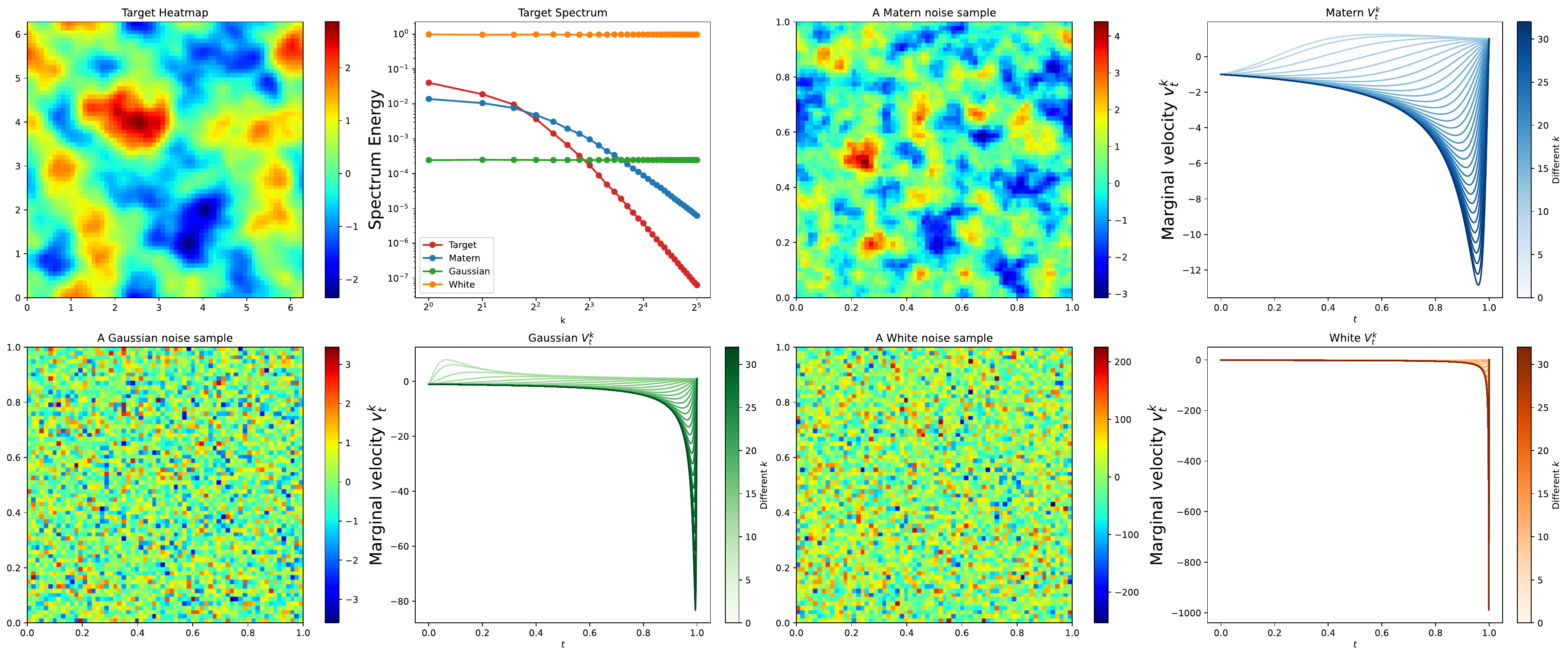}
  \caption{
Samples under various noise distributions and the target distribution, their spectra, and the marginal velocity fields at different frequencies for each noise type.
  }
  \label{fig:different_noises}
  \vspace{-1.5em}
\end{figure}
Here, we focus on how different noise choices affect infinite-dimensional generative models, and analyze their impact primarily through a spectral perspective.
In practice, the i.i.d. Gaussian array used after discretization is extremely rough. It is a finite-dimensional baseline rather than a trace-class Gaussian measure on the continuum Hilbert space assumed in Sec.~\ref{3_1}; this distinction is important when interpreting cross-resolution behavior.
We therefore generalize the noise family to Mat\'ern measures mentioned above, which are analytically tractable and practically flexible.
\begin{theorem}\label{thm_cm_space} 
Let $u$ be Mat\'ern noise with smoothness parameter $\nu$ in dimension $d$.
Then $u \in H^{s}$ almost surely for all $s<\nu$, while the associated Cameron--Martin space is $\mathcal{H}_{\mathrm{CM}} = H^{\nu + d/2}$.
The proof is provided in Appendix~\ref{pf:3.3}.
\end{theorem}
This theorem identifies the source regularity for which typical target
samples lie in the Mat\'ern source's Cameron--Martin space.

However, we find that satisfying Thm.~\ref{thm_cm_space} remains a rather loose criterion for Mat\'ern noise.
We therefore study the consequences of such regularity mismatch and quantify how noise regularity shapes the geometry and spectral stiffness of the marginal velocity field.
To make the analysis explicit, we first derive a closed-form expression for the marginal velocity field in the Gaussian setting.
\begin{theorem}\label{thm:3_4}
Given $\alpha_t, \beta_t\in C^1([0,1])$ with $\alpha_t^2+\beta_t^2>0$ for every $t$, consider the interpolant process defined in Equ.~\eqref{interpolation}, where $g\sim \mathcal{N}(0,Q_0)$ and $f\sim \mathcal{N}(0,Q_1)$ are independent samples from Gaussian measures supported on a Hilbert space $\mathcal{H}$. Assume that $Q_0$ and $Q_1$ have a common eigenbasis with strictly positive eigenvalues.
Then the marginal velocity field admits the closed-form expression
\begin{equation}\label{eq:v(t)}
\begin{aligned}
    v_t(x) &= \mathbb{E}[\dot{f}_t \mid f_t = x] = V(t)x, \\
    V(t) &= \left( \dot{\alpha}_t \alpha_t I + \dot{\beta}_t \beta_t Q_1Q_0^{-1} \right)
           \left( \alpha_t^2I + \beta_t^2 Q_1Q_0^{-1} \right)^{-1}.
\end{aligned}
\end{equation}
The proof is provided in Appendix~\ref{pf:3.4}.
\end{theorem}
To further examine the impact of noise regularity, we consider two Mat\'ern fields with covariance operators $Q_i$, $i=0,1$ as in Fig.~\ref{fig:different_noises}.
In this case, the marginal velocity field can be analyzed in the Fourier domain, where it admits the representation
\begin{equation}
    \tilde{V}_k(t)
    = \frac{(t-1)+t\mu_k}{(t-1)^2 + t^2\mu_k},
    \qquad
    \mu_k = \frac{\mathcal{F}_k(Q_1)}{\mathcal{F}_k(Q_0)}.
\end{equation}
Here, $\tilde{V}_k(t)$ denotes the Fourier coefficient of the velocity field at frequency $k$.
For Mat\'ern covariances of the form $Q_i = c_i^2(\kappa_i^2-\Delta)^{-(\nu_i + d/2)}$, we obtain 
\begin{equation}\label{eq_mu_k}
    \mu_k
    = \frac{c_1^2}{c_0^2}
      \frac{(\kappa_0^2 + 4\pi^2\|k\|^2)^{\nu_0 + d/2}}
           {(\kappa_1^2 + 4\pi^2\|k\|^2)^{\nu_1 + d/2}}.
\end{equation}
\paragraph{Observation.}\label{observation}
When the extreme case $Q_0 = Q_1$ is considered, we have $\mu_k = 1$ for all $k$, and the marginal velocity field in the Fourier domain reduces to
$\tilde{V}_k(t) = \frac{2t-1}{2t^2 - 2t + 1}$, which is independent of frequency.
In this regime, the marginal velocity field is spectrally smooth, making it comparatively easier for a model to learn.
In contrast, when the noise and data distributions are mismatched in regularity, particularly when $\nu_0 < \nu_1$, it follows from \citet{chen2025scale} that $\mu_k \to 0$ as $\|k\|\to\infty$.
Consequently, $\tilde{V}_k(t)$ develops increasingly sharp variations. For high-frequency modes with $0<\mu_k\leq 1$, its minimum on $t\in[0,1]$ is $\tilde{V}_k^{\min} = -\frac{1}{2}\left(\sqrt{\mu_k} + \frac{1}{\sqrt{\mu_k}}\right)$, 
diverging to $-\infty$ as $\mu_k \to 0$.
This behavior indicates severe spectral stiffness in the marginal velocity field, which poses significant challenges for both learning and numerical integration as in Fig.~\ref{fig:different_noises}.

In practice, functional data are represented on discretized meshes $k< \infty$. As a result, the spectral-stiffness effect is partially mitigated in experiments.
As illustrated in Fig.~\ref{fig:different_noises}, when the regularity mismatch between the noise and target measures is pronounced, as in the case of white noise, the high-frequency components of the marginal velocity field become extremely sharp as $t\to 1$.
Although discretized Gaussian noise partially mitigates this effect, an admissible Mat\'ern source can reduce the mismatch over the modes represented by the model.
This analysis motivates selecting the source through both continuum admissibility and finite-band spectral alignment.

\subsection{\themodel: Noise-Aware Generative Operator}
Here, we propose a noise-aware generative operator framework (NAGO) that uses target spectra to choose the source measure before training.
The overall pipeline is illustrated in Fig.~\ref{fig:NAGO}.
At a high level, NAGO consists of three steps:
(i) estimating the spectral decay of the target data measure;
(ii) selecting a Mat\'ern noise model whose Cameron--Martin space contains typical target samples, while controlling spectral stiffness; and
(iii) training a functional flow-matching model with a neural-operator backbone.

\textbf{Spectral characterization.}
Motivated by Equ.~\eqref{eq_matern}, we characterize the regularity of the target measure $\mu$ via its spectral energy.
Let $\tilde{f}(k)$ denote the Fourier coefficient of $f\sim\mu$ at frequency $k\in\mathbb{Z}^d$.
We define the spectral energy function as
\begin{equation}
    S_d(k) := \mathbb{E}_{f\sim \mu}\big[\|\tilde{f}(k)\|^2\big],\qquad k\in\mathbb{Z}^d.
\end{equation}
A classical characterization states that a function belongs to the Sobolev space $H^s$ if and only if
$\sum_{k\in\mathbb{Z}^d}(1+\|k\|^2)^s\,\|\tilde{f}(k)\|^2<\infty$.
Accordingly, the decay behavior of $S_d(k)$ provides an empirical proxy for the Sobolev regularity of the data measure.
In practice, we estimate the decay exponent $\beta$ by fitting a line to the log-log plot of the radially averaged spectrum over a high-frequency band; further details are provided in Appendix~\ref{apdx:pf_3_6}.

\textbf{Admissible noise selection.}
The following theorem gives a sufficient smoothness condition for typical
target samples to lie in the source Cameron--Martin space.
\begin{theorem}\label{thm:noise_condition}
    If the spectrum of the data measure exhibits polynomial decay of the form
    $S_d(k)\sim \|k\|^{-\beta}$ as $\|k\|\to\infty$, then choosing Mat\'ern source smoothness $\nu_0 < \frac{\beta}{2} - d$ ensures this Cameron--Martin inclusion. Proof is shown in Appendix~\ref{apdx_proof_noise_condition}.
    For example, if the target Mat\'ern field is characterized by $\nu = 2.0$
    (corresponding to $\beta = 6$ in $d=2$), then an admissible Mat\'ern noise must satisfy $\nu_0 < 1$.
\end{theorem}
The admissible set from Thm.~\ref{thm:noise_condition} can be refined by
controlling the spectral velocity: excessively rough noise induces stiffness
through the covariance ratio $\mu_k$ in Equ.~\eqref{eq_mu_k},
leading to the spiky behavior discussed in Observation~\ref{observation}.
Therefore, subject to the admissibility constraint in Thm.~\ref{thm:noise_condition}, we give the following theorem:
\begin{theorem}\label{thm:3_6}
    Assume the source and target are Mat\'ern measures with the same $\kappa$ and positive Fourier spectra. Let the target smoothness be $\nu_1$, and let $\nu_{\rm adm}=\nu_1-d/2$ be the Cameron--Martin admissibility threshold from Thm.~\ref{thm:noise_condition}. Define $r_0=\mathcal{F}_0(Q_1)/\mathcal{F}_0(Q_0)$, the target-to-source covariance ratio at the zero mode. For a prescribed uniform bound $L>1$ on the magnitude of the Fourier marginal velocity coefficient, choose $\nu_0$ so that
    \begin{equation}\label{eq:choose_noise}
        0<\nu_0<\nu_{\rm adm}-\varepsilon,
    \quad
        r_0 M_N^{\nu_1-\nu_0}>C_L,
    \quad
    r_0<C_L^{-1},
    \end{equation}
    where $M_N = \frac{\kappa^2}{\kappa^2 + 4\pi^2N^2}<1$, $C_L = (L -\sqrt{L^2-1})^2<1$, $N$ is the maximum represented frequency, and $\varepsilon>0$ enforces strict admissibility. Then $|\tilde V_k(t)|<L$ for every $t\in[0,1]$ and $\|k\|\le N$. The proof is provided in Appendix~\ref{apdx:pf_3_6}.
\end{theorem}
The shared-$\kappa$ result gives a spectral design guide for Mat\'ern sources. Operationally, we (i) estimate the tail exponent $\widehat\beta$, (ii) restrict $\nu_0\leq \widehat\beta/2-d-\varepsilon$, and (iii) choose $(\nu_0,\kappa_0)$ from this admissible set to minimize the empirical mismatch of the covariance ratio $\widehat\mu_k=\widehat S_1(k)/S_0(k)$ over the resolved band. We use
\begin{equation}\label{eq:finite_band_selection}
 (\nu_0,\kappa_0)=\arg\min_{\mathcal A_\varepsilon}
 \sum_{k\in\mathcal K_N}w_k\bigl(\log \widehat\mu_k-\overline{\log\widehat\mu}\bigr)^2,
 \qquad \mathcal A_\varepsilon=\{\nu_0\leq\widehat\beta/2-d-\varepsilon\},
\end{equation}
where $\mathcal K_N$ contains the represented Fourier modes and $w_k$ can downweight unreliable tail estimates. This criterion seeks a nearly frequency-independent ratio, as suggested by Observation~\ref{observation}; the length-scale sweep in Appendix~\ref{apdx:source_selection_ablation} shows why smoothness matching alone is insufficient.
\textbf{Implication for common noise choices.}
Discretized Gaussian noise can be regarded as spectral noise with constant high-frequency energy.
Although it is always available on a fixed grid, its high-frequency energy tends to induce sharper variations in the learned marginal velocity field than an admissible Mat\'ern source selected by Eq.~\eqref{eq:finite_band_selection}.
This explains why appropriately regularized spectral noise can substantially improve both stability and efficiency.

\paragraph{Training and guided generation.}
The subsequent training follows the standard flow matching framework, with Fourier Neural Operator (FNO) used as the backbone architecture.
To support guided generation under partial observations, we employ the infinite-dimensional Tweedie formula.
Assume that $\mu_t \ll \mu_0$, $\dot\alpha_t\neq0$, and
$\beta_t-\dot\beta_t\alpha_t/\dot\alpha_t\neq0$. Then the Tweedie formula is
\begin{equation}\label{eq:tweedie}
\mathbb{E}[f_1|f_t] =
\left(\beta_t - \dot{\beta_t}\frac{\alpha_t}{\dot{\alpha_t}}\right)^{-1}
\left(f_t - \frac{\alpha_t}{\dot{\alpha_t}}v_t(f_t)\right).
\end{equation}
The proof is provided in Appendix~\ref{pf:tweedie}.
The complete training and sampling algorithms are summarized in Appendix~\ref{sec_alg}.

%% file: src/5_Experiment.tex
\section{Experiments}

\begin{figure}[!t] 
  \centering
    \includegraphics[width=\textwidth]{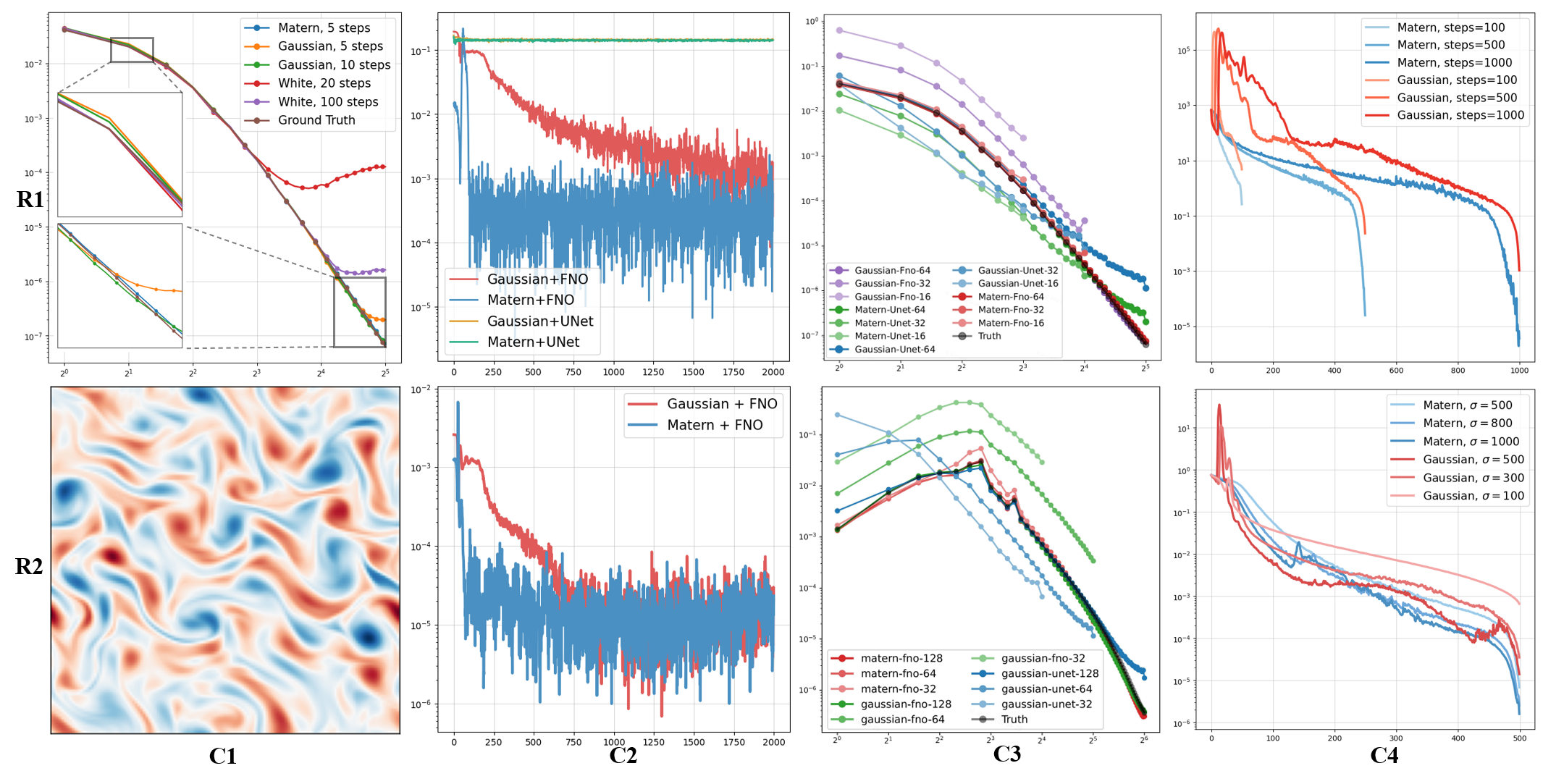}
\caption{Noise-source comparison on the Mat\'ern-field (top) and SNS (bottom) benchmarks. Columns show representative samples or spectra, error versus solver budget, cross-resolution spectra, and guided reconstruction loss, respectively.}
\label{fig_all}
\vspace{-15pt}
\end{figure}

In this section, we empirically evaluate the proposed \themodel with the explicit goal of validating the hypotheses raised in the Introduction regarding noise-induced spectral stiffness and its impact on training stability, sampling efficiency, resolution consistency, and guided generation.
Specifically, we focus on three aspects:
(i) generation accuracy and efficiency at a fixed resolution,
(ii) consistency across varying spatial resolutions,
and (iii) numerical stability and accuracy under guided generation. 
Additional experimental details and ablation studies are provided in Appendix~\ref{apdx_exp_details}.

\textbf{Benchmarks.}
Three datasets are considered in this study, including linear Mat\'ern fields, the stochastic Allen--Cahn (SAC) equation, and the stochastic Navier--Stokes (SNS) equation.
Detailed descriptions of the underlying SPDEs and data generation procedures are provided in Appendix~\ref{apdx:intro_dataset}. 

\textbf{Experimental Setup.}
We evaluate three complementary aspects corresponding to the core claims of this work:
generation quality and efficiency,
resolution consistency,
and stability under guided generation.
The spectral accuracy $E(k)$ is defined as $E(k) = \sum_{k\leq \|m\|\leq k+1}\left|\tilde{u}(m)\right|^2,$ 
and the log-MSE (lMSE),
$\frac{1}{N}\sum_k (\log \tilde{E}_k-\log E_k)^2$,
is used as the primary spectral metric. For SNS, Appendix~\ref{apdx:sns_distribution_metrics} additionally reports multiscale MMD, marginal Wasserstein and Jensen--Shannon distances, structure functions, spatial correlations, and physical observables.
All experiments employ RK4 as the neural ODE solver to ensure that performance differences arise from noise design rather than solver choice.

\textbf{Baselines.}
We compare combinations of source distributions and backbones, summarized in Table~\ref{tab:baseline}. Controlled source ablations keep the data, architecture, optimizer, training budget, RK4 solver, and evaluation samples fixed; only the source distribution changes.
\subsection{Generation at Fixed Resolution}
We first evaluate unconditional generation at the same spatial resolution as the training data. 
For the Mat\'ern field benchmark, Fig.~\ref{fig_all} (top) shows that the source choice strongly affects the learned spectral energy distribution.
White noise concentrates excessive energy in high-frequency modes, causing the marginal velocity field to become ill-conditioned as $t\to 1$, and even 100 integration steps are insufficient for accurate sampling.
Gaussian noise partially alleviates this issue but still requires substantially finer temporal discretization.
This behavior is consistent with the noise-induced spectral stiffness discussed in the Introduction. In contrast, the selected Mat\'ern source attains low spectral error with 3--5 RK4 steps. Similar few-step trends occur for SAC and SNS.

Table~\ref{tab_all} reports the benchmark-level comparison. A denser controlled sweep on the Mat\'ern target (Appendix~\ref{apdx:ode_step_sweep}) shows that Gaussian FFM reaches its stable spectral-error regime at about 30 RK4 steps, whereas NAGO attains lower lMSE at 3--5 steps. Thus, source design reduces the accurate step and neural-function-evaluation budget by $6$--$10\times$ in this setting. The U-Net results indicate that the source can also help a grid-based backbone, but they do not establish architecture-independent gains.


\textbf{Generation Time Cost.}
On one RTX~4090 at batch size 16, the SNS pseudo-spectral solver requires 20.3814\,s per recorded batch, while five-step NAGO requires a median 0.228848\,s, a raw post-training throughput ratio of $89.1\times$. This comparison excludes dataset construction and model training and is not an equal-fidelity or end-to-end break-even claim. The Mat\'ern-specific efficiency evidence is the separate $6$--$10\times$ reduction in RK4 steps above; the timing protocol and autocorrelation-adjusted accounting are given in Appendix~\ref{apdx:speed_protocol}.



\subsection{Generation at Varying Resolutions}
We next separate same-resolution generation at $64^2$ from transfer to $32^2$ and $128^2$. As shown in Fig.~\ref{fig_all} (C3), the Mat\'ern source better preserves the reference spectrum across these grids. At $128^2$, it retains $93.99\%$ of energy in the frequency band shared with training and places $2.39\%$ in newly available modes, compared with $23.64\%$ and $76.92\%$ for the Gaussian source (Appendix~\ref{apdx:resolution_diagnostics}).

These results show preservation of the learned shared-frequency distribution and suppression of spurious new-band energy, not reconstruction of an unseen continuum tail. Indeed, the $128^2$ Mat\'ern palinstrophy ratio remains 6.72, so fine-scale extrapolation is still imperfect.

\subsection{Generation with Guidance}
Finally, we consider guided generation, where low-resolution observations are provided and the goal is to reconstruct the full field while quantifying uncertainty.

Previous studies \citep{lu2025dpm} have reported numerical oscillations or divergence during guided generation, particularly under strong guidance, which is closely linked to steep or ill-conditioned marginal velocity fields.
Several works \citep{chen2025scale, albergobuilding, albergo2025stochastic} have proposed interpolation or guidance strength schedules to mitigate these effects.

Figure~\ref{fig_all} (C4) and Table~\ref{tab:guidance_sns_compare} compare guided SNS reconstruction with 25\% observations under the original per-step guidance parameterization. Mat\'ern runs show smoother early loss trajectories and lower error for several matched settings. However, when total guidance intervention is held fixed as the solver budget changes, Gaussian sources can be better on observation MSE while Mat\'ern sources can retain better spectral fidelity. We therefore interpret guidance as a metric- and hyperparameter-dependent secondary result rather than direct causal confirmation of the stiffness mechanism; the fixed-total controls are reported in Appendix~\ref{apdx:guide_loss}.

\textbf{Remark.}
Guidance strength and its accumulation across steps materially affect the comparison. Appendix~\ref{apdx:guide_loss} reports this sensitivity and separates ODE refinement from cumulative guidance intervention.

\begin{table*}[!t]
\centering
\caption{
Log-MSE for different test cases.
Values in parentheses indicate the number of ODE integration steps used during inference.
Best results are highlighted in bold.
}
\resizebox{\textwidth}{!}{
\begin{tabular}{lccccccccc}
\toprule
\multirow{2}{*}{Noise}
& \multirow{2}{*}{Model/Steps}
& \multicolumn{3}{c}{Mat\'ern Field $(\times 10^{-3})$}
& \multicolumn{3}{c}{Stochastic Allen--Cahn $(\times 10^{-3})$}
& \multicolumn{2}{c}{Stochastic NS $(\times 10^{-2})$} \\
\cmidrule(lr){3-5} \cmidrule(lr){6-8} \cmidrule(lr){9-10}
& & C1(5) & C2(10) & C3(50) & C1(5) & C2(10) & C3(50) & C1(5) & C2(10) \\
\midrule

\multirow{2}{*}{Mat\'ern}
& FNO
& $\mathbf{1.70}$
& $\mathbf{0.638}$
& $\mathbf{1.44}$
& $\mathbf{0.553}$
& $\mathbf{0.778}$
& $\mathbf{0.612}$
& $\mathbf{4.44}$
& $\mathbf{4.40}$ \\
& UNet
& $312$
& $342$
& $275$
& $30.7$
& $32.1$
& $32.7$
& $7.27$
& $7.55$ \\

\midrule

\multirow{2}{*}{Gaussian}
& FNO
& $113$
& $5.66$
& $1.52$
& $7.16$
& $2.31$
& $0.899$
& $5.11$
& $9.03$ \\
& UNet
& $2860$
& $2700$
& $2760$
& $7.99$
& $6.03$
& $4.84$
& $44.7$
& $28.7$ \\

\midrule

White
& FNO
& $1.77{\times}10^{4}\,(20)$
& $8930\,(50)$
& $1780\,(100)$
& $204$
& $5.16$
& $1.20$
& $5220$
& $2670$ \\

\bottomrule
\end{tabular}
}
\label{tab_all}
\vspace{-10pt}
\end{table*}

\begin{table}[!t]
\label{dps}
\centering
\caption{Sampling with guidance for the SNS benchmark.}
\label{tab:guidance_sns_compare}
\footnotesize
\setlength{\tabcolsep}{3pt}
\renewcommand{\arraystretch}{1.12}
\resizebox{\linewidth}{!}{%
\begin{tabular}{lcccccccc}
\toprule
\multirow{2}{*}{Steps}
& \multicolumn{2}{c}{Mat\'ern + FNO}
& \multicolumn{2}{c}{Gaussian + FNO}
& \multicolumn{2}{c}{Mat\'ern + U-Net}
& \multicolumn{2}{c}{Gaussian + U-Net} \\
\cmidrule(lr){2-3}
\cmidrule(lr){4-5}
\cmidrule(lr){6-7}
\cmidrule(lr){8-9}
& L2 & MSE & L2 & MSE & L2 & MSE & L2 & MSE \\
\midrule
100
& \(\mathbf{8.42{\times}10^{-2}}\)
& \(\mathbf{7.27{\times}10^{-3}}\)
& \(1.16{\times}10^{-1}\)
& \(2.83{\times}10^{-2}\)
& \(6.26{\times}10^{-2}\)
& \(4.01{\times}10^{-3}\)
& \(3.10{\times}10^{-1}\)
& \(9.85{\times}10^{-2}\) \\
500
& \(\mathbf{1.03{\times}10^{-1}}\)
& \(\mathbf{1.08{\times}10^{-2}}\)
& \(1.59{\times}10^{-1}\)
& \(2.62{\times}10^{-2}\)
& \(2.76{\times}10^{-1}\)
& \(7.85{\times}10^{-2}\)
& \(1.52\)
& \(2.37\) \\
1000
& \(\mathbf{8.70{\times}10^{-2}}\)
& \(\mathbf{7.75{\times}10^{-3}}\)
& \(1.53{\times}10^{-1}\)
& \(2.39{\times}10^{-2}\)
& \(4.66{\times}10^{-1}\)
& \(2.27{\times}10^{-1}\)
& \(1.98\)
& \(4.04\) \\
\bottomrule
\end{tabular}%
}
\vspace{-2em}
\end{table}

%% file: src/6_Conclusion.tex
\section{Conclusion and Limitations}

In this work, we introduced \themodel, a noise-aware generative operator for sampling invariant measures of SPDEs. The key idea is to impose function-space admissibility first and then align the source and target covariance over the modes represented by the model. This separates the theoretical roughness requirement from the practical objective of reducing spectral stiffness. Across three stationary stochastic-field benchmarks, the selected Mat\'ern source improves few-step spectral accuracy and several complementary distributional statistics; on the controlled Mat\'ern target, it reduces the accurate RK4 budget from about 30 to 3--5 steps. The scope is nevertheless limited to stationary fields on periodic regular grids, where spectra can be estimated reliably and Mat\'ern sampling is FFT-diagonalizable. Cross-resolution results preserve shared frequencies but do not recover an unseen spectral tail, and guidance comparisons remain metric- and hyperparameter-dependent. Extending the admissibility-constrained criterion to irregular geometries, nonstationary spectra, and learned anisotropic sources is an important direction for future work.

%% file: src/7_Appendix_1.tex
\newpage
\appendix
\onecolumn

\startcontents[appendix]
\pdfbookmark[0]{Appendix Contents}{appendix.contents}
\section*{Appendix Contents}
\printcontents[appendix]{}{1}{\setcounter{tocdepth}{2}}
\clearpage

\input{src/2_Related}

\section{Preliminaries and notation}
\label{apdx:random_field_prelim}

We briefly review several standard facts on random fields, covariance operators, Karhunen--Lo\`eve expansions, and Gaussian measures on Hilbert spaces.
These notions provide the functional-analytic foundation for the noise measures considered in this work.
\subsection{Random Fields}
\subsubsection{Random Fields as Hilbert-Valued Random Variables}

Let $(\Omega,\mathcal{F},\mathbb{P})$ be a probability space and let $D\subset \mathbb{R}^d$ be a bounded domain.
We denote by $H$ a separable Hilbert space of functions on $D$, typically $H=L^2(D)$ or a Sobolev space.
A random field is viewed as an $H$-valued random variable
\begin{equation}
    u:\Omega\rightarrow H.
\end{equation}
For each $\omega\in\Omega$, $u(\omega)$ is a function on $D$.
When emphasizing the spatial variable, we write $u(x,\omega)$.

For $1\leq p<\infty$, the Bochner space $L^p(\Omega;H)$ consists of all strongly measurable random variables $u:\Omega\rightarrow H$ satisfying
\begin{equation}
    \mathbb{E}\|u\|_H^p<\infty,
\end{equation}
with norm
\begin{equation}
    \|u\|_{L^p(\Omega;H)}
    :=
    \left(\mathbb{E}\|u\|_H^p\right)^{1/p}.
\end{equation}
When $p=2$, $L^2(\Omega;H)$ is a Hilbert space with inner product
\begin{equation}
    \langle u,v\rangle_{L^2(\Omega;H)}
    :=
    \mathbb{E}\langle u,v\rangle_H.
\end{equation}

The law of $u$ is the probability measure
\begin{equation}
    m := u_{\#}\mathbb{P}
\end{equation}
on $H$.
Thus, a random field can equivalently be regarded as a probability measure on a function space.

\subsubsection{Mean and Covariance Operator}

Assume that $u\in L^2(\Omega;H)$.
Its mean is the element $\mu\in H$ defined by
\begin{equation}
    \mu := \mathbb{E}u.
\end{equation}
The covariance operator $\mathcal{C}:H\rightarrow H$ is defined weakly by
\begin{equation}
    \langle \mathcal{C}\phi,\psi\rangle_H
    =
    \mathbb{E}
    \left[
        \langle u-\mu,\phi\rangle_H
        \langle u-\mu,\psi\rangle_H
    \right],
    \qquad
    \forall \phi,\psi\in H .
    \label{eq:cov_operator}
\end{equation}
Equivalently,
\begin{equation}
    \mathcal{C}
    =
    \mathbb{E}\bigl[(u-\mu)\otimes(u-\mu)\bigr],
\end{equation}
where $(a\otimes b)\phi := a\langle b,\phi\rangle_H$.

The operator $\mathcal{C}$ is self-adjoint, positive semidefinite, and trace class.
Moreover,
\begin{equation}
    \operatorname{Tr}(\mathcal{C})
    =
    \mathbb{E}\|u-\mu\|_H^2.
    \label{eq:trace_second_moment}
\end{equation}

When $H=L^2(D)$ and $\mathcal{C}$ admits an integral kernel $C(x,y)$, we have
\begin{equation}
    (\mathcal{C}\phi)(x)
    =
    \int_D C(x,y)\phi(y)\,dy,
    \qquad
    C(x,y)
    =
    \operatorname{Cov}(u(x,\cdot),u(y,\cdot)).
\end{equation}

\subsubsection{Karhunen--Lo\`eve Expansion}

Assume that $H=L^2(D)$ and that the covariance operator $\mathcal{C}$ is compact, self-adjoint, positive semidefinite, and trace class.
Let $\{(\lambda_i,\phi_i)\}_{i\geq 1}$ be its eigenpairs:
\begin{equation}
    \mathcal{C}\phi_i=\lambda_i\phi_i,
    \qquad
    \lambda_i\geq 0,
\end{equation}
where $\{\phi_i\}_{i\geq 1}$ is an orthonormal basis of $H$.
Then the Karhunen--Lo\`eve expansion of $u$ is
\begin{equation}
    u(x,\omega)
    =
    \mu(x)
    +
    \sum_{i=1}^{\infty}
    \sqrt{\lambda_i}\,\xi_i(\omega)\phi_i(x),
    \label{eq:kl_expansion}
\end{equation}
where, for $\lambda_i>0$,
\begin{equation}
    \xi_i(\omega)
    =
    \frac{1}{\sqrt{\lambda_i}}
    \langle u(\omega)-\mu,\phi_i\rangle_H .
\end{equation}
The random variables $\{\xi_i\}_{i\geq 1}$ satisfy
\begin{equation}
    \mathbb{E}[\xi_i]=0,
    \qquad
    \mathbb{E}[\xi_i\xi_j]=\delta_{ij}.
\end{equation}
If $u$ is Gaussian, then $\{\xi_i\}_{i\geq 1}$ are independent standard normal random variables.

The expansion converges in $L^2(\Omega;H)$, and
\begin{equation}
    \mathbb{E}\|u-\mu\|_H^2
    =
    \sum_{i=1}^{\infty}\lambda_i
    =
    \operatorname{Tr}(\mathcal{C}).
\end{equation}

If $C(x,y)$ is continuous, then Mercer's theorem gives
\begin{equation}
    C(x,y)
    =
    \sum_{i=1}^{\infty}
    \lambda_i\phi_i(x)\phi_i(y),
    \label{eq:mercer}
\end{equation}
with convergence under the usual Mercer assumptions.

\subsubsection{Gaussian Measures on Hilbert Spaces}

An $H$-valued random variable $u$ is Gaussian if
\begin{equation}
    \langle u,\phi\rangle_H
\end{equation}
is a real-valued Gaussian random variable for every $\phi\in H$.
If $u$ has mean $\mu$ and covariance operator $\mathcal{C}$, we write
\begin{equation}
    u\sim \mathcal{N}(\mu,\mathcal{C}).
\end{equation}
Equivalently,
\begin{equation}
    \langle u,\phi\rangle_H
    \sim
    \mathcal{N}
    \left(
        \langle \mu,\phi\rangle_H,
        \langle \mathcal{C}\phi,\phi\rangle_H
    \right),
    \qquad
    \forall \phi\in H .
\end{equation}

In infinite-dimensional Hilbert spaces, the covariance operator of a Gaussian measure must be trace class.
Conversely, if $\mathcal{C}:H\rightarrow H$ is self-adjoint, positive semidefinite, and trace class, then there exists a Gaussian measure $\mathcal{N}(\mu,\mathcal{C})$ on $H$.

If $\mathcal{C}$ has eigenpairs $\{(\lambda_i,\phi_i)\}_{i\geq 1}$, then a centered Gaussian random variable with covariance $\mathcal{C}$ admits the representation
\begin{equation}
    u
    =
    \sum_{i=1}^{\infty}
    \sqrt{\lambda_i}\,\xi_i\phi_i,
    \qquad
    \xi_i\overset{\mathrm{i.i.d.}}{\sim}\mathcal{N}(0,1).
    \label{eq:hilbert_gaussian_series}
\end{equation}
The series converges in $L^2(\Omega;H)$ if and only if $\sum_{i=1}^{\infty}\lambda_i<\infty.$

\subsubsection{Spectral Decay and Regularity}

The covariance spectrum determines the regularity of Gaussian random fields.
Assume that $H=L^2(D)$ and that $\{\phi_k\}$ is a Fourier-type basis indexed by frequency $k$.
Let
\begin{equation}
    u
    =
    \sum_k
    \sqrt{\lambda_k}\,\xi_k\phi_k,
    \qquad
    \xi_k\overset{\mathrm{i.i.d.}}{\sim}\mathcal{N}(0,1).
\end{equation}
Then, formally,
\begin{equation}
    \mathbb{E}\|u\|_{H^s(D)}^2
    \asymp
    \sum_k
    (1+|k|^2)^s\lambda_k.
    \label{eq:regularity_spectral_decay}
\end{equation}
Hence $u\in H^s(D)$ in mean square whenever
\begin{equation}
    \sum_k
    (1+|k|^2)^s\lambda_k<\infty.
\end{equation}
Thus, faster decay of $\lambda_k$ corresponds to smoother samples, while slower decay corresponds to rougher samples with stronger high-frequency content.

A Mat\'ern-type Gaussian measure provides a convenient family of covariance spectra with tunable regularity.
In a Fourier basis, a typical Mat\'ern covariance has eigenvalues of the form
\begin{equation}
    \lambda_k
    \propto
    \left(\kappa^2+|k|^2\right)^{-\alpha},
    \label{eq:matern_eigenvalue_decay}
\end{equation}
where $\kappa>0$ controls the correlation length and $\alpha>0$ controls spectral decay.
Increasing $\alpha$ suppresses high-frequency modes and produces smoother random fields.

\subsubsection{Implication for Source Noise Design}

The preceding facts show that a Gaussian noise measure on a function space is determined by its covariance operator.
In particular, choosing the source noise distribution is equivalent to choosing a covariance spectrum $\{\lambda_k\}$, which fixes the regularity and high-frequency energy of the initial random field.

Therefore, in function-space generative modeling, the source noise is not a neutral default.
If the source noise is much rougher than the target measure, the induced generative dynamics must remove excessive high-frequency energy.
This produces spectral imbalance in the marginal velocity field and may lead to stiff dynamics.
Regularity-aware noise design aims to reduce this mismatch by choosing a source covariance spectrum whose decay is compatible with the target invariant measure.

\subsection{Noise Measure}
\subsubsection{Gaussian Measure}
Since we consider the infinite-dimensional case, unlike finite dimensions, not all translations preserve the measure. 
We need to consider those directions in $H$ along which translating a Gaussian measure does not change its essential nature (i.e., keeps it equivalent). 
Including the SPDE, these problems are also dealt with in Hilbert Space.
First, we have some definitions:
\begin{definition}
  Let $m_1, m_2$ be two measures on $H$.
  \begin{itemize}
    \item $m_1 \ll m_2$ means measure $m_1$ is absolutely continuous respect to $m_2$: if $m_1(A) = 0$ for all $A$ s.t. $m_2(A) = 0$. 
     This means that the support of $m_1$ is a subset of the support of $m_2$.
    \item If $m_1 \ll m_2, m_2 \ll m_1$, $m_1, m_2$ are said to be equivalent $m_1 \sim m_2$.
    \item If there exists a measurable set $A$ s.t. $m_1(A) = 0$ and $m_2(A^c) =1$, then $m_1, m_2$ are singular $m_1 \perp m_2$. 
  \end{itemize}
\end{definition}

\begin{theorem}[Radon-Nikodym Theorem]
  Let $S = (H,\mathcal{B}(H))$ be a measurable space. $m_1, m_2$ are two $\sigma$-finite measures on $S$. If $m_1 \ll m_2$, then there exists a measurable function $f:H\rightarrow [0, \infty)$ s.t. $m_1(A) = \int_A f m_2(du), \forall A \in \mathcal{B}(H)$.

  The function $f$ is called the Radon-Nikodym derivative of $m_1$ with respect to $m_2$ and is denoted by $f = \frac{dm_1}{dm_2}$.
\end{theorem}

\begin{remark*}
  In the finite-dimensional case, the Radon-Nikodym derivative is the probability density function. 
  Like $m_2$ is the Lebesgue measure, $m_1(dx) = p(x)dx$, then the Radon-Nikodym derivative is the probability density function $p(x)$. 
\end{remark*}

\begin{definition}[KL-divergence]
  The KL-divergence is actually related to the Radon-Nikodym derivative: if $m_1 \ll m_2$, then:
  \begin{equation}
    \operatorname{KL}(m_1||m_2) = \int_H \log\left(\frac{dm_1}{dm_2}(x)\right) dm_1(x) 
    = \int_H \log\left(f(x)\right) f(x)dm_2(x)
  \end{equation}
  This is the KL divergence between measures. We define the Radon-Nikodym derivative to be infinite if $m_1 \not\ll m_2$.
  So for $H=\mathbb{R}^d, dm_1 = p_1(x)dx, dm_2 = p_2(x)dx$, the KL-divergence can be written as:
  \begin{equation}
    \operatorname{KL}(m_1||m_2) = \int_{\mathbb{R}^d} \log\left(\frac{p_1(x)}{p_2(x)}\right) p_1(x) dx
  \end{equation}
  The Radon-Nikodym derivative captures local density ratios, and the KL divergence is their global average.
\end{definition}

Then we need to consider the characteristic functional of Gaussian measure, cause one can show that characteristic functions of two measures are identical, then the measures are identical as well:
\begin{definition}[Characteristic Functional]
  The characteristic functional of a measure $m$ is defined as:
  \begin{equation}
    \hat{m}(\lambda) = \int_H e^{i\langle \lambda, u\rangle_H} m(du)
  \end{equation}
\end{definition}
For Gaussian measure $m=N(\mu, \mathcal{C})$, we have the characteristic functional:
\begin{equation}
  \hat{m}(\lambda) = \int_H e^{i\langle \lambda, u\rangle_H} m(du) = e^{i\langle \lambda, \mu\rangle_H-\frac{1}{2}\langle \lambda, \mathcal{C}\lambda\rangle_H}, \qquad \lambda\in H
\end{equation}
So conversely, let $\mu\in H, \mathcal{C}$ is a trace class operator on $H$. Then there exists a unique probability measure $m$ on $H$ s.t. $\hat{m}(\lambda) = e^{i\langle \lambda, \mu\rangle_H-\frac{1}{2}\langle \lambda, \mathcal{C}\lambda\rangle_H}, \forall \lambda\in H$. Here $m=N(\mu, \mathcal{C})$. Meanwhile, we have $\int_H \|x\|^2_H m(dx) =\operatorname{Tr}(\mathcal{C}) + \|\mu\|^2_H$.

\begin{definition}[Reproducing Kernel Hilbert Space]
    The subspace $\mathcal{C}^{1/2}H$ is called the reproducing kernel of measure $N_{\mathcal{C}}$. If $\operatorname{Ker}(\mathcal{C})=0$, then $\mathcal{C}^{1/2}H$ is dense in $H$. 
\end{definition}
Then we have an isomorphism between $H$ and $L^2(H, N_{\mathcal{C}})$:
\begin{equation}
    f\in \mathcal{C}^{1/2}H \rightarrow W_f\in L^2(H, N_{\mathcal{C}}), \qquad W_f(x)=\langle \mathcal{C}^{-1/2}f, x\rangle_H,\quad x\in H
\end{equation}
for $\int_H W_f(x) W_g(x) N_{\mathcal{C}}(dx) = \langle f, g\rangle_{\mathcal{C}}, \quad \forall f, g\in H$.
\subsubsection{Camaron Martin Space}

Unlike the finite dimensional case, there is no natural Brownian Motion process in infinite dimensions. 
Here is why: For example, the space time white noise $W_t$ is defined on $H=L^2([0, 1])$, but it doesn't take values in $H$ a.s.
Since the white noise has covariance operator $I$, it is not trace class in $H$. 
So for Hilbert space $H$, we need to define the Brownian Motion process on $H$ by using the Cameron-Martin space $U$.
In other words, if we want to define the Diffusion Process on infinite dimensional space, 
we need to consider the direction in $H$ along which translating a measure preserves the absolutely continuity.
\begin{definition}[Cameron-Martin Space]
The Cameron-Martin space $U$ is defined as:
\begin{equation}
  U:=\{h\in H| m_h\ll m\},m_h(A)=T_h^\#m(A):= m(A-h),\forall A \in \mathcal{B}(H)
\end{equation} 
where $T_h$ is the translation operator: $T_h(u)=u + h$, $T_h^\#m$ is the push-forward measure of $m$ by $T_h$, 
\end{definition}
So the Cameron-Martin space $U$ is more regular than the original space $H$.

\begin{theorem}[Cameron-Martin Theorem]
  For $m=N(0, \mathcal{C})$, we have:
  \begin{equation}
    m_h\sim m \text{ if and only if } h\in \mathcal{C}^\frac{1}{2}H:=U
  \end{equation}
  Since $\mathcal{C}^{-1}$ is unbounded, $U$ is a proper subspace of $U = \mathcal{C}^{1/2}H$.
  The inner product is defined as:
\begin{equation}
  \langle \phi, \psi\rangle_U = \langle \mathcal{C}^{-1/2}\phi, \mathcal{C}^{-1/2}\psi\rangle_H
\end{equation}
\end{theorem}

This can be generalized to general Gaussian measure cases:
\begin{theorem}[Feldman-Hajek Theorem~\citep{Da_Prato_Zabczyk_2014}]\label{thm_FH}
  The following statements hold:\\
  1. $m_1 = N(\mu_1, \mathcal{C}_1), m_2 = N(\mu_2, \mathcal{C}_2)$ are either singular or equivalent.\\
  2. They are equivalent if and only if 
  \begin{itemize}
      \item $\mathcal{C}_1^{1/2}(H) = \mathcal{C}_2^{1/2}(H)=H_0$, $\mu_1 - \mu_2 \in H_0$
      \item $(\mathcal{C}_1^{-1/2}\mathcal{C}_2^{1/2})(\mathcal{C}_1^{-1/2}\mathcal{C}_2^{1/2})^*-I$ is a Hilbert-Schmidt operator.
  \end{itemize}
  Specifically, if $\mathcal{C}_1=\mathcal{C}_2=\mathcal{C}$, we have the Radon Nikodym derivative:
  \begin{equation} 
    \frac{dm_1}{dm_2}(u) = \exp{\left(\langle \mu_1 - \mu_2, \mathcal{C}^{-1}(u - \mu_2) \rangle_H - \frac{1}{2}\|\mathcal{C}^{-1/2}(\mu_1 - \mu_2)\|^2_H\right)}
  \end{equation}
  Therefore, the KL divergence between $m_1$ and $m_2$ is:
  \begin{equation}
    \operatorname{KL}[m_1|m_2] = \frac{1}{2}\left(\langle\Delta\mu, C^{-1}\Delta \mu\rangle_H\right)
  \end{equation}
\end{theorem}

First, we assume that the Cylindrical Wiener Process $\hat{W}_t$ on $H$, but it does not take values in $H$ cause $I$ is not a trace class.
Here we suppose $\hat{W}_t$ is $H_1$-valued:
\begin{equation}
  \hat{W}_t = \sum_{i=1}^{\infty} \beta_i(t) \phi_i
\end{equation}
where $\beta_i(t)$ are independent standard Brownian motions and $\phi_i$ are the eigenbasis of $H$.
Normally, we can define the $\mathcal{C}$-Wiener Process on $H$ as:
\begin{definition}[$\mathcal{C}$-Wiener Process]
  A H-valued process $W_t$ is called $\mathcal{C}$-Wiener process if:
  1. $W_0 = 0$
  2. $W_t$ has continuous trajectories: $\mathbb{R}^+\rightarrow H$
  3. $W_t$ has independent Gaussian increments: $W_t - W_s \sim N(0, (t-s)\mathcal{C})$
\end{definition}
Therefore by applying the covariance operator $\mathcal{C}$ to cylindrical Wiener Process $\hat{W}_t$, we can define the $H$-valued $\mathcal{C}$-Wiener process $W_t$ as:
\begin{equation}
  W_t = \sum_{i=1}^{\infty} \sqrt{\lambda_i} \beta_i(t) \phi_i = \mathcal{C}^{1/2}\hat{W}_t
\end{equation}
where $\lambda_i$ are the eigenvalues of $\mathcal{C}$. 
So $H$ is actually the Cameron-Martin space of $H_1$.

To sum up, in infinite dimensional diffusion model, suppose we have $u\sim m_0, \eta\sim m_1$, where $m_0$ is the initial measure and $m_1$ is the target measure, all defined on $H$.
In the diffusion process, we add noise process $\eta$ to $u$, getting $v = u + \eta$. Then $v\sim \nu=m_0 * m_1$. Here we need to do some assumptions on measure $m_0$ and $m_1$: 
$m_0$ needs to be supported on the Cameron-Martin space $U$ of $m_1$, that is $m_0 \ll m_1$, $m_0$ should be more regular than $m_1$. 
Only in this way, we can define the Radon-Nikodym derivative of $\nu$ with respect to $m_1$, and the perturbed measure $\nu$ is equivalent to $m_1$.

Normally, to guarantee the diffusion progress, we have two choices: 1. The support of $m_{data}$ is contained in the Cameron-Martin space of $N(0, \mathcal{C}_{data})$ 2. Choose a large space $K$, s.t. the support of $m_{data}$ and $N(0, \mathcal{C})$ are supported on $K$.

\begin{definition}
    If A is a linear operator from $\mathcal{D}(A) \subset H$ to Hilbert space $H$, with an orthonormal basis of eigenfunctions $\{\phi_j\}$ 
    and corresponding increasing eigenvalues $\{\lambda_j\}$, 
    then $A^{\alpha}$ is defined as:
    \begin{equation}
        A^{\alpha}u = \sum_{j=1}^\infty \lambda_j^\alpha \langle u, \phi_j\rangle \phi_j
    \end{equation}
    and the domain $\mathcal{D}(A^{\alpha})$ is the set of all $u\in H$ such that $A^{\alpha}u\in H$.
\end{definition}

\begin{definition}[Absolute continuity of measures]\label{apdx_def}
Let $(X,\mathcal{F})$ be a measurable space, and let $\mu$ and $\nu$ be two measures on $(X,\mathcal{F})$.
We say that $\nu$ is \emph{absolutely continuous} with respect to $\mu$, and write
\[
\nu \ll \mu,
\]
if for every measurable set $A \in \mathcal{F}$, $\mu(A) = 0 \quad \Longrightarrow \quad \nu(A) = 0.$
\end{definition}

\begin{definition}\label{FourierTransform}
The Fourier transform of $u(x)$ in this paper is defined like this:
	\begin{equation}\left\{
		\begin{aligned}
			& \mathcal{F}u(k) = \int_{\mathbb{R}^d} u(x) e^{-ikx} dx,\\
			& \mathcal{F}^{-1}u(x) = \frac{1}{(2\pi)^{d}}\int_{\mathbb{R}^d} u(k) e^{ikx} dk.
		\end{aligned}\right.
	\end{equation}
\end{definition}

\subsubsection{Random Fields on $\mathbb{R}^d$}\label{apdx_sec_matern}
In this section, background knowledge on random fields is first reviewed, including their formulation, representation, and numerical treatment, with an emphasis on stationary random fields. Subsequently, more general classes of random fields beyond the Mat\'ern family are discussed, including non-stationary random fields that exhibit anisotropy or temporal dependence.
\subsubsection{Stationary Random Fields}\label{apdx:srf}
First we define stationary random field on $\mathbb{R}^d$ as:
\begin{definition}[Stationary Random Field]
  A second-order random field ${u(x): x\in D}$ is called stationary if the mean is constant and covariance function 
  depends only on the difference $x-y$, i.e. $\mu(x) = \mu,\ C(x, y) = C(x-y)$.
\end{definition}
Then we can define the covariance operator $\mathcal{C}$ as:
\begin{equation}
  \mathcal{C}\phi = \int_{\mathbb{R}^d} C(x-y)\phi(y)dy
\end{equation}
We find that it is actually the convolution operator of $C(x)$ with $\phi(x)$.
Stationary random fields have some beautiful properties.
\begin{theorem}[Wiener-Khinchin Theorem]
    There exists a stationary random field ${u(x): x\in D}$ with mean $\mu$ and covariance function $c(x)$ that is mean square continuous if and only if 
    the function $c(x): \mathbb{R}^d\rightarrow \mathbb{R}$ is such that 
    \begin{equation}
        c(x) = \frac{1}{(2\pi)^{d}}\int_{\mathbb{R}^d} e^{ik \cdot x}dF(k)=\frac{1}{(2\pi)^{d}}\int_{\mathbb{R}^d} e^{ik \cdot x}S(k)dk = \left(\mathcal{F}^{-1}S\right)(x)
    \end{equation}
    where $F(k)$ is some measure on $\mathbb{R}^d$ called spectral distribution and $\hat{S}(x)$ is the Fourier transform of $S(k)$, 
	called spectral density.
    Reversely, $S(k) = \left(\mathcal{F}c\right)(k) = \hat{c}(k)$.
    If $S(k)$ is non-negative and integrable, then $c(x)$ is a valid covariance function.
\end{theorem}

\begin{theorem}[Spectral Density of Random Field]\label{spectral_density_random_field}
	Assume $u(x)$ has zero mean, then 
\begin{equation}\label{spectraldensity}
	S_u(k)=\frac{1}{(2\pi)^{d}}\mathbb{E}[|\hat{u}(k)|^2]
\end{equation}
\end{theorem}


By defining the pseudo-differential operators, the class of SPDEs is defined by:
\begin{equation}
	\mathcal{L}_gu = W, \mathcal{L}_g = \mathcal{F}^{-1}g\mathcal{F}
\end{equation}
where $g:\mathbb{R}^d\rightarrow \mathbb{C}$ must be a sufficiently regular and Hermitian-symmetric function, that is it must satisfy: $g(k) = \overline{g(-k)}$, $\overline{\cdot}$ denotes the complex conjugate.
So if we have $\mathcal{L}_gu = W$, then:
\begin{equation}
	u=\mathcal{L}_{\frac{1}{g}}W
\end{equation}

\begin{theorem}
	The spectral density of $\mathcal{L}_gu$ and of $u$ are related by:
	\begin{equation}
		S_{\mathcal{L}_gu}(k) = \left|g(k)\right|^2S_u(k)
	\end{equation}
	Generally, if 
	\begin{equation}\label{SPDEGeRF}
		\mathcal{L}_gu = w
	\end{equation}
	where $w$ is a GeRF source term, then $S_w(k) = \left|g(k)\right|^2S_u(k)$.
	Therefore, when $w = W$, $S_u =\frac{1}{(2\pi)^{d}\left|g(k)\right|^2}\mathbb{E}[|\hat{W}(k)|^2]=\frac{1}{\left|g(k)\right|^2}$. 
  Then, 
\begin{equation}
	u(x) =\mathcal{L}_{\frac{1}{g}}w(x) = \mathcal{L}_{\sqrt{\frac{S_u}{S_w}}}w(x)
\end{equation}
\end{theorem}

Then consider the exitence of the function.
\begin{theorem}
	Let $w(x)$ be a real stationary GeRF over $\mathbb{R}^d$, and let $g:\mathbb{R}^d\rightarrow \mathbb{C}$ be a symbol function. 
	Then for Equ.~\eqref{SPDEGeRF}, there exists a unique stationary solution $u(x)$ if and only if:
	there exists $N\in \mathbb{N}$ s.t. 
	\begin{equation}
		\int_{\mathbb{R}^d}\frac{dS_w(k)}{\left|g(k)\right|^2(1+\|k\|^2)^N} < \infty
	\end{equation}
	and 
	\begin{equation}
		S_u(k) = \left|g(k)\right|^{-2}S_w(k)
	\end{equation}
	Moreover, $S_u(k)$ is unique if and only if $\left|g\right|>0$.
\end{theorem}


Hence the key is the symbol function $g(k)$.
The following theorem shows that solutions of SPDEs with White Noise source term is the starting point of more general solutions, when the source term can be any stationary GeRF.
\begin{theorem}\label{uniquenessandexistence}
	Let $w(x)$ be a real stationary GeRF over $\mathbb{R}^d$ with covariance distribution $C_w(x)$. 
	Let $g$ be a symbol function over $\mathbb{R}^d$ such that $\frac{1}{g}$ is smooth with polynomially bounded derivatives of all orders. 
	Then, there exists a unique stationary solution to (\ref{SPDEGeRF}) and its covariance distribution is given by
\begin{equation}
	C_u(x) = C_u^W * C_w(x)
\end{equation}
where $C_u^W$ is the covariance function of the solution to the SPDE with White Noise source term.
\end{theorem}
\begin{proof}
	The proof is straightforward by using Wiener-Khinchin theorem.
\end{proof}
For any precision operator which is a polynomial in the Laplacian, $Q = p(-\Delta)$, such as the Matern operator with $\nu \in \mathbb{N}$, 
this results in a polynomial $F(Q) = p(\|k\|^2)$.

Then we introduce the Mat\'ern field. 
\begin{proof}[Proof of Equ.~\eqref{eq:cov_matern}]
	Do Fourier transform on Equ.~\eqref{matern_field}:
\begin{equation}
	\left\{\mathcal{F}(\kappa^2 - \Delta)^{\alpha/2} u\right\}(k) = (\kappa^2 + \|k\|^2)^{\alpha/2}(\mathcal{F}u)(k)
\end{equation}
then we have 
\begin{equation}\label{spectralsolution}
	(\mathcal{F}u)(k) = \hat{u}(k) = \frac{\hat{W}(k)}{(\kappa^2 + \|k\|^2)^{\alpha/2}} 
\end{equation}
Therefore, u can be written as:
\begin{equation}
	u(x) = \mathcal{F}^{-1}\left[\frac{\hat{W}(k)}{(\kappa^2 + \|k\|^2)^{\alpha/2}}\right]
\end{equation}
Then the stationary covariance function of u is given by:
\begin{equation}
	c(x) = Cov(u(x), u(0))
\end{equation}
By the definition of spectral density Equ.~\eqref{spectraldensity} and Equ.~\eqref{spectralsolution} we have:
\begin{equation}\label{Suk}
		S_u(k) = \frac{1}{(2\pi)^{d}}\frac{\mathbb{E}\left[\left|\hat{W}(k)\right|^2\right]}{(\kappa^2 + \|k\|^2)^{\alpha}} 
		= \frac{1}{(\kappa^2 + \|k\|^2)^{\alpha}} 
\end{equation}
Then we have the variance of u:
\begin{equation}
	c(0) = \frac{1}{(2\pi)^{d}}\int_{\mathbb{R}^d} S_u(k) dk = \frac{\Gamma(\nu)}{(4\pi)^{d/2}\kappa^{2\nu}\Gamma(\alpha) }:=\sigma^2
\end{equation}
	By Wiener-Khinchin theorem, we have:
\begin{equation}
	\begin{aligned}
		c(x) &= (\mathcal{F}^{-1}S_u)(x)=\mathcal{F}^{-1}\left[\frac{1}{(\kappa^2 + \|k\|^2)^{\alpha}}\right]\\
		&= \frac{\|x\|^\nu K_{\nu}(\kappa\|x\|)}{(4\pi)^{d/2}2^{\nu-1}\kappa^{\nu}\Gamma(\alpha)}\\
		& = \frac{\sigma^2}{2^{\nu -1}\Gamma(\nu)}(\kappa \|x\|)^\nu K_\nu (\kappa \|x\|)
	\end{aligned}
\end{equation}
\end{proof}
Here we show some numerical cases of Mat\'ern Field in Fig.~\ref{apdx_pic: matern_case}. It is known that $\nu$ and $\ell$ controls the regularity and the correlated length. 
\begin{figure}[ht] 
  \centering
    \includegraphics[width=\textwidth]{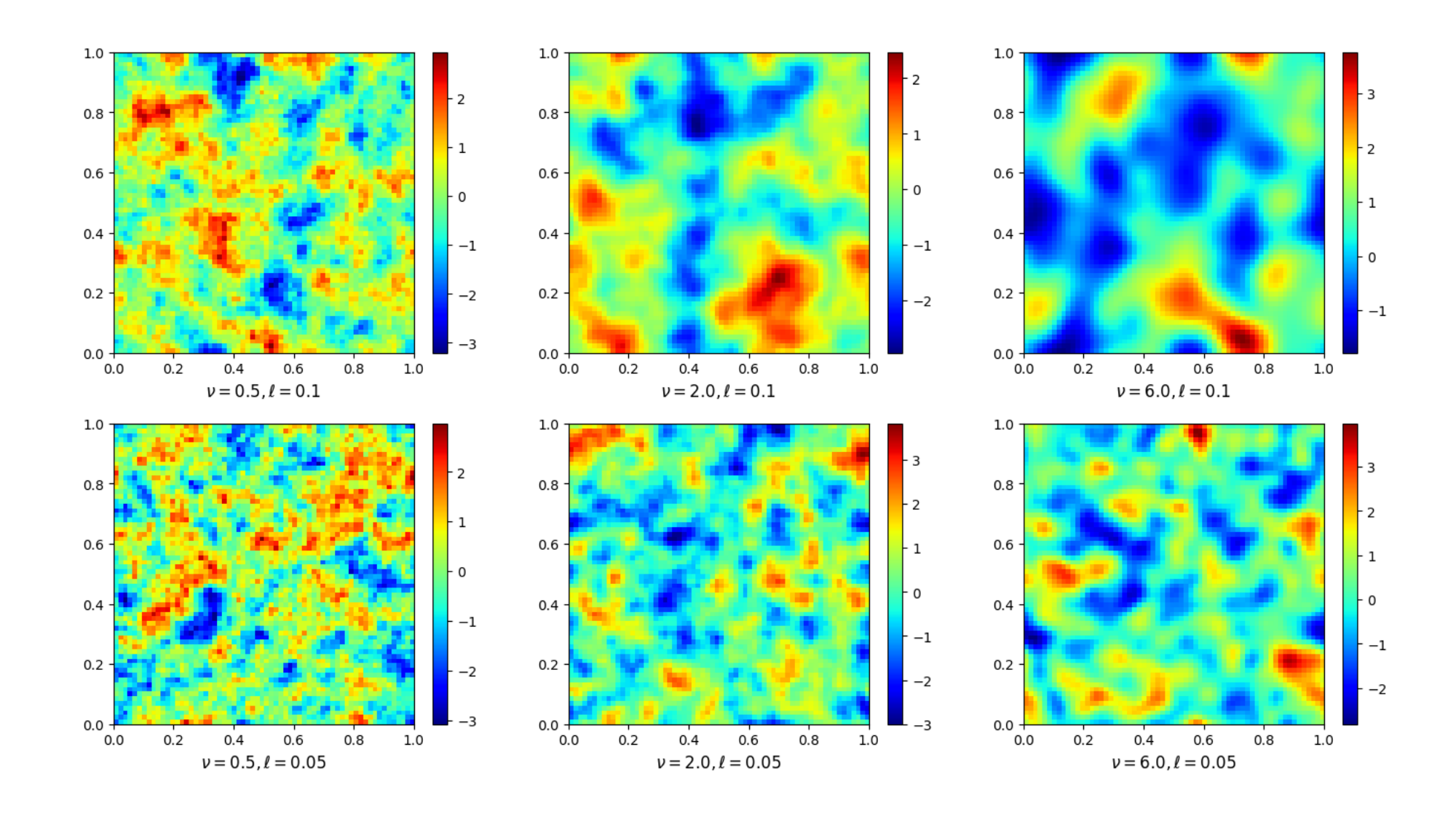}
\caption{Matern Fields with different $\nu,\ell$
}
\label{apdx_pic: matern_case}
\end{figure}

The smoothness parameter $\nu$ controls the regularity of the random field.
Larger values of $\nu$ correspond to smoother sample paths, for which the spectral
energy decays more rapidly with respect to frequency, resulting in lower energy
content in the high-frequency regime.
Conversely, smaller values of $\nu$ lead to rougher fields with increased
high-frequency components and more pronounced local fluctuations.
The correlation length parameter $\ell$ governs the spatial dependence structure
of the field.
Larger values of $\ell$ imply stronger correlations between distant points, giving
rise to coherent structures over larger spatial regions, whereas smaller values of
$\ell$ corresponds to rapidly decaying correlations and nearly independent local
behavior.
From a spectral perspective, $\ell$ primarily determines the frequency scale at
which the spectral energy begins to decay. In Fig.~\ref{apdx_pic: matern_cov}, the spectral energy of different Mat\'er fields is presented. It can be observed that a smaller correlation length $\ell$ leads to a faster decay of the covariance function, whereas stronger regularity $\nu$ results in a flatter behavior of the covariance function near zero.
\begin{figure}[ht] 
  \centering
    \includegraphics[width=\textwidth]{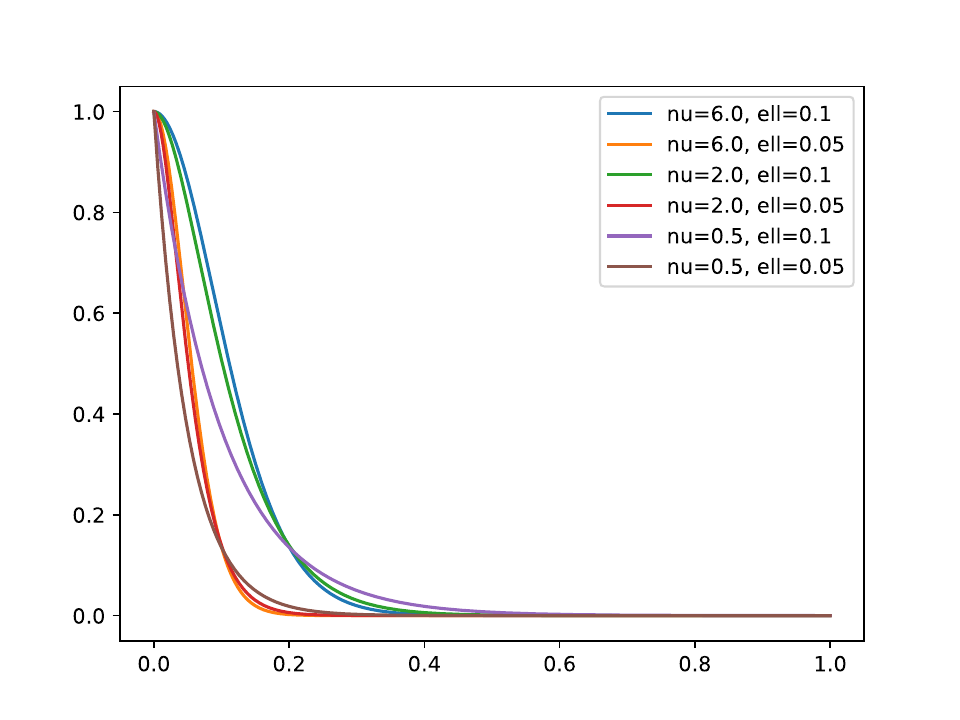}
\caption{Covariance of Matern Fields with different $\nu,\ell$
}
\label{apdx_pic: matern_cov}
\end{figure}
\paragraph{Remark on anisotropic noise measures.}
The proposed framework can also accommodate anisotropic source noise measures.
For example, one may construct Gaussian random fields with direction-dependent spectral decay or correlation lengths to match anisotropic target measures.
However, such choices typically require additional problem-specific prior knowledge about dominant directions or anisotropic correlations.
To keep the method broadly applicable and avoid introducing benchmark-specific design choices, we restrict our experiments to isotropic Mat\'ern-type noise measures.
Extending the noise design principle to anisotropic random fields is an interesting direction for future work.

\subsection{Solutions of SPDEs} \label{apdx_spde}
It should be noted that, in the case of stationary SPDEs, the learned SPDE solver effectively functions as an operator mapping between function spaces under randomness. 
For example, a stationary SPDE can be formulated as follows. Assume that $f \in L^2(\Omega, L^2(D))$ is a random field distributed according to the measure $\mu$. The objective is to find $u:D\times \Omega \to \mathbb{R}$ in the weak sense such that, $\mathbb{P}$-almost surely:
\begin{equation}\left\{
\begin{aligned}
& \mathcal{L}\left(u(x, w)\right) = f(x, w), \quad &x\in D, \\
& u(x, w) = g(x), \quad &x\in \partial D,
\end{aligned}\right.
\label{stationaryspde}
\end{equation}
where $\mathcal{L}$ is a stationary differential operator.
Therefore, stationary SPDEs are conceptually aligned with neural operators in that both aim to learn mappings between function spaces.

A function $u : [0,T] \to \mathcal{H}_U$ is referred to as a \emph{mild solution}
of the semilinear evolution equation if it satisfies the following integral equation:
\begin{equation}
u(t)
=
S(t)u_0
+
\int_0^t S(t-s)
\bigl[
F(u_s)\,\mathrm{d}s
+
G(u_s)\,\mathrm{d}W_s
\bigr],
\end{equation}
where $S(t) = e^{t\mathcal{L}}$ is the semigroup generated by $\mathcal{L}$.
The following theorem provides the theoretical guarantee of existence and uniqueness
for solutions to SPDEs.

\begin{theorem}\label{condition_invariant}
Assuming the global Lipschitz regularity on $F$ and $G$ are satisfied,
the mild solution to Eq.~(\ref{SPDE}) exists and is unique.
\end{theorem}

Usually, we care about the asymptotic properties of solutions of stochastic
equations as time goes to infinity \citep{Da_Prato_Zabczyk_2014}.

\begin{theorem}
Under the assumption of Theorem~{\ref{condition_invariant}}, and the following dissipativity condition:
there exists $C>0$ such that $\forall x,y \in \mathcal{H}$, $n \in \mathbb{N}$,
\begin{equation}
2\langle \mathcal{L}_n(x-y), x-y \rangle
+
2\langle F(x)-F(y), x-y \rangle
+
\|G(x)-G(y)\|^2
\le
- C \|x-y\|^2 ,
\tag{7}
\end{equation}
where $\mathcal{L}_n = n\mathcal{L}(n - \mathcal{L})^{-1}$
denotes the Yosida approximation of $\mathcal{L}$.
Then there exists a unique invariant measure $\mu$ associated with Eq.~(\ref{SPDE}).

\end{theorem}

%% file: src/2_Related.tex
\section{Related Work}\label{sec_related_work}
In recent years, both operator-learning models and generative modeling techniques have advanced rapidly. We begin by briefly reviewing existing research on neural operators and their applications to modeling SPDEs. We then discuss recent progress in generative modeling within scientific computing and highlight how our approach differs from prior work in both areas.
\subsection{Neural Solvers}
Numerous neural solvers have been proposed to approximate the solutions or operators of PDEs and SPDEs using neural networks. \citet{raissi2019physics} introduced Physics-Informed Neural Networks (PINNs), in which a neural network is optimized using PDE constraints as self-supervised losses to obtain PDE solutions. These methods have also been extended to solve SPDEs \citep{zhang2020learning, neufeld2024solving}, although they are difficult to generalize across different types of SPDEs. Among neural operators, the Fourier Neural Operator (FNO) \citep{kovachki2023neural, li2021fourierneuraloperatorparametric} and DeepONet \citep{Lu_2021} are among the most widely used due to their representative architectures, which were designed with universal approximation properties. The primary objective of neural operators is to learn mappings between distributions in function spaces. 
Many other related operator-based approaches have been proposed to address specific challenges, such as irregular domains and data \citep{liu2025aerogto, liu2025efficient}, as well as long-term forecasting \citep{liu2024papm, liu2024deep}. These methods have demonstrated strong performance across a variety of tasks; however, their applicability remains largely restricted to PDE-related problems.

By incorporating time-dependent randomness, Neural SPDE \citep{salvi2022neural} and DLRNet \citep{gong2023deep, hu2022neuraloperatorregularitystructure, hu2025modelbasedclosedloopcontrolalgorithm} have been proposed for modeling SPDE operators. The distinction between modeling SPDE operators and PDE operators lies in the fact that the former accounts for infinite-dimensional, time-dependent randomness, whereas the latter does not. In the work of \citet{salvi2022neural}, kernel-based methods were employed to handle time evolution by learning the semigroup operator generated by the differential operator. In the work of \citet{gong2023deep}, the regularity features of SPDEs were combined with neural networks to improve generalization; however, the computational cost becomes prohibitively high for high-order signatures. 

However, in real-world SPDE dynamics, the driving noise is rarely observable; only system states at selected spatial and temporal points are typically available. Moreover, numerical solvers require extremely fine temporal discretization, making full-trajectory prediction computationally costly. In this work, we therefore focus not on forecasting the SPDE evolution, but on sampling from its stationary distribution once equilibrium is reached. For example, in thin-film growth described by the EW/KPZ equations under fixed operating conditions, our interest lies in the stabilized surface statistics rather than the transient dynamics. Importantly, all objects of interest reside in function spaces, making operator-learning models essential for accurately capturing the underlying infinite-dimensional structure.

\subsection{Generative Models}

In recent years, rapid progress has been made in generative modeling, including DDPM~\citep{ho2020denoising}, DDIM~\citep{songdenoising}, score-based models~\citep{songscore}, and flow matching~\citep{lipmanflow}. Beyond finite-dimensional settings, these ideas have been extended to infinite-dimensional generative models that aim to directly represent distributions over function spaces. For example, \citet{pidstrigach2024infinite, lim2023score, kerrigan2022diffusion, franzese2023continuous} propose score- or diffusion-based approaches that are invariant to discretization and operate directly in Hilbert or Banach spaces, thereby enabling a variety of applications in scientific machine learning.
Although FFM~\citep{kerrigan2023functional} establishes functional flow matching, it uses Gaussian sources and does not empirically isolate source design. Other functional models adapt GAN or VAE architectures \citep{na2025probability, rahman2022generative, seidman2023variational}, while recent work designs interpolants or conditional velocity fields \citep{albergo2025stochastic, chen2025scale, chen2025lipschitzguideddesigninterpolationschedules, boffi2025flowmapmatchingstochastic}. Functional Mean Flow targets one-step generation in Hilbert space \citep{li2025functionalmeanflow}, and FunDPS develops guided diffusion sampling and an infinite-dimensional Tweedie formula for PDE inverse problems \citep{yao2025guided}. NAGO is complementary: it isolates the source measure and links its continuum admissibility and finite-band covariance ratio to spectral stiffness.

Source design has also emerged in finite-dimensional flow matching. LeDiFlow learns a data-dependent image prior to shorten probability paths \citep{zwick2025lediflow}; Better Source, Better Flow learns condition-dependent text-to-image sources with variance and directional regularization \citep{kim2026bettersource}; and Optimal Transport Flow Matching by Design constructs a low-frequency image prior that admits an approximately optimal identity coupling \citep{malnick2026optimal}. These works establish that source distributions matter beyond our setting. Our contribution is narrower and different: an analytic, admissibility-constrained source rule for probability measures on function spaces, with a spectral-stiffness analysis and cross-resolution SPDE evaluation.

A crucial design choice in diffusion-based generative models is the noise schedule, i.e., how noise levels are varied during the forward and reverse processes. Prior work has shown that the noise schedule significantly affects model behavior: for example, different noise schedules can shift optimal performance and robustness across tasks and resolutions, and the choice of schedule often needs to be tailored rather than fixed \citep{chen2023importancenoiseschedulingdiffusion}. Subsequent research has undertaken both theoretical and empirical analysis of noise schedules, proposing new metrics and schedule designs to improve generation quality and training dynamics \citep{hang2025improved,guo2023rethinking}. Moreover, studies focusing on specific application domains (e.g., design plausibility in conditional diffusion) have explored parametric noise scheduling to balance generative fidelity and structural constraints \citep{fan2023noise}. Together, these works highlight that noise choice and scheduling are fundamental for stable, high-quality generation in diffusion models. 

\subsection{Difference between finite and infinite dimensional generative models.} \label{apdx_finite_infinite}
Generative models differ fundamentally depending on whether they are formulated in finite-dimensional Euclidean spaces or in infinite-dimensional function spaces.
In finite-dimensional settings, the target distribution is defined on a fixed space $\mathbb{R}^d$, where the dimension $d$ is prescribed by the chosen representation.
For example, an image model may operate on a fixed pixel grid, while a latent-variable model may operate on a fixed-dimensional feature vector.
In this regime, probability distributions can often be described through densities with respect to the Lebesgue measure, and generative dynamics such as diffusion processes, score-based models, normalizing flows, or flow matching can be formulated as stochastic or deterministic differential equations on $\mathbb{R}^d$.
The associated analysis typically relies on finite-dimensional calculus, including gradients, divergences, Jacobians, change-of-variables formulas, and numerical discretizations of ordinary or stochastic differential equations.

Function-space generative modeling has a different mathematical structure.
Here the target object is a probability measure on a space of functions or fields, such as a Hilbert space, a Banach space, or a Sobolev-type space.
Samples are not merely finite vectors, but discretized observations of continuum objects.
Typical examples include Gaussian random fields, solutions of stochastic partial differential equations, invariant measures of stochastic dynamics, and continuous physical fields.
This change of state space affects the basic theoretical language of generative modeling.
In infinite-dimensional spaces, there is no translation-invariant Lebesgue reference measure.
Consequently, densities, likelihoods, and scores cannot in general be interpreted in the same way as in $\mathbb{R}^d$.
Instead, the construction of generative dynamics must be formulated at the level of probability measures, operators, covariance structures, and infinite-dimensional evolution equations.
Questions such as existence, uniqueness, regularity, and stability become central, since the flow must be well defined as a continuum object rather than only after a finite discretization.

The second major difference concerns the role of spatial resolution.
In a finite-dimensional model, the resolution is usually part of the definition of the data representation.
Changing the resolution changes the ambient dimension and therefore typically defines a different finite-dimensional modeling problem.
For example, an image model trained on $64\times 64$ images and another model trained on $256\times 256$ images are usually treated as models on different Euclidean spaces.
By contrast, in function-space generative modeling, different resolutions should be understood as different discretizations of the same underlying continuum measure.
A model that is genuinely defined in function space should therefore be discretization-consistent: as the grid is refined, the learned distribution and the induced generative dynamics should converge to a well-defined limiting object.
Equivalently, the model should not rely on artifacts of a particular grid size.

This distinction is especially important for multiscale SPDE data.
The target invariant measures often have nontrivial spatial regularity and strongly nonuniform spectral structure.
If the source noise is chosen without regard to this regularity, its discretized high-frequency content may dominate the induced transport dynamics.
For instance, standard discretized Gaussian noise can inject substantial energy into modes that are not representative of the target measure.
The learned velocity field must then spend significant capacity suppressing these unresolved high-frequency components.
This can lead to spectrally stiff marginal velocity fields, unstable training, and generation behavior that changes with the spatial resolution.
Thus, in function-space generative modeling, the source noise is not merely an implementation detail.
It directly affects whether the learned flow is well conditioned, resolution robust, and compatible with the target continuum measure.

This perspective motivates the regularity-aware noise design considered in this work.
Rather than fixing the source distribution to be standard Gaussian noise at every discretization level, we choose a noise measure whose regularity is matched to that of the target invariant measure.
The goal is to reduce high-frequency mismatch at the level of the continuum formulation, thereby producing better-conditioned marginal velocity fields and more stable discretization-consistent generation.

\section{Invariant measures}\label{apdx:intro}
In many engineering and scientific applications, the quantities of interest are not deterministic functions but uncertain spatial fields.
For example, the permeability of a subsurface medium may vary sharply across space and cannot be fully determined from limited observations.
Representing such a parameter by a single fixed function therefore ignores the uncertainty that is essential for prediction and decision-making.
A more appropriate formulation is to model the parameter as a random field and characterize its distribution through a probability measure on a function space.
In SPDE-based modeling, the invariant measure of a stochastic dynamics provides such a statistical description: it encodes the long-term spatial structure, variability, and correlations of the field.
Sampling from this invariant measure then produces statistically plausible realizations of the unknown medium, which can be used for uncertainty propagation, inverse problems, risk assessment, and design optimization.

\paragraph{\textbf{Why are multiscale invariant measures challenging for functional generative models?}}
Multiscale SPDE invariant measures often contain spatial structures across a wide range of frequencies.
A functional generative model must therefore learn not only the distribution of samples, but also a resolution-consistent transport mechanism between probability measures on function spaces.
In practice, many implementations use standard Gaussian noise as the source distribution.
This choice is convenient, but it can be poorly matched to the target invariant measure.
When the source noise is much rougher than the data, it injects excessive high-frequency energy into the generative path.
The induced marginal velocity field must then transport rough noise into smoother target fields, which creates strong spectral imbalance and leads to stiff dynamics.
This stiffness can make training unstable and sampling expensive, especially when the spatial resolution changes.

\paragraph{\textbf{Why are interpolation schedules or adaptive solvers not enough?}}
Existing methods often try to improve stability through interpolation schedules \citep{tong2024improving, chen2024adaptive}, guidance strategies, or specialized ODE solvers.
These techniques can be useful, but they mainly act on the generative path or on the numerical integration after the source and target measures have already determined the geometry of the flow.
If the initial noise measure is strongly mismatched with the target regularity, the resulting marginal velocity field is already poorly conditioned.
Downstream heuristics may reduce some symptoms of stiffness, but they do not remove its source.
As a result, they may require additional tuning, increase computational cost, or fail to remain robust across discretization resolutions.

\paragraph{\textbf{Our perspective: design the noise measure in function space.}}
We instead treat the source noise measure as a design variable.
The central idea is that the geometry of the functional flow should be controlled at the level of probability measures, rather than corrected only through downstream mechanisms.
By choosing a noise distribution whose regularity is compatible with that of the target invariant measure, we can reduce high-frequency mismatch before it is inherited by the marginal velocity field.
This leads to better-conditioned flows, more stable training, and improved cross-resolution behavior.

%% file: src/7_Appendix_2.tex
\section{Proofs of Main Theoretical Results}
\label{apdx:proofs}

\subsection{Proof of Theorem~\ref{thm:3_1}}
\label{pf:3.1}

\begin{proof}
Let $(\mathcal{G},\mathcal{B}(\mathcal{G}))$ and
$(\mathcal{U},\mathcal{B}(\mathcal{U}))$ be standard Borel spaces.
For each $t\in[0,1]$, let $\Gamma_t$ be a joint probability measure on
$\mathcal{G}\times\mathcal{U}$ with first marginal $\eta_t$ on $\mathcal{G}$ and
second marginal $\mu_1$ on $\mathcal{U}$.
By the disintegration theorem, there exists a regular conditional distribution
$\Gamma_t(\mathrm{d}g\mid f)$ such that
\begin{equation}
    \Gamma_t(\mathrm{d}g,\mathrm{d}f)
    =
    \Gamma_t(\mathrm{d}g\mid f)\,\mu_1(\mathrm{d}f).
\end{equation}
Assume that for $\mu_1$-a.e. $f\in\mathcal{U}$,
\begin{equation}
    \Gamma_t(\cdot\mid f)\ll \eta_t .
\end{equation}
Then the Radon--Nikodym theorem gives a measurable density
\begin{equation}
    \rho_t(g,f)
    :=
    \frac{\mathrm{d}\Gamma_t(\cdot\mid f)}{\mathrm{d}\eta_t}(g),
\end{equation}
so that
\begin{equation}
    \Gamma_t(\mathrm{d}g\mid f)
    =
    \rho_t(g,f)\,\eta_t(\mathrm{d}g)
    \qquad
    \text{for } \mu_1\text{-a.e. } f .
\end{equation}

Let $v_t(g\mid f)$ be a conditional velocity field satisfying the integrability condition
\begin{equation}
    \int_0^1
    \int_{\mathcal{U}}
    \int_{\mathcal{G}}
    \|v_t(g\mid f)\|^2
    \,
    \Gamma_t(\mathrm{d}g\mid f)
    \mu_1(\mathrm{d}f)
    \mathrm{d}t
    <
    \infty .
\end{equation}
Define the marginal velocity field by
\begin{equation}
    v_t(g)
    :=
    \int_{\mathcal{U}}
    v_t(g\mid f)\rho_t(g,f)\,\mu_1(\mathrm{d}f),
    \qquad
    \eta_t\text{-a.e. }g .
    \label{eq:marginal_velocity_rn}
\end{equation}
The above integral is understood in the Bochner sense.

For any bounded measurable test function $\varphi:\mathcal{G}\to\mathbb{R}$,
Fubini's theorem gives
\begin{align}
    \int_{\mathcal{G}} \varphi(g)v_t(g)\,\eta_t(\mathrm{d}g)
    &=
    \int_{\mathcal{U}}
    \int_{\mathcal{G}}
    \varphi(g)v_t(g\mid f)\rho_t(g,f)\,
    \eta_t(\mathrm{d}g)\mu_1(\mathrm{d}f) \notag\\
    &=
    \int_{\mathcal{U}}
    \int_{\mathcal{G}}
    \varphi(g)v_t(g\mid f)\,
    \Gamma_t(\mathrm{d}g\mid f)\mu_1(\mathrm{d}f).
\end{align}
Therefore $v_t(g)$ is the conditional average of $v_t(g\mid f)$ under the disintegrated measure, and admits the representation
\begin{equation}
    v_t(g)
    =
    \int_{\mathcal{U}}
    v_t(g\mid f)
    \frac{\mathrm{d}\Gamma_t(\cdot\mid f)}{\mathrm{d}\eta_t}(g)
    \,\mu_1(\mathrm{d}f),
    \qquad
    \eta_t\text{-a.e. }g .
\end{equation}
This proves the claim.
\end{proof}

\subsection{Proof of Theorem~\ref{thm:3_2}}
\label{pf:3.2}

\begin{proof}
Let $\eta_t$ denote the marginal measure induced by the mixture
\begin{equation}
    \eta_t(A)
    =
    \int_{\mathcal{U}} \Gamma_t(A\mid f)\,\mu_1(\mathrm{d}f),
    \qquad
    A\in\mathcal{B}(\mathcal{G}).
    \label{eq:marginal_measure_mixture}
\end{equation}
By assumption, there exists a measurable set $G\subset\mathcal{U}$ with
$\mu_1(G)=1$ such that for every $f,h\in G$,
\begin{equation}
    \Gamma_t(\cdot\mid f)
    \sim
    \Gamma_t(\cdot\mid h).
\end{equation}

We prove that $\Gamma_t(\cdot\mid f)\ll \eta_t$ for every $f\in G$.
Let $A\in\mathcal{B}(\mathcal{G})$ satisfy $\eta_t(A)=0$.
Since $\mu_1(G)=1$, the mixture identity gives
\begin{equation}
    0
    =
    \eta_t(A)
    =
    \int_G \Gamma_t(A\mid h)\,\mu_1(\mathrm{d}h).
\end{equation}
Because the integrand is nonnegative, it follows that
\begin{equation}
    \Gamma_t(A\mid h)=0
\end{equation}
for $\mu_1$-a.e. $h\in G$.
Choose one such $h_0\in G$.
For any fixed $f\in G$, the equivalence
$\Gamma_t(\cdot\mid f)\sim \Gamma_t(\cdot\mid h_0)$ implies
\begin{equation}
    \Gamma_t(A\mid f)=0 .
\end{equation}
Hence $\Gamma_t(\cdot\mid f)\ll \eta_t$ for every $f\in G$.
This proves the theorem.
\end{proof}

\subsection{Proof of Theorem~\ref{thm_cm_space}}
\label{pf:3.3}

\begin{proof}
We present the proof on the periodic domain $\mathbb{T}^d$ to avoid irrelevant boundary issues.
The same argument extends to bounded domains after replacing Fourier modes by the eigenfunctions of an appropriate elliptic operator.

Let $u$ be a centered stationary Mat\'ern Gaussian field with covariance spectrum
\begin{equation}
    \lambda_k
    \asymp
    \left(\kappa^2+4\pi^2|k|^2\right)^{-(\nu+d/2)},
    \qquad
    k\in\mathbb{Z}^d ,
    \label{eq:matern_spectrum_torus}
\end{equation}
where $\kappa>0$ and $\nu>0$.
Then $u$ admits the Fourier representation
\begin{equation}
    u
    =
    \sum_{k\in\mathbb{Z}^d}
    \sqrt{\lambda_k}\,\xi_k e_k,
    \qquad
    \xi_k\sim\mathcal{N}(0,1),
\end{equation}
where $\{e_k\}_{k\in\mathbb{Z}^d}$ is the Fourier basis.

For $s\in\mathbb{R}$, the Sobolev norm on $\mathbb{T}^d$ is
\begin{equation}
    \|u\|_{H^s}^2
    =
    \sum_{k\in\mathbb{Z}^d}
    (1+|k|^2)^s |\widehat u_k|^2 .
\end{equation}
Taking expectation and using the covariance spectrum gives
\begin{align}
    \mathbb{E}\|u\|_{H^s}^2
    &=
    \sum_{k\in\mathbb{Z}^d}
    (1+|k|^2)^s \lambda_k \notag\\
    &\asymp
    \sum_{k\in\mathbb{Z}^d}
    (1+|k|^2)^s
    \left(\kappa^2+4\pi^2|k|^2\right)^{-(\nu+d/2)} .
\end{align}
For large $|k|$, the summand behaves like
\begin{equation}
    |k|^{2s-2\nu-d}.
\end{equation}
The series over $\mathbb{Z}^d$ converges if and only if
\begin{equation}
    2s-2\nu-d<-d,
\end{equation}
which is equivalent to
\begin{equation}
    s<\nu .
\end{equation}
Thus,
\begin{equation}
    \mathbb{E}\|u\|_{H^s}^2<\infty
    \qquad
    \text{for all } s<\nu .
\end{equation}
Since $\|u\|_{H^s}^2$ is a nonnegative random variable, this implies
\begin{equation}
    u\in H^s
    \qquad
    \text{a.s. for all } s<\nu .
\end{equation}

It remains to identify the Cameron--Martin space.
Let $\mu=\mathcal{N}(0,C)$ be the Gaussian measure induced by $u$.
The Cameron--Martin space is
\begin{equation}
    \mathcal{H}_{\mathrm{CM}}
    =
    \operatorname{Ran}(C^{1/2}),
    \qquad
    \|h\|_{\mathcal{H}_{\mathrm{CM}}}
    =
    \|C^{-1/2}h\|_{L^2}.
\end{equation}
In Fourier variables, $C$ acts by multiplication with $\lambda_k$.
Therefore
\begin{equation}
    \|h\|_{\mathcal{H}_{\mathrm{CM}}}^2
    =
    \sum_{k\in\mathbb{Z}^d}
    \lambda_k^{-1}|\widehat h_k|^2 .
\end{equation}
Using \eqref{eq:matern_spectrum_torus},
\begin{align}
    \|h\|_{\mathcal{H}_{\mathrm{CM}}}^2
    &\asymp
    \sum_{k\in\mathbb{Z}^d}
    \left(\kappa^2+4\pi^2|k|^2\right)^{\nu+d/2}
    |\widehat h_k|^2 \notag\\
    &\asymp
    \|h\|_{H^{\nu+d/2}}^2 .
\end{align}
Hence
\begin{equation}
    \mathcal{H}_{\mathrm{CM}}
    =
    H^{\nu+d/2}
\end{equation}
as sets, with equivalent norms.
This completes the proof.
\end{proof}

\subsection{Proof of Theorem~\ref{thm:3_4}}
\label{pf:3.4}

\begin{proof}
Let $g,f$ be independent centered Gaussian random elements in a separable Hilbert space
$\mathcal{H}$ with covariance operators $Q_0$ and $Q_1$, respectively.
Let
\begin{equation}
    y:=f_t=\alpha_t g+\beta_t f,
    \qquad
    \dot f_t=\dot\alpha_t g+\dot\beta_t f .
\end{equation}
Since $(g,f)$ is jointly Gaussian and $(y,\dot f_t)$ is obtained from $(g,f)$ by a linear transformation,
$(y,\dot f_t)$ is jointly Gaussian in $\mathcal{H}\times\mathcal{H}$.
Therefore the conditional expectation $\mathbb{E}[\dot f_t\mid y]$ is linear in $y$:
\begin{equation}
    \mathbb{E}[\dot f_t\mid y]
    =
    \operatorname{Cov}(\dot f_t,y)
    \operatorname{Cov}(y)^{-1}y ,
\end{equation}
where the inverse is understood on the range of $\operatorname{Cov}(y)$.

Using independence of $g$ and $f$, we compute
\begin{align}
    \operatorname{Cov}(y)
    &=
    \alpha_t^2 Q_0+\beta_t^2 Q_1,\\
    \operatorname{Cov}(\dot f_t,y)
    &=
    \dot\alpha_t\alpha_t Q_0+\dot\beta_t\beta_t Q_1 .
\end{align}
Thus,
\begin{equation}
    v_t(x)
    :=
    \mathbb{E}[\dot f_t\mid f_t=x]
    =
    \left(
        \dot\alpha_t\alpha_t Q_0+\dot\beta_t\beta_t Q_1
    \right)
    \left(
        \alpha_t^2 Q_0+\beta_t^2 Q_1
    \right)^{-1}x .
    \label{eq:velocity_operator_general}
\end{equation}

In the simultaneously diagonalizable setting of the theorem, the relative
covariance operator $Q_1Q_0^{-1}$ acts modewise on the relevant spectral
subspace. Thus \eqref{eq:velocity_operator_general} can be written as
\begin{equation}
    v_t(x)=V(t)x,
\end{equation}
with
\begin{equation}
    V(t)
    =
    \left(
        \dot\alpha_t\alpha_t I
        +
        \dot\beta_t\beta_t Q_1Q_0^{-1}
    \right)
    \left(
        \alpha_t^2 I
        +
        \beta_t^2 Q_1Q_0^{-1}
    \right)^{-1}.
    \label{eq:velocity_operator_relative}
\end{equation}
This is the claimed formula.
\end{proof}

\subsection{Proof of Theorem~\ref{thm:noise_condition}}
\label{apdx_proof_noise_condition}

\begin{proof}
Let $f\sim\mu$, and denote by $\widehat f_k$ its Fourier coefficients.
By definition of the spectral density,
\begin{equation}
    \mathbb E |\widehat f_k|^2 = S_d(k).
\end{equation}
For $s\in\mathbb R$, the squared Sobolev norm satisfies
\begin{equation}
    \mathbb E \|f\|_{H^s}^2
    =
    \sum_{k\in\mathbb Z^d}
    (1+|k|^2)^s S_d(k).
\end{equation}
Since $S_d(k)\asymp |k|^{-\beta}$ as $|k|\to\infty$, the high-frequency
part of the above series is comparable to
\begin{equation}
    \sum_{k\in\mathbb Z^d}
    |k|^{2s-\beta}.
\end{equation}
Using the standard lattice-shell estimate
\[
    \#\{k\in\mathbb Z^d: n\le |k|<n+1\}\asymp n^{d-1},
\]
this series has the same convergence behavior as
\begin{equation}
    \sum_{n=1}^{\infty}
    n^{d-1+2s-\beta}.
\end{equation}
Therefore it converges if and only if
\begin{equation}
    d-1+2s-\beta<-1,
\end{equation}
or equivalently,
\begin{equation}
    s<\frac{\beta-d}{2}.
\end{equation}
Hence, for every $s<s_{\mathrm{data}}$, where
\[
    s_{\mathrm{data}}:=\frac{\beta-d}{2},
\]
we have
\begin{equation}
    \mathbb E\|f\|_{H^s}^2<\infty.
\end{equation}
In particular,
\begin{equation}
    f\in H^s
    \qquad
    \mu\text{-a.s. for every }s<s_{\mathrm{data}}.
    \label{eq:data_regular_threshold}
\end{equation}
In the Gaussian case with nondegenerate Fourier modes, the converse also
holds: $f\notin H^s$ almost surely for $s\ge s_{\mathrm{data}}$. Thus
$s_{\mathrm{data}}$ is the sharp Sobolev regularity threshold of the target
field.

Now consider the source Mat\'ern Gaussian noise with smoothness parameter
$\nu_0$. Its Cameron--Martin space is
\begin{equation}
    \mathcal H_{\mathrm{CM}}^{(0)}
    =
    H^{\nu_0+d/2}.
\end{equation}
To guarantee that typical target samples belong to the source
Cameron--Martin space, it is enough to require
\begin{equation}
    \nu_0+\frac d2 < s_{\mathrm{data}}.
\end{equation}
Indeed, under this condition, \eqref{eq:data_regular_threshold} gives
\[
    f\in H^{\nu_0+d/2}
    =
    \mathcal H_{\mathrm{CM}}^{(0)}
    \qquad
    \mu\text{-a.s.}
\]
Substituting $s_{\mathrm{data}}=(\beta-d)/2$ yields
\begin{equation}
    \nu_0+\frac d2
    <
    \frac{\beta-d}{2}.
\end{equation}
Equivalently,
\begin{equation}
    \nu_0
    <
    \frac{\beta}{2}-d.
\end{equation}
This proves the claimed condition.
\end{proof}

\subsection{Proof of Theorem~\ref{thm:3_6}}
\label{apdx:pf_3_6}
\begin{proof}
Let $\mu_k$ denote the positive spectral covariance ratio between the target
and source measures at Fourier mode $k$. For the linear interpolation,
write
\[
    q_k(t)=(1-t)^2+\mu_k t^2,
    \qquad
    \widetilde V_k(t)=\frac{q_k'(t)}{2q_k(t)}.
\]
The extrema of $\widetilde V_k$ on $[0,1]$ give
\begin{equation}
    \sup_{t\in[0,1]}|\widetilde V_k(t)|
    =
    \frac12
    \left(
        \sqrt{\mu_k}
        +
        \frac{1}{\sqrt{\mu_k}}
    \right).
    \label{eq:vmin_mu}
\end{equation}
Indeed, when $0<\mu_k<1$, the minimum occurs at
$t=(1-\sqrt{\mu_k})/(1+\mu_k)$ and equals the negative of the
right-hand side of \eqref{eq:vmin_mu}; the maximum is $1$ at $t=1$.
When $\mu_k>1$, the minimum is $-1$ at $t=0$, while the maximum occurs at
$t=(1+\sqrt{\mu_k})/(1+\mu_k)$ and equals the right-hand side.
For $\mu_k=1$, the maximum absolute value is $1$ at either endpoint.
Fix a stiffness threshold $L>1$. The uniform magnitude bound
\[
    |\widetilde V_k(t)|<L,\qquad t\in[0,1],
\]
is equivalent to
\[
    \sqrt{\mu_k}+\frac{1}{\sqrt{\mu_k}}<2L.
\]
Setting $z_k=\sqrt{\mu_k}>0$, this becomes
\[
    z_k+z_k^{-1}<2L.
\]
Solving the quadratic inequality gives
\[
    L-\sqrt{L^2-1}
    <
    z_k
    <
    L+\sqrt{L^2-1}.
\]
Therefore
\begin{equation}
    C_L
    <
    \mu_k
    <
    C_L^{-1},
    \qquad
    C_L:=
    \left(L-\sqrt{L^2-1}\right)^2 .
    \label{eq:mu_stiffness_interval}
\end{equation}
Here we used
\[
    \left(L+\sqrt{L^2-1}\right)^2=C_L^{-1}.
\]

Now assume that both the target and source are Mat\'ern-type Gaussian
measures with the same length-scale parameter $\kappa$. Denote their
smoothness parameters by $\nu_1$ and $\nu_0$, respectively, and let
$r_0=\mathcal F_0(Q_1)/\mathcal F_0(Q_0)>0$. Their spectral covariance ratio
takes the form
\begin{equation}
    \mu_k
    =
    r_0
    \left(
        \frac{\kappa^2}
        {\kappa^2+4\pi^2\|k\|^2}
    \right)^{\nu_1-\nu_0}.
    \label{eq:mu_k_matern_ratio}
\end{equation}
By the Cameron--Martin admissibility condition, we choose
\[
    \nu_0<\nu_{\rm adm}
    =
    \nu_1-\frac d2 .
\]
In particular, $\nu_0<\nu_1$. Hence
$\nu_1-\nu_0>0$, and $\mu_k$ is decreasing with respect to
$\|k\|$.

For all resolved modes $\|k\|\le N$, define
\[
    M_N
    :=
    \frac{\kappa^2}{\kappa^2+4\pi^2N^2}
    \in(0,1).
\]
Since $\mu_k$ is minimized at the largest resolved frequency, we have
\begin{equation}
    \mu_k
    \ge
    r_0
    M_N^{\nu_1-\nu_0},
    \qquad
    \|k\|\le N .
    \label{eq:mu_lower_bound}
\end{equation}
Therefore, the lower bound $\mu_k>C_L$ in
\eqref{eq:mu_stiffness_interval} is guaranteed by
\begin{equation}
    r_0
    M_N^{\nu_1-\nu_0}
    >
    C_L .
    \label{eq:lower_condition_raw}
\end{equation}

It remains to control the upper bound $\mu_k<C_L^{-1}$. Since
$\mu_k$ is largest at the zero mode, it is sufficient to impose
\begin{equation}
    \mu_0
    =
    r_0
    <
    C_L^{-1}.
    \label{eq:upper_ratio_condition}
\end{equation}
Combining \eqref{eq:lower_condition_raw} and
\eqref{eq:upper_ratio_condition} gives
\[
    C_L
    <
    \mu_k
    <
    C_L^{-1},
    \qquad
    \|k\|\le N .
\]
By \eqref{eq:mu_stiffness_interval}, this implies
\[
    |\widetilde V_k(t)|<L,
    \qquad t\in[0,1],\quad \|k\|\le N .
\]

Finally, the Cameron--Martin admissibility condition from
Theorem~\ref{thm:noise_condition} requires
\[
    \nu_0<\nu_{\rm adm}
    =
    \frac{\beta}{2}-d .
\]
Choosing
\[
    \nu_0<\nu_{\rm adm}-\varepsilon
\]
with $\varepsilon>0$ enforces this condition strictly. Therefore any
$\nu_0$ satisfying
\[
    0<\nu_0<\nu_{\rm adm}-\varepsilon,
    \qquad
    r_0 M_N^{\nu_1-\nu_0}>C_L,
    \qquad
    r_0<C_L^{-1}
\]
is both Cameron--Martin admissible and satisfies the prescribed spectral
stiffness bound. This proves the claim.
\end{proof}

\subsection{Proof of Eq.~\ref{eq:tweedie}}
\label{pf:tweedie}

\begin{proof}
Let $H$ be a separable Hilbert space.
Assume $f,g\in L^2(\Omega;H)$ and, for each $t\in[0,1]$,
\begin{equation}
    f_t
    =
    \alpha_t g+\beta_t f,
    \label{eq:linear_path_tweedie}
\end{equation}
where $\alpha_t,\beta_t$ are deterministic $C^1$ scalar functions.
Assume further that
\begin{equation}
    \dot\alpha_t\neq 0
\end{equation}
and
\begin{equation}
    c_t
    :=
    \beta_t-\dot\beta_t\frac{\alpha_t}{\dot\alpha_t}
    \neq 0 .
    \label{eq:ct_nonzero}
\end{equation}
The conditional velocity is
\begin{equation}
    v_t(f_t\mid f)
    :=
    \frac{\mathrm{d}}{\mathrm{d}t}f_t
    =
    \dot\alpha_t g+\dot\beta_t f ,
    \label{eq:conditional_velocity_tweedie}
\end{equation}
where the derivative is understood in $L^2(\Omega;H)$.

From \eqref{eq:conditional_velocity_tweedie}, we solve for $g$:
\begin{equation}
    g
    =
    \frac{1}{\dot\alpha_t}
    \left(
        v_t(f_t\mid f)-\dot\beta_t f
    \right).
\end{equation}
Substituting this into \eqref{eq:linear_path_tweedie} yields
\begin{align}
    f_t
    &=
    \alpha_t
    \frac{1}{\dot\alpha_t}
    \left(
        v_t(f_t\mid f)-\dot\beta_t f
    \right)
    +
    \beta_t f \notag\\
    &=
    \frac{\alpha_t}{\dot\alpha_t}v_t(f_t\mid f)
    +
    \left(
        \beta_t-\dot\beta_t\frac{\alpha_t}{\dot\alpha_t}
    \right)f \notag\\
    &=
    \frac{\alpha_t}{\dot\alpha_t}v_t(f_t\mid f)
    +
    c_t f .
\end{align}
Therefore,
\begin{equation}
    f
    =
    c_t^{-1}
    \left(
        f_t
        -
        \frac{\alpha_t}{\dot\alpha_t}v_t(f_t\mid f)
    \right).
    \label{eq:f_from_ft_velocity}
\end{equation}

Taking conditional expectation with respect to $\sigma(f_t)$ gives
\begin{align}
    \mathbb{E}[f\mid f_t]
    &=
    c_t^{-1}
    \left(
        f_t
        -
        \frac{\alpha_t}{\dot\alpha_t}
        \mathbb{E}\left[
            v_t(f_t\mid f)
            \mid f_t
        \right]
    \right).
\end{align}
By the definition of the marginal velocity field,
\begin{equation}
    v_t(f_t)
    =
    \mathbb{E}
    \left[
        v_t(f_t\mid f)
        \mid f_t
    \right]
    \qquad
    \text{a.s.}
\end{equation}
Hence
\begin{equation}
    \mathbb{E}[f\mid f_t]
    =
    c_t^{-1}
    \left(
        f_t
        -
        \frac{\alpha_t}{\dot\alpha_t}v_t(f_t)
    \right),
\end{equation}
which proves Eq.~\ref{eq:tweedie}.
\end{proof}

%% file: src/7_Appendix_3.tex
\section{Algorithm of \themodel}\label{sec_alg}
\paragraph{Why spectral analysis matters for scientific data.}
For data generated by PDEs and SPDEs, the frequency domain provides a natural description of both physical structure and numerical difficulty.
The decay of spectral energy reflects the regularity of the underlying random fields, while the distribution of energy across modes determines how much large-scale and small-scale variation must be transported by the generative flow.
This is particularly important in function-space generation: a source noise distribution with excessive high-frequency energy may look harmless after discretization, but it can induce a spectrally stiff marginal velocity field and cause resolution-dependent artifacts.
Conversely, a generated sample can have a small pointwise error while still failing to reproduce the correct multiscale energy distribution.
We therefore use spectral diagnostics as a central evaluation tool.
They allow us to test whether the learned model preserves the target measure across spatial scales and whether the proposed regularity-aware noise design improves stability beyond visual sample quality or grid-based errors alone.
\paragraph{Algorithm~1 (Flow Matching Training).}
Algorithm~1 summarizes the flow-matching training procedure under a prescribed interpolation schedule.
Given a dataset $\mathcal{D}$ of samples from the target distribution, scalar schedules $\alpha_t$ and $\beta_t$, and a parametric velocity field $v_\theta(\cdot,t)$, the objective is to learn a time-dependent vector field that transports samples from a source noise distribution to the target distribution.
Prior to training, we estimate the spectral decay rate $\beta_{\rm spec}$ of the data and select a source noise distribution $\pi$ accordingly.
This step implements the proposed noise-aware principle: the source noise should be no smoother than the target distribution, while avoiding excessive high-frequency energy that would induce spectrally stiff transport dynamics.

During each iteration, a mini-batch $\{x_i\}_{i=1}^m$ is drawn from the dataset, and a matching batch of noise samples $\{\varepsilon_i\}_{i=1}^m$ is drawn independently from the chosen noise prior $\pi$.
Time stamps $\{t_i\}_{i=1}^m$ are independently sampled from $\mathrm{Unif}(0,1)$, so that the model is trained across the full interpolation horizon.
For each triplet $(x_i,\varepsilon_i,t_i)$, the interpolated state and the corresponding target velocity are defined by
\begin{equation}
    x_{t_i}
    =
    \alpha_{t_i}\varepsilon_i
    +
    \beta_{t_i}x_i,
    \qquad
    u_{t_i}
    =
    \dot{\alpha}_{t_i}\varepsilon_i
    +
    \dot{\beta}_{t_i}x_i .
\end{equation}
The network prediction $\hat u_{t_i}=v_\theta(x_{t_i},t_i)$ is trained by minimizing
\begin{equation}
    \mathcal{L}(\theta)
    =
    \frac{1}{m}
    \sum_{i=1}^m
    \left\|
        \hat u_{t_i}-u_{t_i}
    \right\|_{\mathcal H}^2 .
\end{equation}
In the discrete implementation, the Hilbert-space norm is evaluated by its grid-based approximation.
Repeating this procedure until convergence yields a trained velocity field for unconditional generation and guided sampling.

\begin{algorithm}[ht]
  \caption{Flow Matching Training}
  \label{alg:flow_matching}
  \begin{algorithmic}
    \STATE {\bfseries Input:} dataset $\mathcal D$,
    schedules $\alpha_t,\beta_t$, model $v_\theta$, learning rate $\eta$.
    \STATE Spectral decay rate $\beta$ analysis of data.
    \STATE Decide noise distribution $\pi$ by $\beta$.
    \REPEAT
    \STATE Sample mini-batch $\{x_i\}_{i=1}^m \sim \mathcal D$
    \STATE Sample noise $\{\epsilon_i\}_{i=1}^m \sim \pi$
    \STATE Sample times $\{t_i\}_{i=1}^m \sim \mathrm{Unif}(0,1)$
    \FOR{$i=1$ {\bfseries to} $m$}
      \STATE $x_{t_i} \leftarrow \alpha_{t_i}\epsilon_i + \beta_{t_i}x_i$
      \STATE $u_{t_i} \leftarrow \dot\alpha_{t_i}\epsilon_i + \dot\beta_{t_i}x_i$
      \STATE $\hat u_{t_i} \leftarrow v_\theta(x_{t_i}, t_i)$
    \ENDFOR
    \STATE $\mathcal L(\theta) \leftarrow \frac{1}{m}\sum_{i=1}^m \|\hat u_{t_i}-u_{t_i}\|_2^2$
    \STATE $\theta \leftarrow \theta - \eta \nabla_\theta \mathcal L(\theta)$
    \STATE Update model parameters
    \UNTIL{convergence}
  \end{algorithmic}
\end{algorithm}

\paragraph{Algorithm~2 (Guided Flow Matching Sampling with DPS).}
Algorithm~2 describes the guided sampling procedure based on the trained velocity field \(v_\theta(\cdot,t)\), where external observations are incorporated through a data-consistency term. Let \(A\) denote a measurement operator and let \(y\) be the observed measurement corrupted by noise of level \(\sigma_y\). The goal is to generate a sample that remains plausible under the learned generative prior while being consistent with the observation model.

The procedure starts from an initial state \(x\sim \pi\) drawn from the same noise prior used at training time. The state is evolved from \(t=0\) to \(t=1\) using a numerical ODE solver with step size \(\Delta t\) over \(N\) steps. At each step, the base velocity \(v_\theta(x,t)\) is computed and a guidance correction is constructed following the DPS principle. An estimate \(\hat x_1\) of the terminal (denoised) sample is formed from the current state (depending on the chosen schedule/parameterization), and a gradient signal is computed to reduce the mismatch between \(A\hat x_1\) and \(y\). The guidance gradient takes the form
\[
g \;\propto\; \sigma \,\nabla_x \,\Phi\!\bigl(A\hat x_1,\,y\bigr),
\]
where \(\Phi\) denotes a data-fidelity loss (e.g., \(\|A\hat x_1-y\|_2^2\)) and \(\sigma\) is the guidance strength. The update direction is obtained by combining the base velocity and the guidance correction, after which the state is advanced by one solver step. After \(N\) iterations, the resulting state is returned as the guided sample.
\begin{algorithm}[ht]
  \caption{Guided Flow Matching Sampling with DPS Guidance}
  \label{alg:guided_fm_dps}
  \begin{algorithmic}
    \STATE {\bfseries Input:} trained velocity field $v_\theta(\cdot,t)$, noise prior $\pi$,
    measurement operator $A$, observation $y$, noise level $\sigma_y$, step size $\Delta t$, number of steps $N$
    \STATE Sample initial state $x \sim \pi$
    \STATE Set $t \leftarrow 0$
    \FOR{$n=1$ {\bfseries to} $N$}
      \STATE Estimate $\hat{x}_1$ at $t=1$.
      \STATE Compute DPS gradient: $g \leftarrow \sigma\nabla_{x}\Phi\bigl(A\hat{x}_1, y\bigr)$.
      \STATE Compute base velocity: $v \leftarrow v_\theta(x,t)$.
      \STATE Update state $x$ by neural ode.
      \STATE Update time: $t \leftarrow t + \Delta t$
    \ENDFOR
    \STATE {\bfseries Output:} guided sample $x$
  \end{algorithmic}
\end{algorithm}

%% file: src/7_Appendix_4.tex
\section{Experiments Details} \label{apdx_exp_details}
This section provides the experimental details needed to reproduce and interpret the reported results.
All experiments were conducted on a single NVIDIA A40 GPU with 48\,GB of memory.
We document the full data preprocessing pipeline, model architecture, optimizer settings, learning-rate schedule, batch size, and training duration.
Unless otherwise stated, all comparisons are performed under the same experimental configuration across different noise choices.
In particular, we keep the backbone architecture, training budget, optimizer, and evaluation protocol fixed, and vary only the noise formulation.
This controlled setup allows us to isolate the effect of noise design on training stability, sampling efficiency, and generation quality.

We also include ablation studies to examine the sensitivity of the proposed method to key hyperparameters, including the noise smoothness, correlation length, regularization strength, and sampling-step configuration.
These experiments are designed to identify parameter regimes in which the method remains stable and to assess whether the observed improvements are robust to reasonable hyperparameter changes.
The complete experimental results are reported in the appendix.
In addition to the summary results presented in the main text, we provide per-metric comparisons, convergence curves, run-to-run variability, and qualitative visualizations.
Where appropriate, we also include diagnostic plots, such as training loss curves and stability indicators across samples or time steps, to give a more complete picture of model behavior.

Overall, the results show that Mat\'ern noise provides consistent advantages over Gaussian noise under matched experimental settings.
During training, it leads to faster convergence and reduced sensitivity to hyperparameter choices.
During inference, it improves numerical stability and produces more reliable samples, with fewer resolution-dependent artifacts and unstable trajectories.
These findings support the central claim of this work: regularity-aware noise design improves both the optimization behavior and the inference-time robustness of functional generative models for multiscale SPDE data.
\subsection{Baselines}
Our baselines are constructed from standard choices commonly used in flow-matching generative models, but we report them according to their two key components: the source noise distribution and the neural backbone.
This notation makes the comparison more transparent for our purpose, since the central question of this work is how the choice of noise affects training and sampling under otherwise matched settings.
For example, ``Gaussian + FNO'', also known as DiffusionPDE~\citep{huang2024diffusionpde} denotes a standard flow-matching model using Gaussian source noise and an FNO backbone, which is also known as FFM \citep{kerrigan2023functional}, while ``Mat\'ern + FNO'' replaces only the source noise with the proposed regularity-aware Mat\'ern noise.

We use the Fourier Neural Operator (FNO) as the primary backbone because it is a widely used operator-learning architecture for PDE data and can efficiently capture global interactions through spectral representations.
To test whether the observed improvements are tied to this specific architecture, we also evaluate U-Net as an alternative grid-based backbone.
When U-Net is used, the model is instantiated as a conventional flow-matching model on Euclidean grids and is therefore treated as a standard baseline rather than as part of our operator-based formulation.
Across all backbone choices, we keep the optimizer, training budget, data preprocessing, interpolation schedule, sampling protocol, and evaluation metrics fixed unless otherwise stated.
Thus, performance differences can be attributed primarily to the source noise design and, secondarily, to the chosen backbone.

Several related functional generative modeling frameworks are discussed for context.
In particular, Functional Flow Matching (FFM) is closely related to our setting at a conceptual level, but its published experiments do not provide quantitative comparisons or systematic analyses of source-noise selection.
We therefore include FFM as a conceptual reference rather than a direct empirical baseline.

\begin{table}[H]
\centering
\caption{Summary of baseline methods.}
\label{tab:baseline}
\begin{tabular}{lcc}
\toprule
Method & Model & Noise \\
\midrule
Vanilla FM~\citep{lipmanflow} & U-Net & Gaussian \\
FFM~\citep{kerrigan2023functional} & FNO & Gaussian \\
DiffusionPDE~\citep{huang2024diffusionpde} & FNO & Gaussian\\
GAN~\citep{goodfellow2014generative} & U-Net & -- \\
VAE~\citep{kingma2013auto} & U-Net & -- \\
GANO~\citep{rahman2022generative} & FNO & -- \\
VANO~\citep{seidman2023variational} & FNO & -- \\
\bottomrule
\end{tabular}
\end{table}

\subsection{Introduction of Datasets}\label{apdx:intro_dataset}
We next introduce the numerical benchmarks used to evaluate the proposed method.
The experiments are designed to test whether the choice of source noise affects training stability, sampling efficiency, and resolution consistency when learning invariant measures of multiscale stochastic systems.
We consider three representative stochastic-field settings with different spatial structures and spectral properties.
\subsubsection{Stochastic Allen Cahn Equation (SAC)}
The 1D stochastic Allen--Cahn equation is defined as:
\begin{equation}
    \partial_t u = \partial_{xx} u - V'(u) + \sqrt{2}\eta_t,
    \label{eq:ac_spde}
\end{equation}
posed on the unit interval $[0,1]$ with Neumann boundary conditions and driven by space--time white noise $\eta$. Here, $V(u) = (1 - u^2)^2$ is a double-well potential. Under standard assumptions, the SPDE~\eqref{eq:ac_spde} admits an invariant distribution with formal density
\begin{equation}
    \exp\!\left(
        - \int_0^1 
            \Bigl( \tfrac{1}{2} (\partial_x u(x))^2 + V(u(x)) \Bigr)
        \, dx
    \right).
    \label{eq:ac_invariant}
\end{equation}
For the Allen--Cahn equation, the data distribution is characterized by a path-integral formulation that gives rise to a bimodal invariant measure. Typical realizations of this measure exhibit nearly constant profiles around $u = \pm 1$. Owing to the pronounced nonlinearity inherent in such dynamics, this dataset is considered in our study. The spatial domain is discretized using finite-difference schemes on $N$ equispaced grid points, resulting in an $N$-dimensional approximation of the invariant distribution. Reference samples $x_1$ from this distribution are generated using ensemble Markov chain Monte Carlo (MCMC) methods \citep{chen2025new}.
\subsubsection{Stochastic Navier Stokes Equation (SNS)}
Here we consider the dynamics of the 2D stochastic Navier-Stokes(SNS) equations with random forcing in the domain $[0, 2\pi]^2$. The vorticity formulation of SNS equations is:
\begin{equation}\label{sns}
    d \omega + v\nabla\omega dt = \nu \Delta\omega - \alpha \omega + \sigma dW_t
\end{equation}
where $v = \nabla^\perp \psi = (-\partial_y \psi,\, \partial_x \psi)$ is the velocity expressed in terms of the stream function $\psi$, which solves $-\Delta \psi = \omega$ and $dW_t$ denotes white-in-time colored random forcing. The quantities $\nu$, $\alpha$, and $\sigma$ are positive parameters. 
A pseudo-spectral scheme is used to solve Equ.~\eqref{sns}. The generation of the stochastic Navier--Stokes (SNS) data is based on the methodology proposed by \citet{chen2025new}. 
All training and evaluation data are generated by numerically solving the two-dimensional stochastic Navier--Stokes equations under random external forcing. The viscosity is set to $\nu = 10^{-3}$, with damping and forcing amplitudes given by $\alpha = 0.1$ and $\varepsilon = 1$, respectively.

The stochastic forcing is defined as a finite-dimensional Gaussian process of the form
\begin{align}
\eta(t,x,y) = {} &
\sum_{i=1}^{4} W_i(t)\,\phi_i(x,y)
+ \sum_{i=5}^{8} W_i(t)\,\phi_i(x,y),
\end{align}
where $\{W_i(t)\}_{i=1}^8$ are mutually independent standard Wiener processes, and the spatial basis functions are given by
\[
\{\sin(6x), \cos(7x), \sin(5(x+y)), \cos(8(x+y)),
\cos(6x), \sin(7x), \cos(5(x+y)), \sin(8(x+y))\}.
\]
This construction induces a Gaussian random field that is stationary with respect to spatial translations and injects energy at both low and intermediate Fourier modes.

For this class of stochastic forcing, the resulting stochastic Navier--Stokes dynamics admit a unique invariant measure, as established by \citet{hairer2006ergodicity}. The linear damping term $-\alpha \omega$ is introduced to suppress excessive energy accumulation at large spatial scales, thereby promoting numerical stability and enhancing ergodic mixing of the flow.

The stochastic partial differential equation is solved using a pseudo-spectral method in space combined with an Euler--Maruyama time integration scheme. All simulations are performed on a uniform $256 \times 256$ grid, which provides sufficient resolution for accurately capturing the relevant flow statistics under the considered parameter regime. 
The system is discretized on a \(256 \times 256\) spatial grid with a time step of \(\Delta t = 10^{-4}\), 
and vorticity snapshots are recorded every \(0.5\) time units. 
In total, 1{,}000 independent trajectories are simulated over the time interval \(t \in [0,100]\). 
To ensure convergence to stationarity, data from the initial transient period \(t \in [0,50]\) are discarded. 
The remaining trajectories yield approximately \(10^{5}\) snapshots, 
which are regarded as approximate samples from the invariant measure of Equ.~\eqref{sns}. 
The main models use snapshots downsampled to \(64\times64\); the \(256\times256\) solver output is retained to construct the resolution-transfer evaluations.

\subsection{Experiment 1: Generation at Fixed Resolution}\label{apdx:gafr}
\textbf{Mat\'ern field.}
In the Mat\'ern field case, the target measure is a two-dimensional Mat\'ern Gaussian random field, $\mu_1 = \mathcal{N}\!\left(0,\,(400-\Delta)^{-3}\right)$ with $(\nu_1,\ell_1)=(2,0.1)$. The source uses $(\nu_0,\ell_0)=(0.5,0.05)$, which satisfies the strict admissibility threshold $\nu_0<1$ and has the same $\kappa=20$ as the target. This equal-$\kappa$ choice is the idealized case in Thm.~\ref{thm:3_6}; the length-scale ablation in Appendix~\ref{apdx:source_selection_ablation} evaluates departures from it.
Each random field realization is discretized on an \(N\times N\) grid with \(N=64\). The resulting neural ODE is integrated using a fourth-order Runge--Kutta (RK4) scheme. To assess robustness, the effects of time-step size and noise magnitude on the generated samples are evaluated systematically, and the corresponding quantitative results are summarized in Table~\ref{tab_all}.

\textbf{Architectures and training configuration.}
In the main experiments, the Fourier Neural Operator (FNO) employs four spectral convolution blocks. We follow the \texttt{neuralop} convention in which \(\texttt{n\_modes}=[m_1,m_2]\) is the total number of retained discrete Fourier coefficients per axis, including positive and negative frequencies. Thus \([64,64]\) covers the full \(64\times64\) Fourier grid rather than 64 modes of each sign. The hidden channel width is 64, the time embedding has dimension 64, and the projection width is 256. For the U-Net backbone, four resolution levels are used with \(\texttt{dim\_mults}=[1,2,2,2]\), an attention head dimension of 32, and four attention heads. Training uses batch size 64, Adam, and learning rate \(3\times10^{-4}\). All random-field data are variance-normalized. Qualitative examples are provided in Fig.~\ref{fig:qualitative_all_cases}.

\textbf{SAC.}
In the SAC example, the FNO backbone consists of four spectral convolution layers, with the number of retained modes set to 64 and the projection channel width set to 128. Figure~\ref{fig:qualitative_all_cases}(b, e) presents representative samples from the ground-truth dataset and the corresponding generated samples.

\textbf{SNS.}
In the SNS example, the FNO has four spectral convolution layers, \(\texttt{n\_modes}=[64,64]\) at the \(64\times64\) training resolution, embedding and hidden widths of 128, and projection width 256. Training uses batch size 100, Adam, and learning rate \(3\times10^{-4}\). Figure~\ref{fig:qualitative_all_cases}(c,f) shows a target and a five-step NAGO sample. Timings are reported only under the controlled protocol in Appendix~\ref{apdx:speed_protocol}.

We additionally evaluate a grid-based U-Net under the same data and evaluation protocol. Mat\'ern sources improve several U-Net entries in Table~\ref{tab_all}, showing that the source effect is not exclusive to FNO. Because only two backbones are tested and their optimal hyperparameters may differ, we do not claim architecture-independent superiority.
\begin{figure}[!ht]
  \centering

  \begin{subfigure}[t]{0.34\textwidth}
    \centering
    \includegraphics[width=\textwidth]{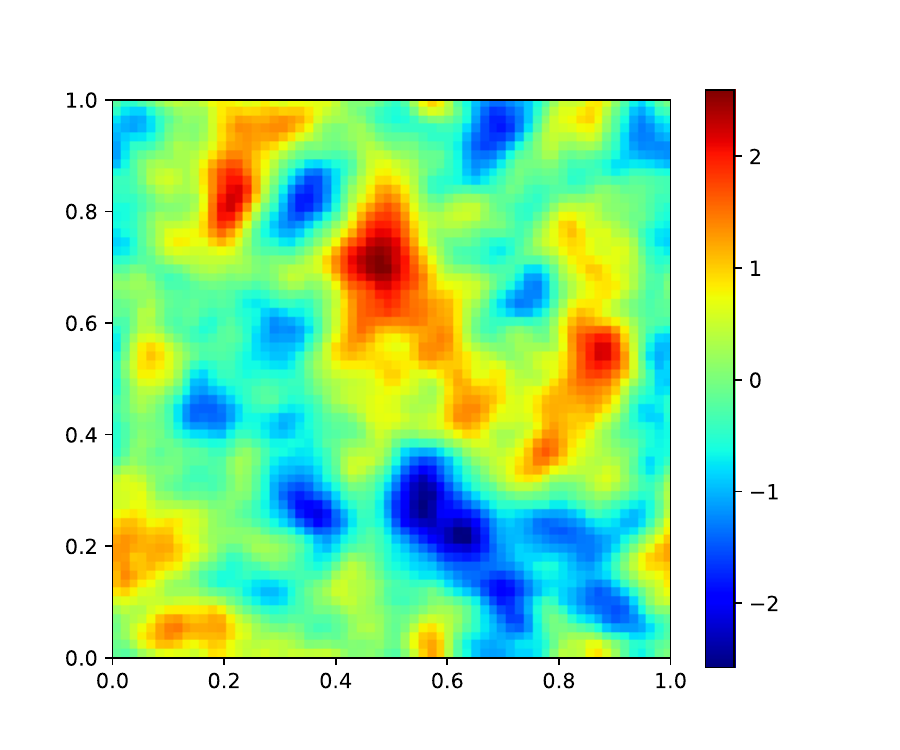}
    \caption{Mat\'ern: Ground truth}
  \end{subfigure}\hfill
  \begin{subfigure}[t]{0.30\textwidth}
    \centering
    \includegraphics[width=\textwidth]{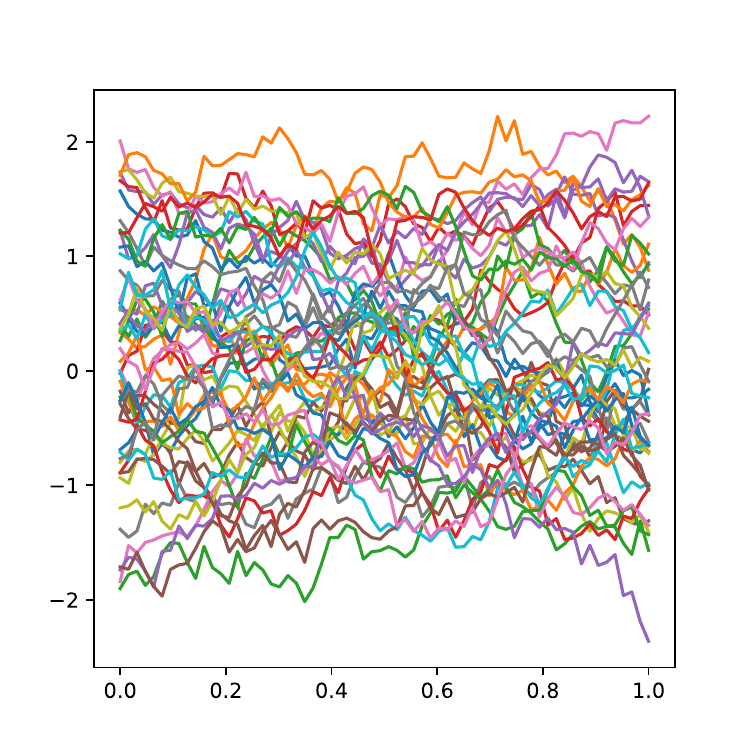}
    \caption{SAC: Ground truth}
  \end{subfigure}\hfill
  \begin{subfigure}[t]{0.34\textwidth}
    \centering
    \includegraphics[width=\textwidth]{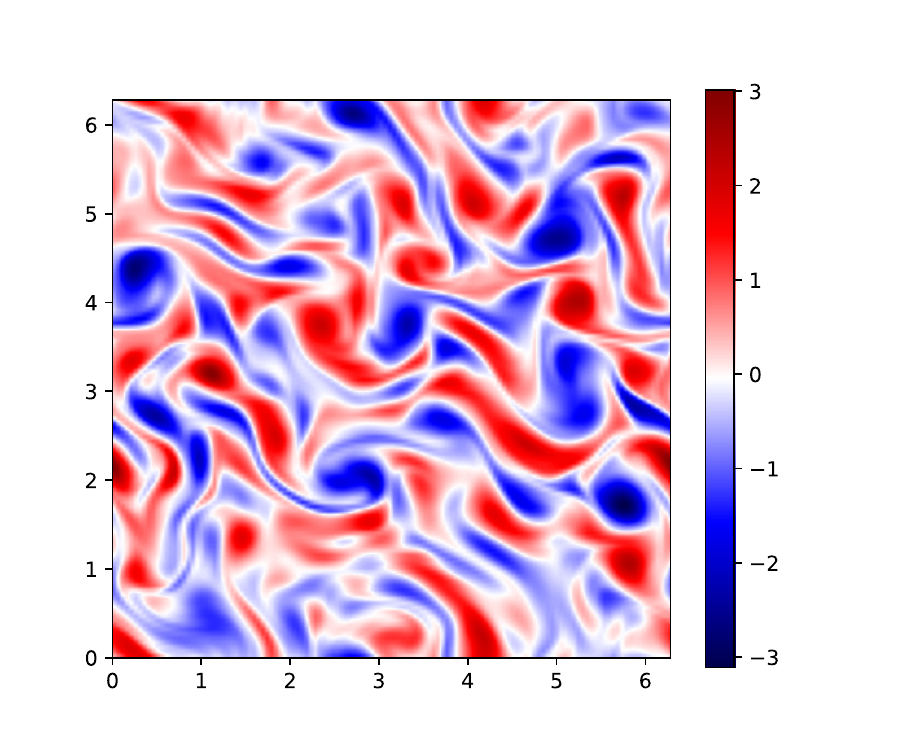}
    \caption{SNS: Ground truth}
  \end{subfigure}

  \vspace{0.6em}

  \begin{subfigure}[t]{0.34\textwidth}
    \centering
    \includegraphics[width=\textwidth]{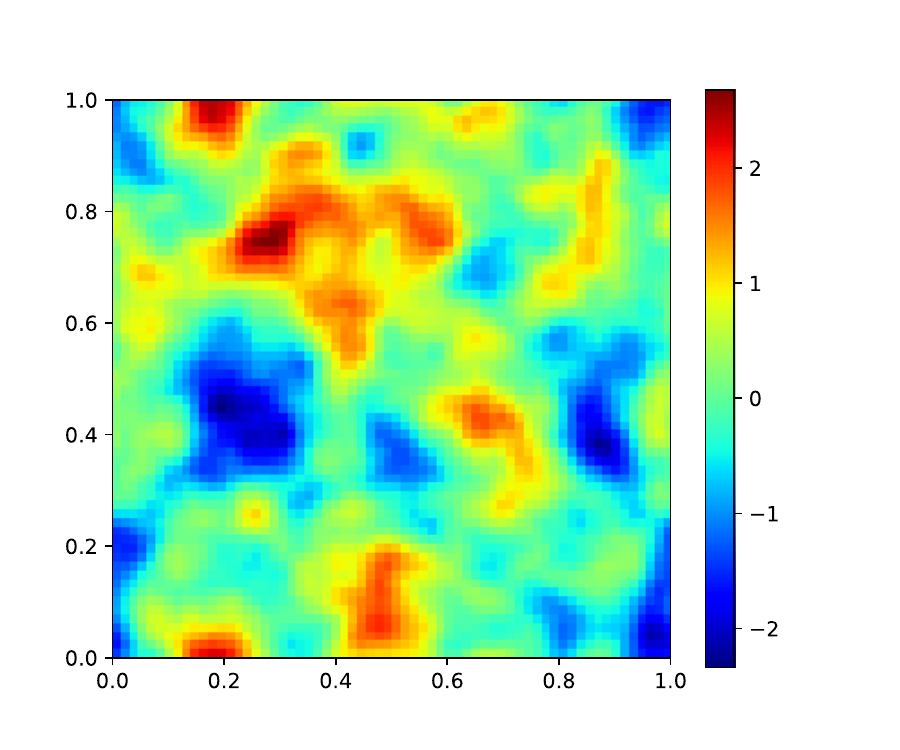}
    \caption{Mat\'ern: Generated by \themodel}
  \end{subfigure}\hfill
  \begin{subfigure}[t]{0.30\textwidth}
    \centering
    \includegraphics[width=\textwidth]{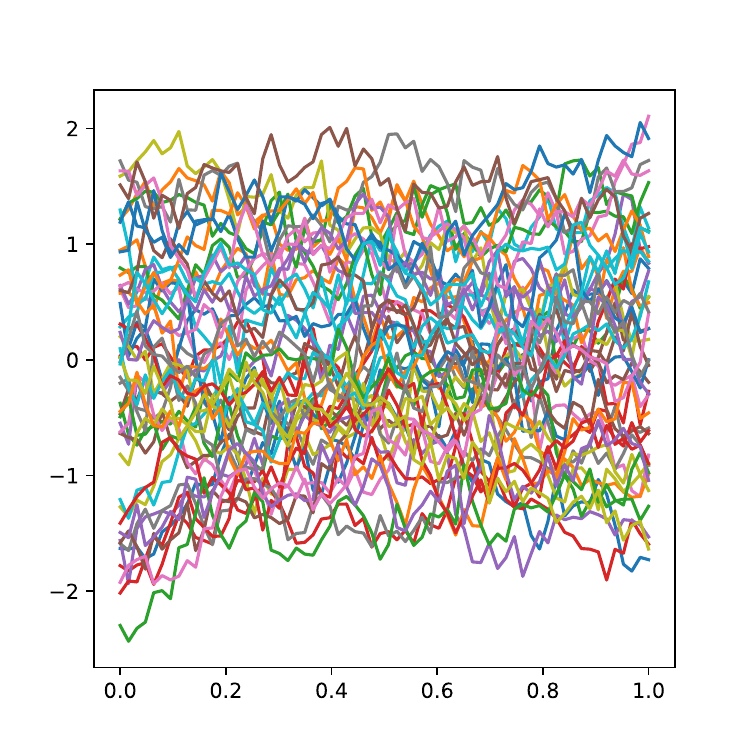}
    \caption{SAC: Generated by \themodel}
  \end{subfigure}\hfill
  \begin{subfigure}[t]{0.34\textwidth}
    \centering
    \includegraphics[width=\textwidth]{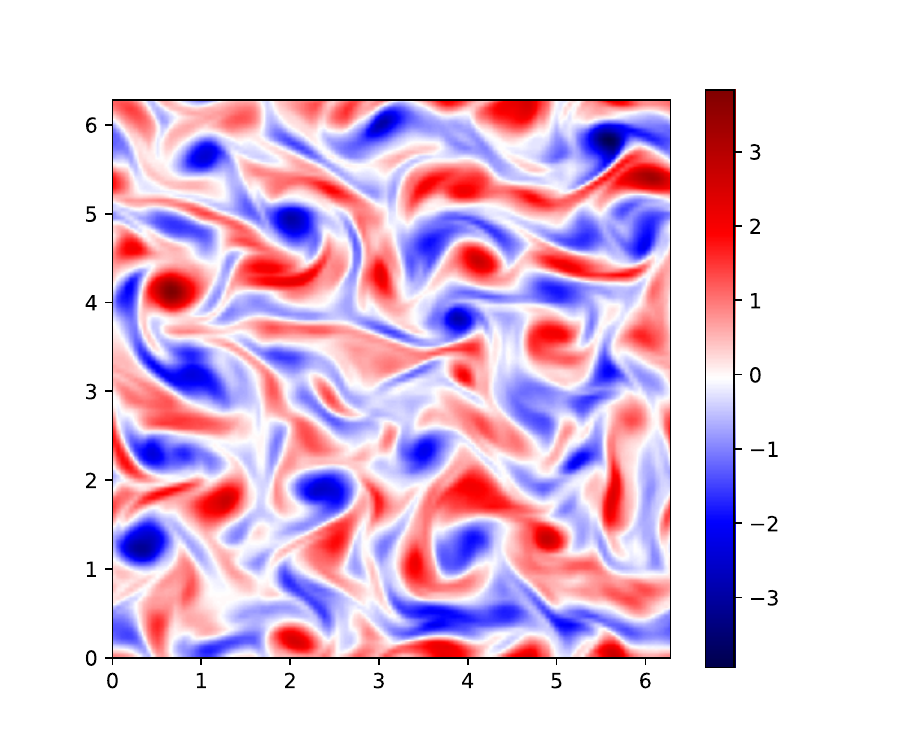}
    \caption{SNS: Generated by \themodel}
  \end{subfigure}

  \caption{Qualitative comparison across three cases: Mat\'ern, SAC, and SNS. The first row shows ground-truth samples, and the second row shows samples generated by \themodel.}
  \label{fig:qualitative_all_cases}
\vspace{-10pt}
  
\end{figure}
\begin{remark}[Non-monotonic error with respect to ODE steps]
\label{rem:non_monotone_steps}
In several experiments, we observe that increasing the number of ODE
integration steps does not always monotonically improve the generation error.
For example, in Table~\ref{tab_all}, the Mat\'ern FNO error increases when
the number of steps is increased from a small value to a larger one, and a
similar non-monotonic behavior appears in guided generation.

This phenomenon does not contradict the standard intuition that smaller ODE
time steps reduce numerical discretization error. In a learned flow model,
the total sampling error is affected not only by the solver discretization
error, but also by the approximation error of the learned marginal velocity
field. Increasing the number of integration steps reduces the former, but it
also increases the number of evaluations of the learned velocity field.
Therefore, small biased errors in the learned field may accumulate along the
sampling trajectory. This accumulation can be more significant in spectrally
stiff regimes, where high-frequency components are more sensitive to
velocity-field errors.

The effect can be further amplified in guided generation. The guidance term
modifies the learned velocity field at sampling time and may introduce an
additional bias that is not present during training. Although this guidance
can improve conditional alignment at moderate step counts, too many
integration steps may repeatedly amplify off-manifold or high-frequency
components, leading to degraded metrics. Therefore, the optimal number of
sampling steps reflects an empirical accuracy--stability trade-off rather
than purely the numerical accuracy of the ODE solver.
\end{remark}







\subsection{Experiment 2: Generation at Varying Resolution} \label{apdx:gavr}

Fig.~\ref{fig_all} (C3) illustrates the spectral density curves (i.e., spectral energy distributions) estimated from samples generated by different frameworks at multiple spatial resolutions. This experiment is designed to evaluate cross-resolution consistency, namely, whether a model trained under a given discretization can preserve the characteristic frequency content of the target distribution when generation is performed on coarser or finer grids. For the Mat\'ern field example, \themodel employs a five-step solver during sampling, whereas all other frameworks are configured with solvers using at least ten steps, reflecting their default inference-time settings and ensuring stable generation. In all cases, 1{,}000 independent samples are generated per resolution to estimate the corresponding spectral energy distributions with sufficient statistical reliability. For each generated sample, a discrete Fourier transform is computed on the associated grid, and the resulting power spectrum is aggregated (e.g., via radial averaging over frequency shells) to obtain a one-dimensional spectral density curve. The curves reported in Fig.~\ref{fig_all} (C4) are then obtained by averaging these spectral estimates across samples, which reduces Monte Carlo noise and enables a clearer comparison of frequency-domain discrepancies among methods.

For the SNS benchmark, \themodel again adopts a five-step solver, while FFM uses a ten-step solver, following the same evaluation principle while accounting for the different benchmark characteristics and computational costs. To compute the spectral energy distributions in this setting, the number of generated samples is adjusted according to resolution in order to balance estimation accuracy and runtime: 200 samples are generated at resolution 128, 500 samples at resolution 64, and 1{,}000 samples at resolution 32. As in the Mat\'ern experiment, spectral densities are computed from the Fourier-domain energy and averaged across the corresponding sample sets. This protocol allows the resolution dependence of the generated frequency content to be assessed in a controlled manner, while maintaining comparable statistical confidence across resolutions.

Fig.~\ref{fig:matern_varying_res} and Fig.~\ref{fig:sns_varying_res} present representative numerical samples of the Mat\'ern field and the SNS benchmark, respectively, generated at different resolutions to complement the spectral-domain analysis with qualitative evidence in the physical domain. Both figures are produced using models trained in Appendix~\ref{apdx:gafr}. For generations, \themodel employs five sampling steps, whereas FFM uses ten steps, consistent with the solver configurations used in Fig.~\ref{fig_all} (C3). It is observed that under the \themodel framework, both the numerical range (e.g., amplitude scale and distribution of extrema) and the spectral energy remain stable across resolutions, indicating that the generated samples preserve the target distributional characteristics under grid changes. In contrast, under the FFM framework, noticeable drift is observed as the resolution varies, manifested as changes in amplitude scaling and shifts in the frequency content, which collectively suggest a mismatch with the target distribution and reduced cross-resolution robustness.

\begin{figure}[!ht]
  \centering
  \begin{subfigure}[t]{0.32\linewidth}
    \centering
    \includegraphics[width=\linewidth]{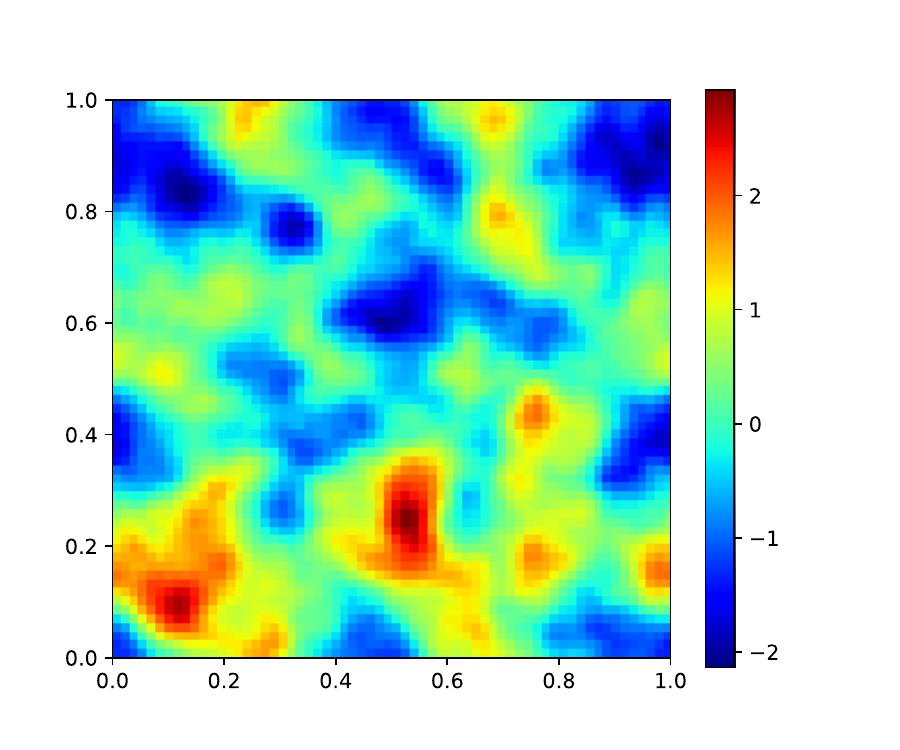}
    \subcaption{}
  \end{subfigure}%
  \begin{subfigure}[t]{0.32\linewidth}
    \centering
    \includegraphics[width=\linewidth]{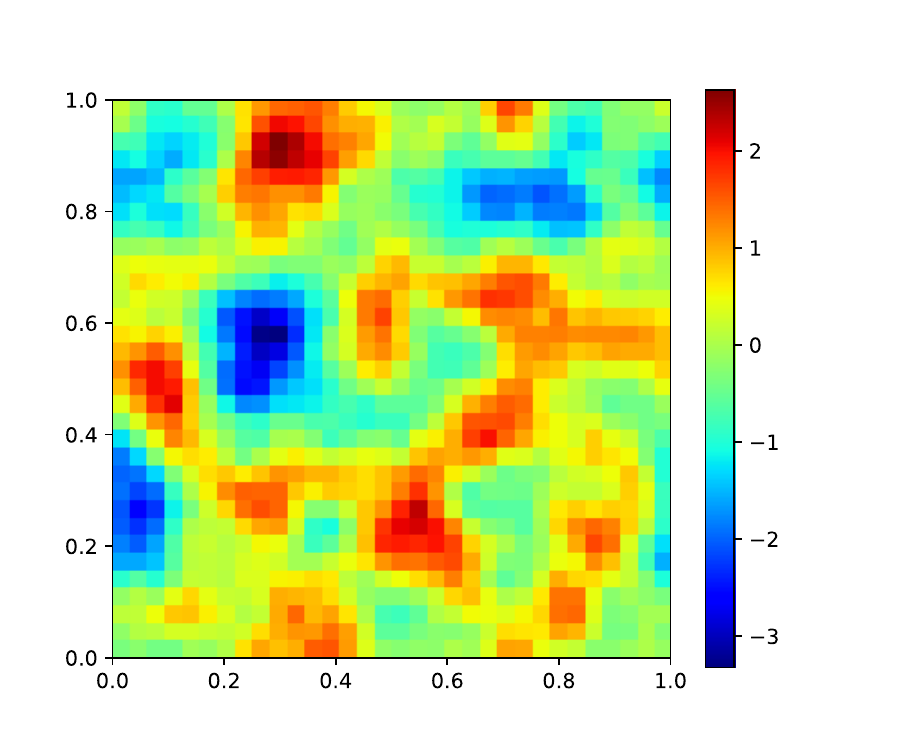}
    \subcaption{}
  \end{subfigure}%
  \begin{subfigure}[t]{0.32\linewidth}
    \centering
    \includegraphics[width=\linewidth]{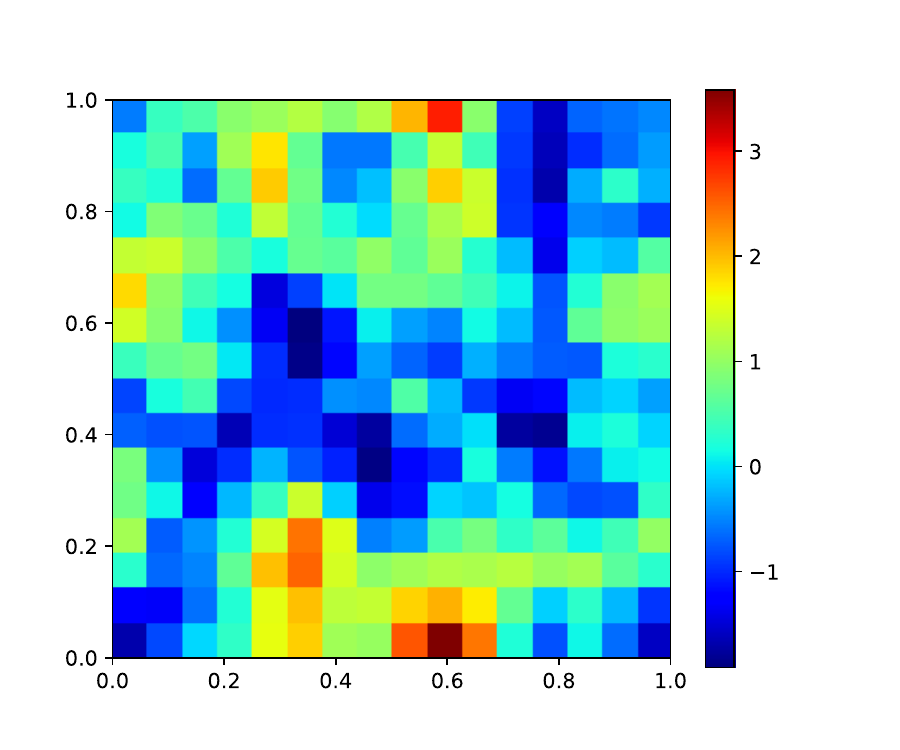}
    \subcaption{}
  \end{subfigure}
  
  \begin{subfigure}[t]{0.32\linewidth}
    \centering
    \includegraphics[width=\linewidth]{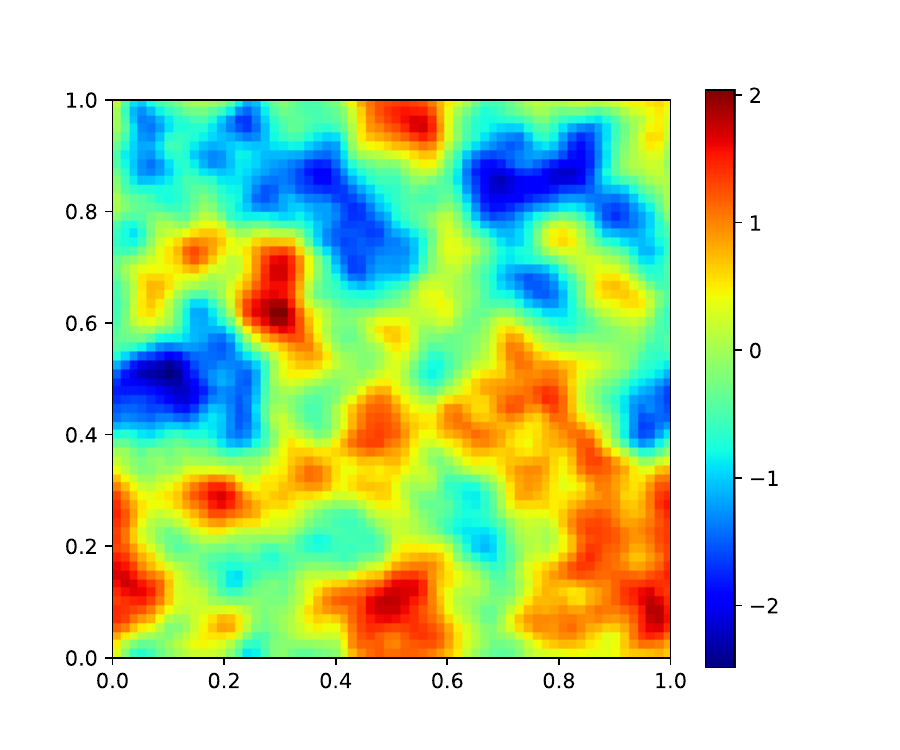}
    \subcaption{}
  \end{subfigure}%
  \begin{subfigure}[t]{0.32\linewidth}
    \centering
    \includegraphics[width=\linewidth]{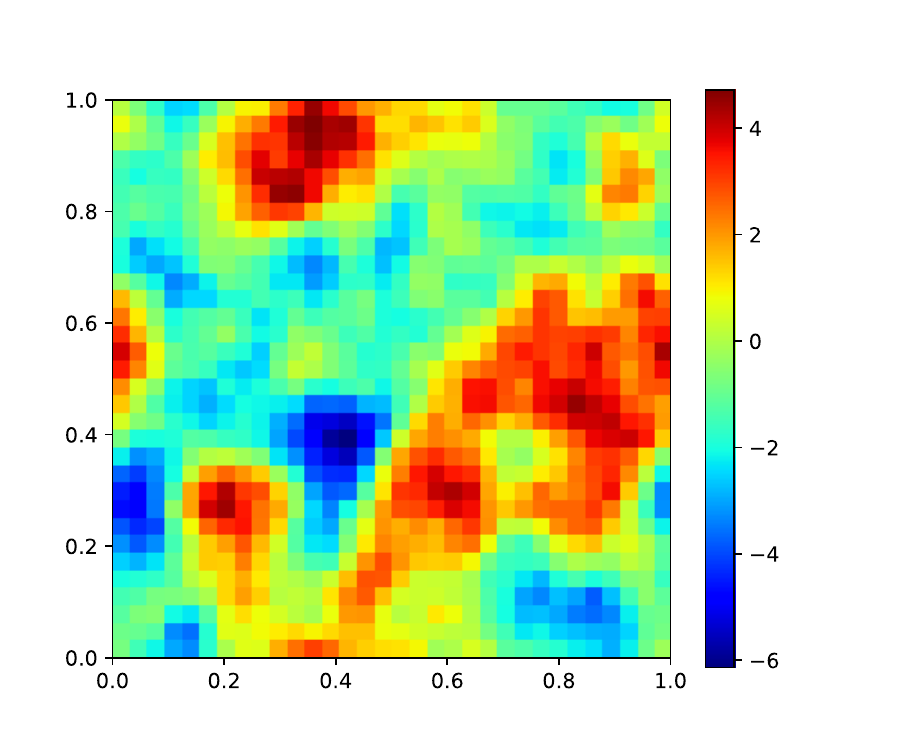}
    \subcaption{}
  \end{subfigure}%
  \begin{subfigure}[t]{0.32\linewidth}
    \centering
    \includegraphics[width=\linewidth]{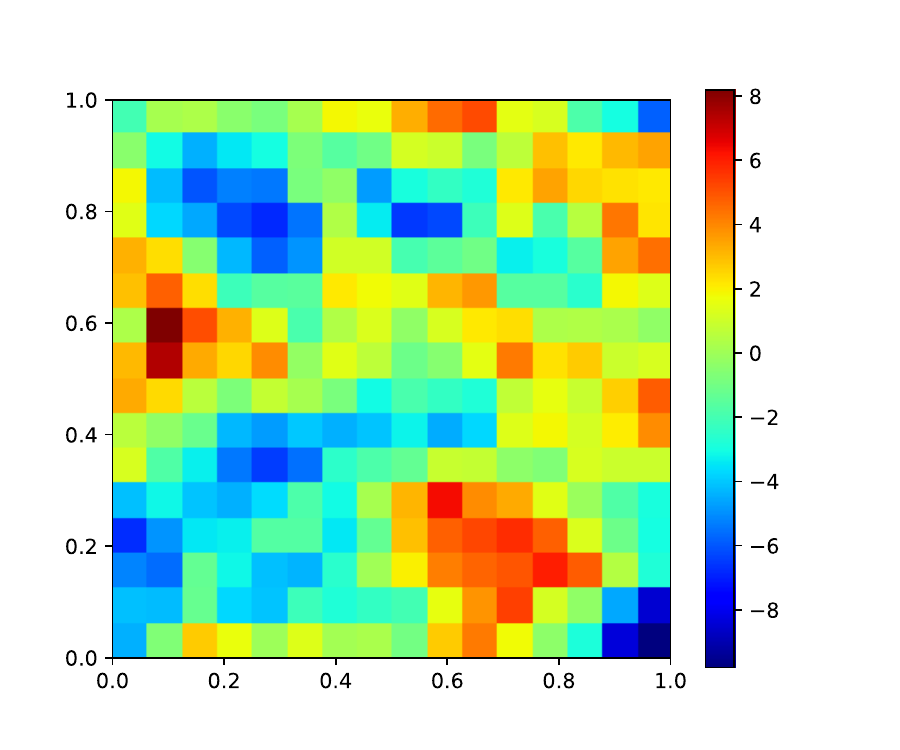}
    \subcaption{}
  \end{subfigure}
  \caption{Matern Field generated at different resolutions. Samples generated at resolution 64 (1st column), 32 (2nd column), 16 (3rd column) with (First row) \themodel and (Second row) FFM}
  \label{fig:matern_varying_res}
\end{figure}

\begin{figure}[!ht]
  \centering
  \begin{subfigure}[t]{0.48\linewidth}
    \centering
    \includegraphics[width=\linewidth]{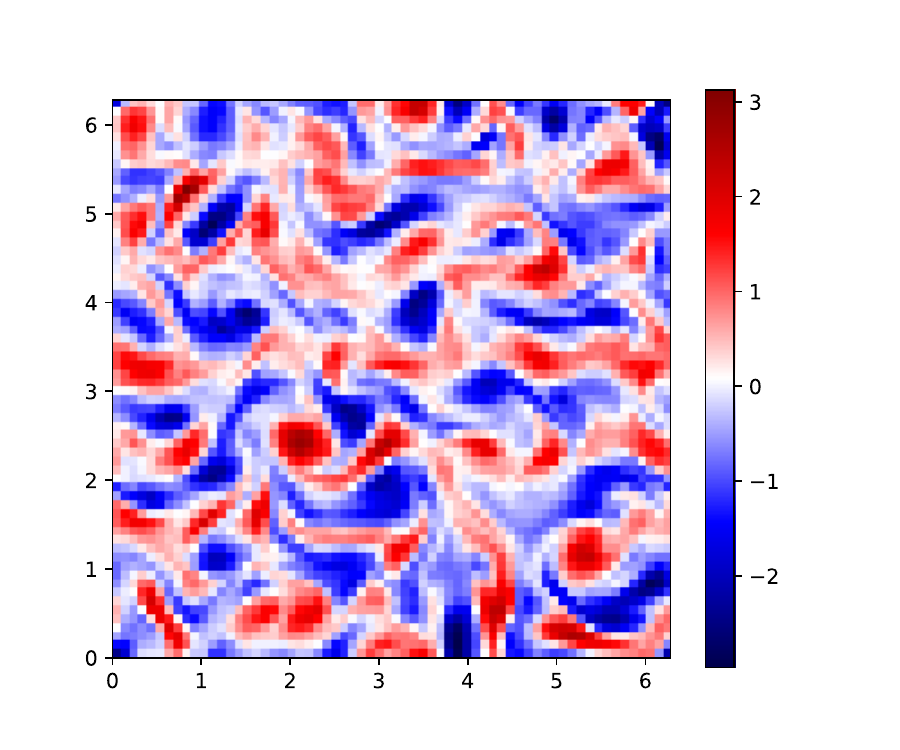}
    \subcaption{}
  \end{subfigure}%
  \begin{subfigure}[t]{0.48\linewidth}
    \centering
    \includegraphics[width=\linewidth]{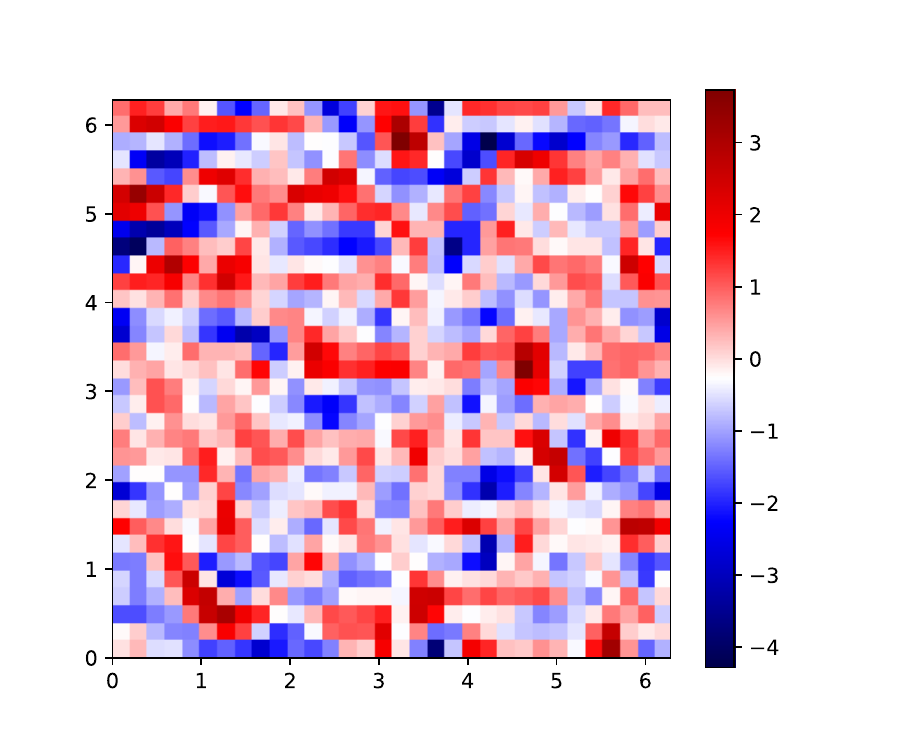}
    \subcaption{}
  \end{subfigure}
  \begin{subfigure}[t]{0.48\linewidth}
    \centering
    \includegraphics[width=\linewidth]{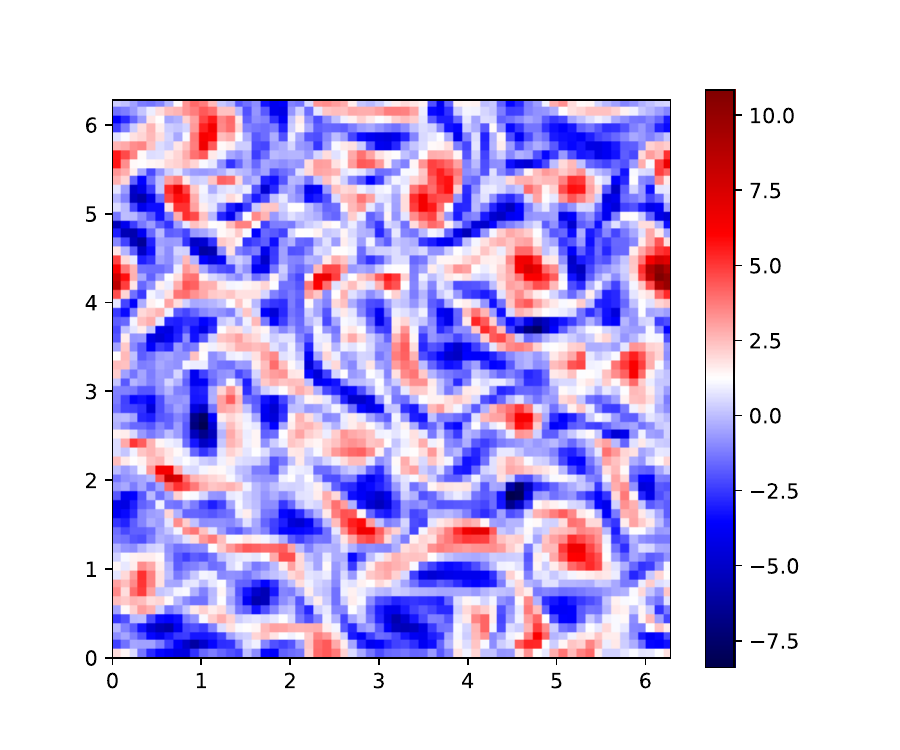}
    \subcaption{}
  \end{subfigure}%
  \begin{subfigure}[t]{0.48\linewidth}
    \centering
    \includegraphics[width=\linewidth]{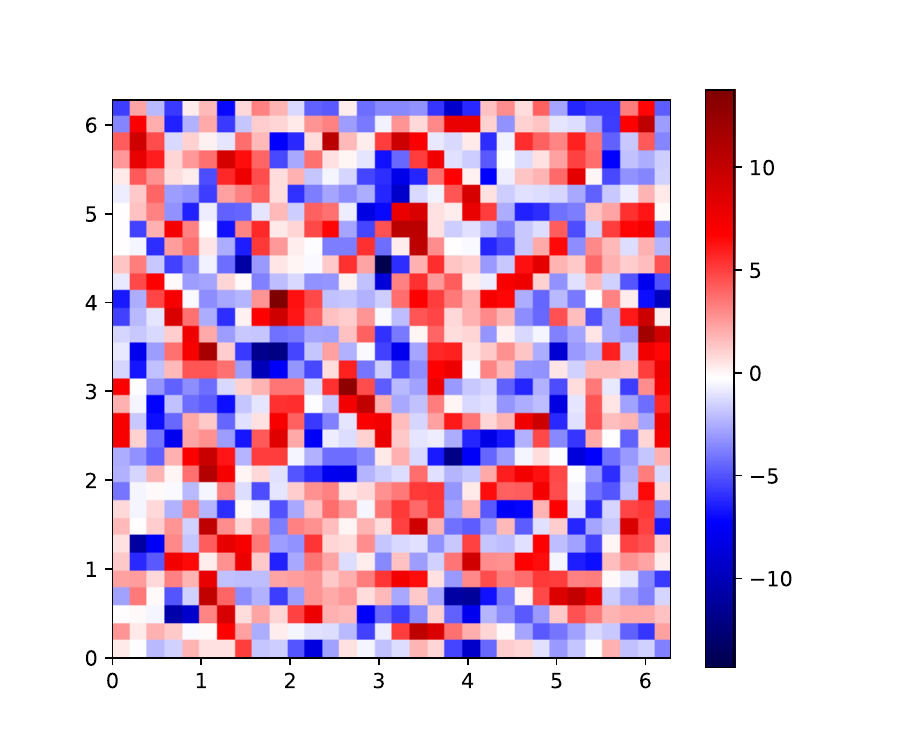}
    \subcaption{}
  \end{subfigure}
  
  \caption{Matern Field generated at different resolutions.}
  \label{fig:sns_varying_res}
\end{figure}

\subsection{Experiment 3: Generation with Guidance}

A detailed comparison of generation quality under guided settings is presented for generative frameworks equipped with different noise choices. In this set of experiments, guidance is provided through low-resolution observations, and the resulting conditional generation performance is evaluated both visually and quantitatively to assess how the noise design affects reconstruction fidelity and uncertainty calibration.

Figure~\ref{fig:matern_big_figure} reports the guided-generation results for the Mat\'ern Field, while Figure~\ref{fig:sns_big_figure} reports the corresponding results for the SNS dataset. In both figures, panel (a) displays the full-resolution ground-truth field, and panel (b) shows the low-resolution observation that is supplied to the generative model as the guidance signal. Panels (c--e) summarize the statistics of samples produced by \themodel, including the pixel-wise sample mean, the pixel-wise sample standard deviation, and the reconstruction mean squared error (MSE) with respect to the ground truth. Panels (f--h) provide the same set of statistics for samples generated by FFM, enabling a direct, method-aligned comparison under identical guidance information.

All reported statistical quantities are estimated from 200 independently generated samples per method. The sample mean is used to characterize the central tendency of the conditional distribution, the sample standard deviation is used to quantify uncertainty and spatial variability under guidance, and the reconstruction MSE map is used to localize regions where systematic bias or loss of fine-scale structure is most pronounced. This evaluation protocol is intended to jointly reflect reconstruction accuracy and uncertainty consistency, thereby providing a more comprehensive assessment of guided generation quality than single-sample visualization alone.

\begin{figure}[!ht]
  \centering
  \begin{subfigure}[t]{0.49\linewidth}
    \centering
    \includegraphics[width=\linewidth]{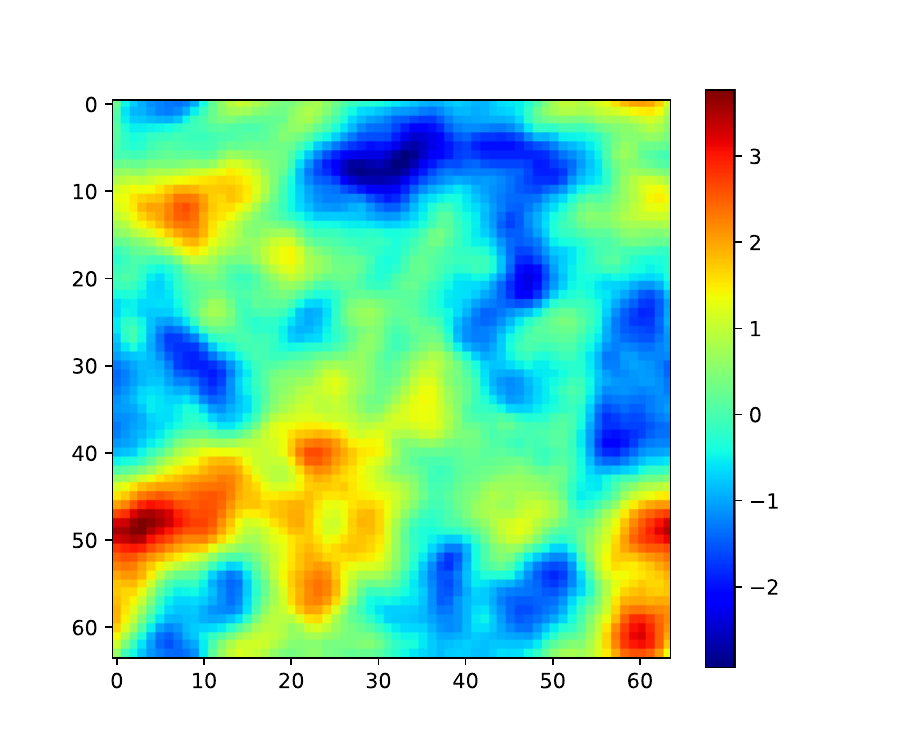}
    \subcaption{}
  \end{subfigure}%
  \begin{subfigure}[t]{0.49\linewidth}
    \centering
    \includegraphics[width=\linewidth]{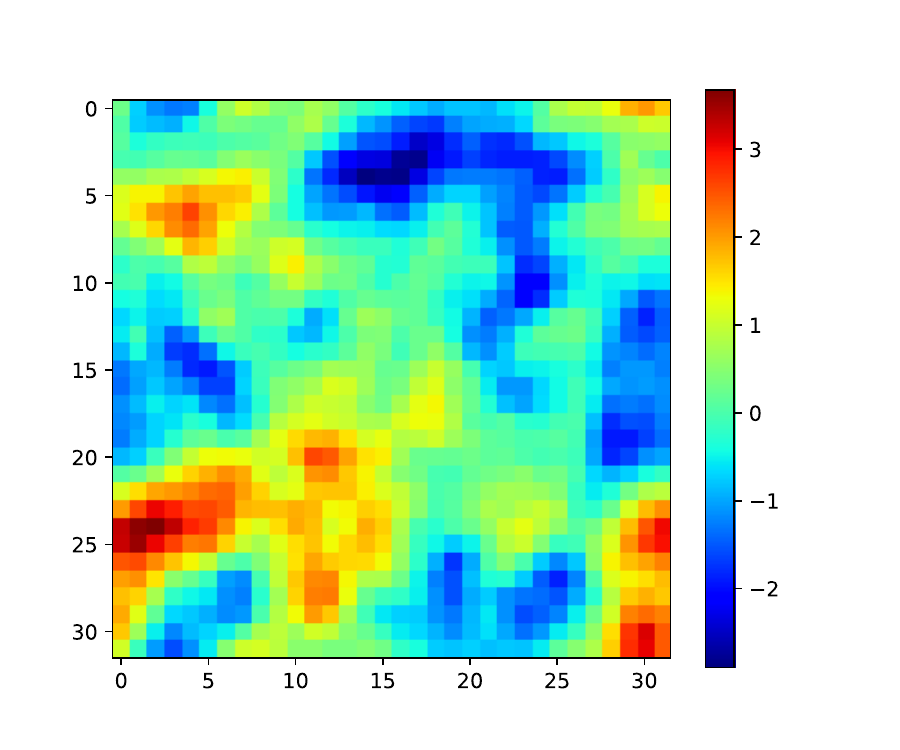}
    \subcaption{}
  \end{subfigure}
  
  \begin{subfigure}[t]{0.32\linewidth}
    \centering
    \includegraphics[width=\linewidth]{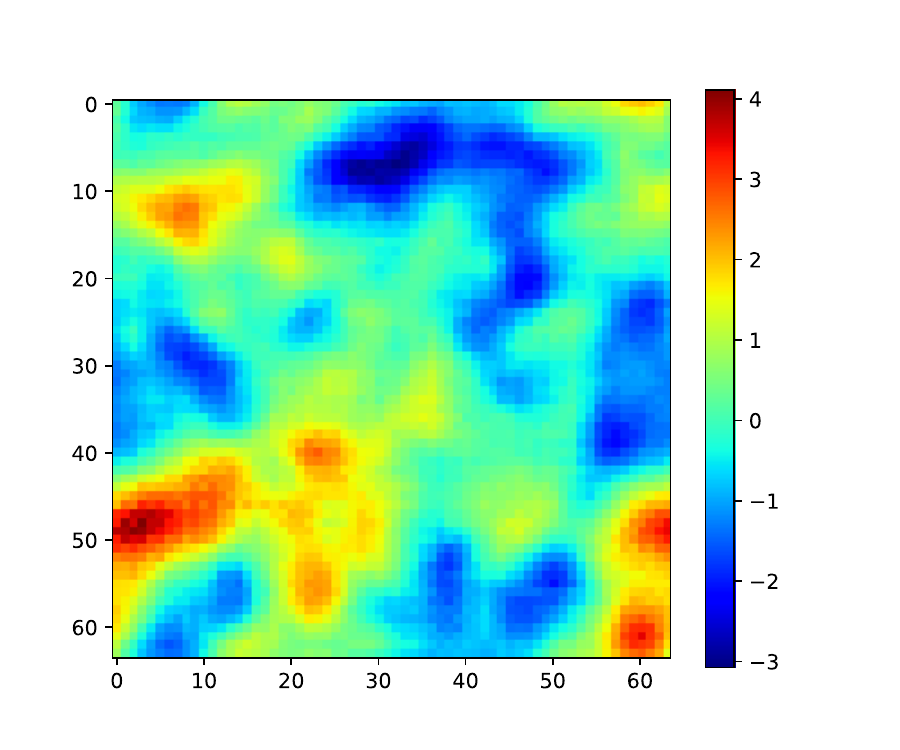}
    \subcaption{}
  \end{subfigure}%
   \begin{subfigure}[t]{0.32\linewidth}
    \centering
    \includegraphics[width=\linewidth]{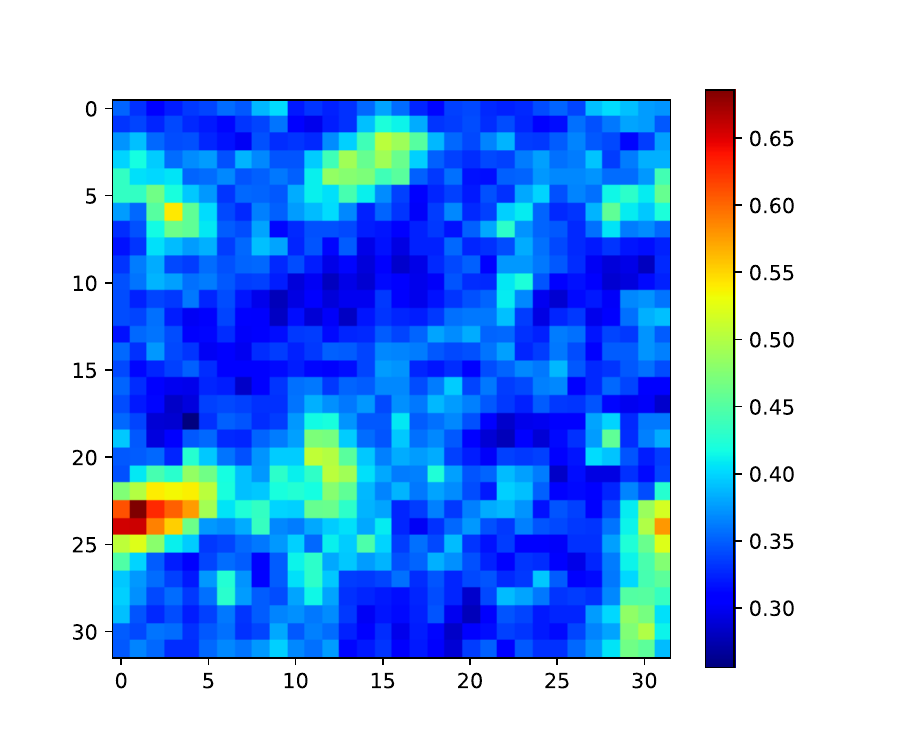}
    \subcaption{}
  \end{subfigure}%
    \begin{subfigure}[t]{0.32\linewidth}
    \centering
    \includegraphics[width=\linewidth]{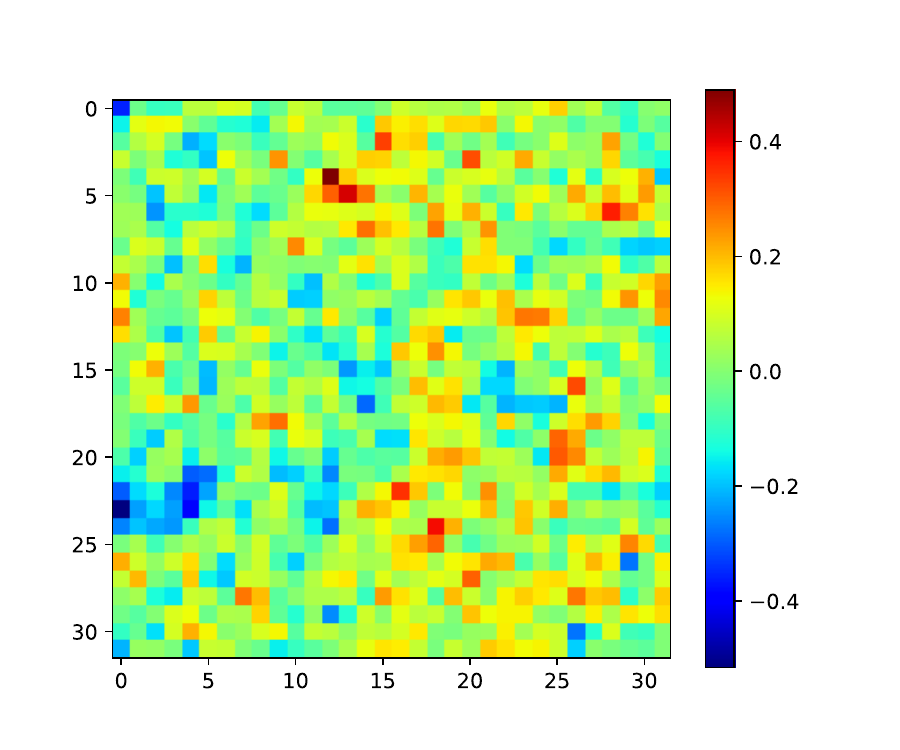}
    \subcaption{}
  \end{subfigure}
  \begin{subfigure}[t]{0.32\linewidth}
    \centering
    \includegraphics[width=\linewidth]{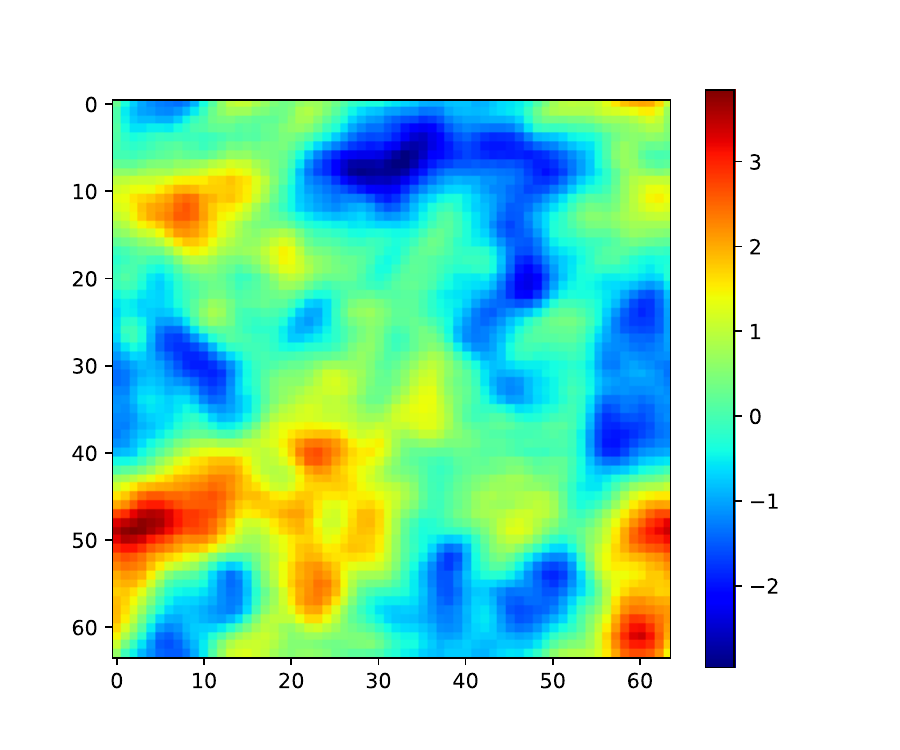}
    \subcaption{}
  \end{subfigure}%
    \begin{subfigure}[t]{0.32\linewidth}
    \centering
    \includegraphics[width=\linewidth]{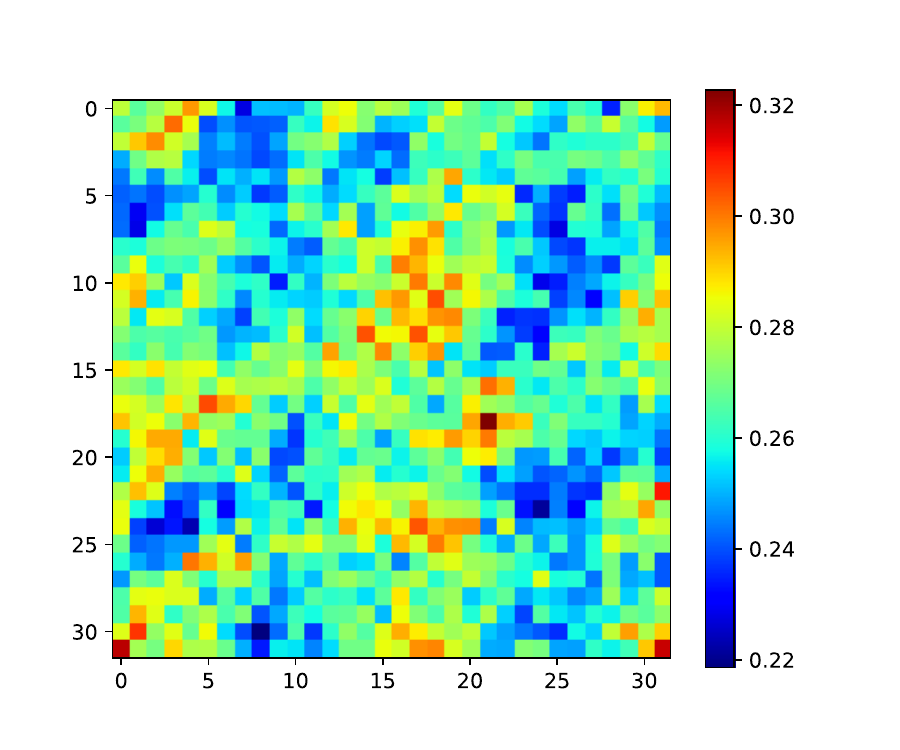}
    \subcaption{}
  \end{subfigure}%
    \begin{subfigure}[t]{0.32\linewidth}
    \centering
    \includegraphics[width=\linewidth]{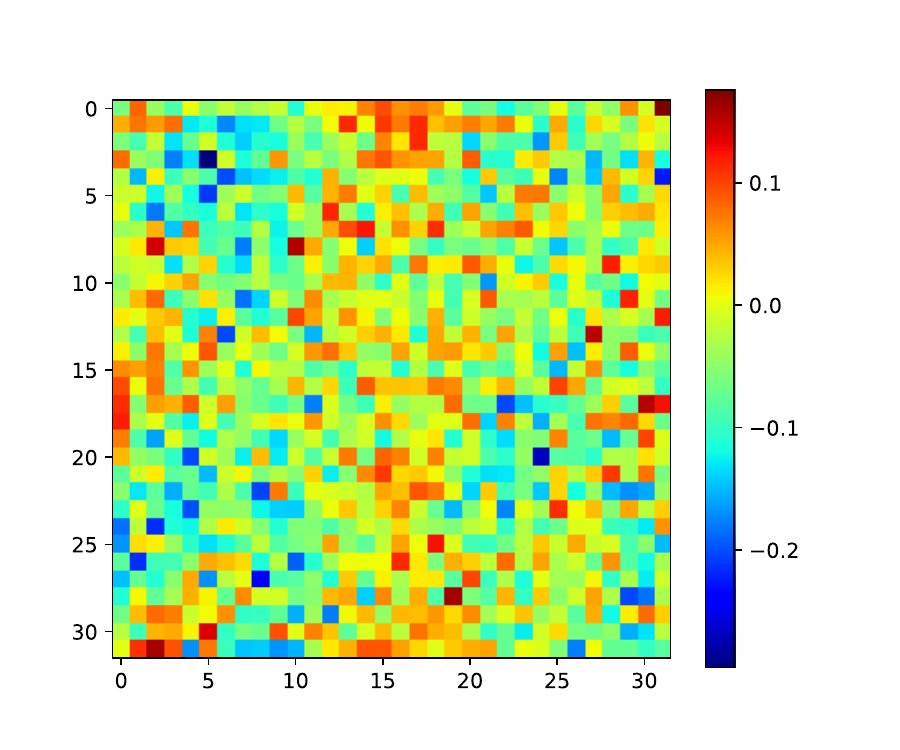}
    \subcaption{}
  \end{subfigure}%
  \caption{Mat\'ern cases (Generation with guidance). (a) Ground Truth. (b) Low resolution with 25\% observation. 
  Mean, standard deviation, and reconstruction MSE loss of 200 samples generated with (c, d, e) \themodel (f, g, h) gaussian noise.}
  \label{fig:matern_big_figure}
\end{figure}

\begin{figure}[!ht]
  \centering
  \begin{subfigure}[t]{0.49\linewidth}
    \centering
    \includegraphics[width=\linewidth]{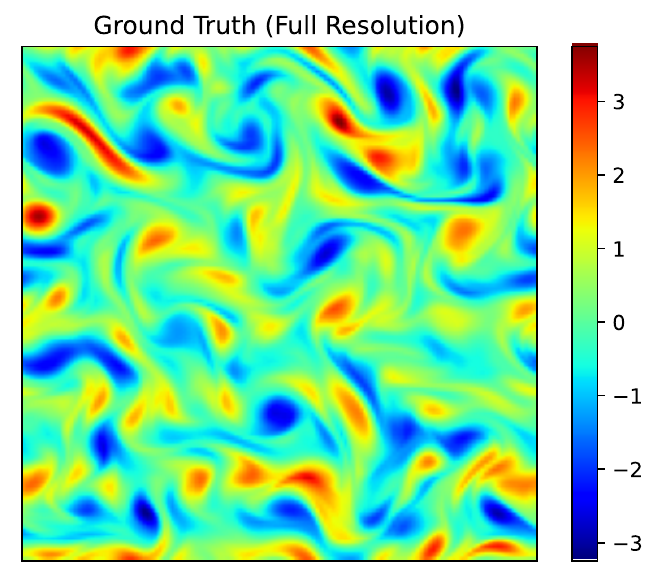}
    \subcaption{}
  \end{subfigure}%
  \begin{subfigure}[t]{0.49\linewidth}
    \centering
    \includegraphics[width=\linewidth]{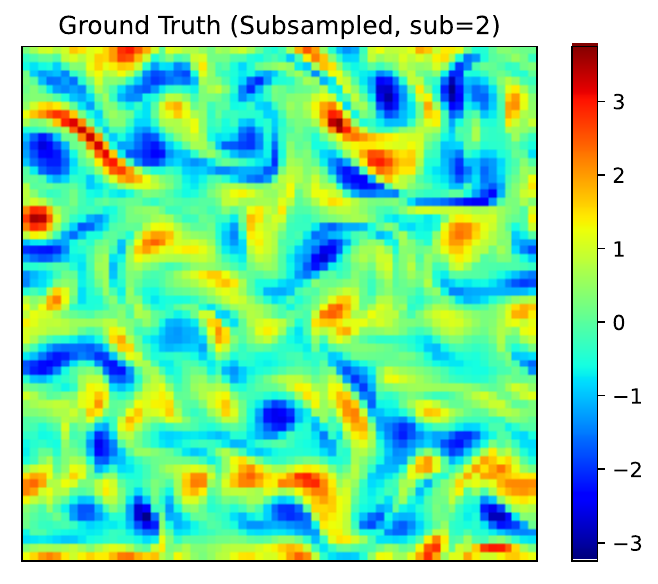}
    \subcaption{}
  \end{subfigure}
    \hfill
  \begin{subfigure}[t]{\linewidth}
    \centering
    \includegraphics[width=\linewidth]{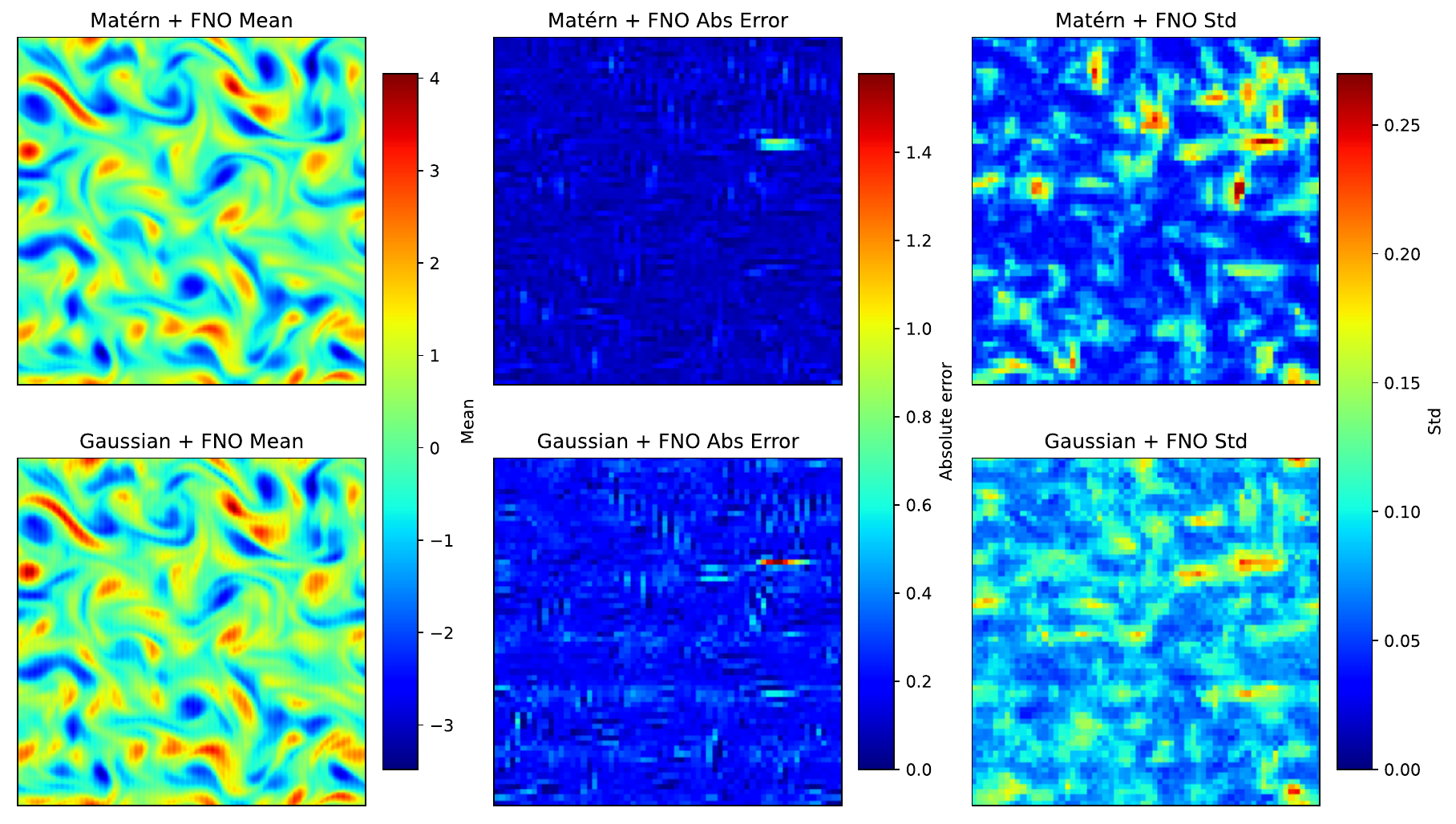}
    \subcaption{}
  \end{subfigure}%
  \caption{SNS cases (Generation with guidance). (a) Ground Truth. (b) Low resolution with 25\% observation. 
  Mean, reconstruction absolute error, and standard deviation of 200 samples generated with (First row) \themodel (Second row) gaussian noise.}
  \label{fig:sns_big_figure}
\end{figure}


\subsection{Ablation Experiments}

\label{apdx:source_selection_ablation}
The source-selection rule involves both Mat\'ern smoothness \(\nu_0\) and correlation length \(\ell_0\). We vary one parameter at a time on the target $(\nu_1,\ell_1)=(2,0.1)$ while keeping the data, FNO, training budget, and evaluation protocol fixed. The smoothness sweep includes inadmissible values to expose the boundary: the best tested choice is \(\nu_0=0.5\), inside the strict condition \(\nu_0<1\). Thus, the result supports choosing the smoothest effective source within a safety-margined admissible set, rather than matching target smoothness exactly.

\begin{table}[H]
\centering
    \caption{Source-smoothness sweep on the Mat\'ern-field target (lower is better).}
\begin{tabular}{lcccccc}
\toprule
       $\nu_0$  & 0.5  & 1.0 & 1.5 & 2.0 & 3.0 & 6.0 \\ \midrule
lMSE  & \textbf{1.90e-3} & 2.91e-3 & 2.92e-3 & 8.94e-3 & 2.28e-2 & 6.78  \\ \bottomrule
\end{tabular}
    \label{tab:smoothness_param}
\end{table}

With \(\nu_0=0.5\) fixed, the length-scale sweep tests the finite-band part of Eq.~\eqref{eq:finite_band_selection}. The equal-$\kappa$ choice \(\ell_0=0.05\) performs best; both shorter and longer correlations increase lMSE, while the discretized Gaussian source is worst. This non-monotone result shows that admissible smoothness alone does not determine a useful source.

\begin{table}[H]
\centering
    \caption{Source length-scale sweep at \(\nu_0=0.5\), with a discretized Gaussian reference.}
\begin{tabular}{lccccc}
\toprule
Source & $\ell_0=0.01$ & $\ell_0=0.05$ & $\ell_0=0.10$ & $\ell_0=0.20$ & Gaussian \\ \midrule
lMSE & 2.995e-3 & \textbf{5.520e-4} & 1.134e-3 & 2.241e-3 & 5.888e-3 \\ \bottomrule
\end{tabular}
    \label{tab:correlation_length}
\end{table}

\subsection{Controlled RK4 Step Sweep}\label{apdx:ode_step_sweep}
We compare fixed checkpoints trained on the same \(64^2\) Mat\'ern target with identical FNO architecture and training protocol. Each entry uses 1{,}000 generated samples. The selected Mat\'ern source reaches lMSE \(1.34\times10^{-3}\) at three steps and \(6.21\times10^{-4}\) at five steps. Gaussian FFM reaches a stable regime only around 30 steps, where its lMSE is \(2.14\times10^{-3}\). Hence the defensible source-specific efficiency gain is a \(6\)--\(10\times\) reduction in RK4 steps and neural-function evaluations, not the larger fixed-budget ratios in individual table cells.

\begin{table}[H]
\centering
\caption{Controlled RK4 sweep on the \(64^2\) Mat\'ern target (lMSE, lower is better).}
\label{tab:controlled_step_sweep}
\begin{tabular}{lrrrrrr}
\toprule
Steps & 3 & 5 & 10 & 20 & 30 & 50 \\
\midrule
Gaussian & 0.2673 & 0.09294 & 0.006363 & 0.005474 & 0.002138 & 0.002232 \\
Mat\'ern & \textbf{0.001339} & \textbf{0.000621} & \textbf{0.000688} & \textbf{0.000850} & \textbf{0.001247} & \textbf{0.001112} \\
\bottomrule
\end{tabular}
\end{table}

\subsection{Complementary SNS Distribution Metrics}\label{apdx:sns_distribution_metrics}
Spectral lMSE alone does not establish distributional fidelity. We therefore evaluate 1{,}000 target and 1{,}000 generated SNS vorticity fields at \(64^2\), using five-step NAGO and ten-step Gaussian FFM. Kernel bandwidths, histogram bins, and normalization are fixed from held-out target samples. NAGO is better on multiscale field MMD, the pointwise marginal distances, velocity structure functions, and three physical observables, whereas Gaussian FFM is better on spatial correlations and vorticity kurtosis. These mixed results support a broad but non-universal distributional improvement.

\begin{table}[H]
\centering
\caption{SNS distributional errors (lower is better).}
\label{tab:sns_distribution_metrics}
\resizebox{\textwidth}{!}{%
\begin{tabular}{lrrrrrrrrr}
\toprule
Method & MMD$^2$ & std. W1 & JS & $S_2$ & $S_4$ & KE & Enstrophy & Palinstrophy & Vort. corr. \\
\midrule
Gaussian (10) & 0.003676 & 0.07487 & 0.002247 & 0.1895 & 0.3506 & 0.1885 & 0.1786 & 0.2262 & \textbf{0.01711} \\
NAGO (5) & \textbf{0.001418} & \textbf{0.04575} & \textbf{0.000798} & \textbf{0.1450} & \textbf{0.2768} & \textbf{0.1330} & \textbf{0.1095} & \textbf{0.1470} & 0.01977 \\
\bottomrule
\end{tabular}}
\end{table}

\subsection{Resolution Diagnostics}\label{apdx:resolution_diagnostics}
Models are trained at \(64^2\) and evaluated at \(32^2\), \(64^2\), and \(128^2\). At \(128^2\), NAGO assigns \(93.99\%\) of generated energy to the band shared with training and \(2.39\%\) to newly available modes; Gaussian FFM assigns \(23.64\%\) and \(76.92\%\), respectively. This measures shared-band preservation and suppression of spurious new-band energy. It does not imply correct extrapolation of unobserved fine scales: the NAGO palinstrophy ratio to the reference remains 6.72.

\subsection{Sampling Throughput Protocol}\label{apdx:speed_protocol}
We time both samplers on one NVIDIA RTX~4090 at batch size 16 with explicit CUDA synchronization. The pseudo-spectral solver runs at \(256^2\), uses \(\Delta t=10^{-4}\), records every 0.5 physical-time units, and is downsampled to \(64^2\). During the stationary production interval it takes 2{,}038.139\,s for 100 recorded batches, or 20.3814\,s per batch. Five-step NAGO takes a median 0.228848\,s per batch over ten repeats (IQR 0.002085\,s), giving the reported raw ratio \(89.1\times\). Checkpoint loading, serialization, dataset generation, and model training are excluded. Because solver outputs are temporally correlated and five-step NAGO does not match the finite-reference sampling floor, this is a post-training framework-throughput comparison, not an equal-fidelity or end-to-end claim.

\subsection{FFT Mat\'ern Sampling Cost}\label{apdx:matern_sampling_cost}
For the stationary periodic-grid setting used here, the Mat\'ern covariance is diagonal in the Fourier basis. For a batch of \(B\) fields on an \(N^d\) grid, we draw white Fourier coefficients, multiply each coefficient by \((\kappa^2+\lVert k\rVert^2)^{-\alpha/2}\), impose Hermitian symmetry for real fields, and apply an inverse FFT. This requires \(O(BN^d\log N)\) time and \(O(BN^d)\) memory. No dense \(N^d\times N^d\) covariance matrix is constructed or factorized. This advantage does not directly extend to irregular domains where Fourier diagonalization is unavailable.

\subsection{Stationary Anisotropic Target}\label{apdx:anisotropic}
We place the anisotropic experiment in the appendix because the current method selects an isotropic source. The target is a stationary anisotropic Mat\'ern field with \(\nu_1=2\), \(\ell_x=0.1\), \(\ell_y=0.05\), and a \(30^\circ\) rotation. Both \(64^2\) FNOs retain the full Fourier grid and are trained from scratch for 4{,}000 steps. At ten RK4 steps, the isotropic Mat\'ern source improves orientation error, anisotropy-ratio error, and full 2D spectral errors over the Gaussian source (Table~\ref{tab:anisotropic}). Both sources nevertheless recover the anisotropy accurately, and paired fixed-checkpoint evaluation seeds do not show uniform superiority: at three-step Mat\'ern versus ten-step Gaussian, 2D log-power MSE is \(0.003001\pm0.000121\) versus \(0.002788\pm0.000176\). The intervals quantify evaluation-sampling variability, not independent training runs.

\begin{table}[H]
\centering
\caption{Stationary anisotropic Mat\'ern target with 1{,}024 samples and ten RK4 steps.}
\label{tab:anisotropic}
\begin{tabular}{lrrrr}
\toprule
Source & Angle error ($^\circ$) & Ratio error & 2D power L1 & Log-power MSE \\
\midrule
Gaussian & 0.1310 & 0.01192 & 0.04611 & 0.002929 \\
Isotropic Mat\'ern & \textbf{0.0815} & \textbf{0.01170} & \textbf{0.03955} & \textbf{0.002563} \\
\bottomrule
\end{tabular}
\end{table}

Figure~\ref{fig:anisotropic_samples} complements the aggregate spectral metrics with unpaired field realizations. The samples are drawn from the same saved 1{,}024-sample, ten-step evaluation used in Table~\ref{tab:anisotropic}; fixed, nonoverlapping array indices are used for the three rows. Both learned distributions reproduce the target's diagonal organization, while the quantitative differences between source choices are assessed by the distribution-level metrics above rather than by visual matching of individual draws.

\begin{figure}[H]
  \centering
  \includegraphics[width=0.94\textwidth]{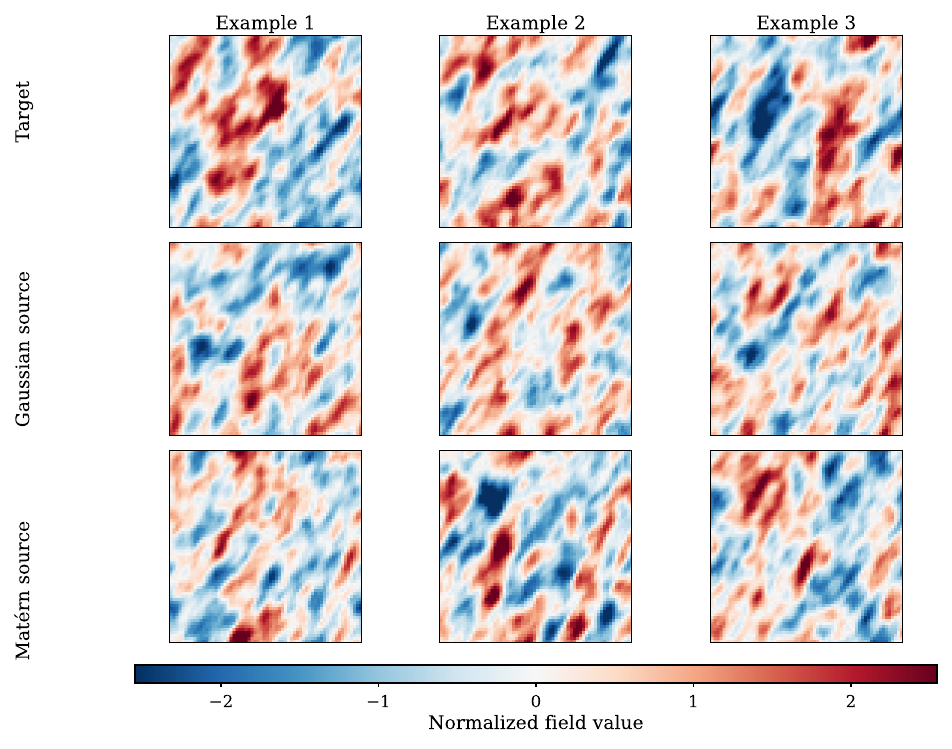}
  \caption{Qualitative samples for the stationary anisotropic Mat\'ern target. Rows show unpaired draws from the target distribution, the Gaussian-source model, and the isotropic Mat\'ern-source model. All panels use the same symmetric color scale.}
  \label{fig:anisotropic_samples}
\end{figure}

\begin{figure}[!ht]
  \centering
  \begin{subfigure}[t]{0.48\textwidth}
    \centering
    \includegraphics[width=\textwidth]{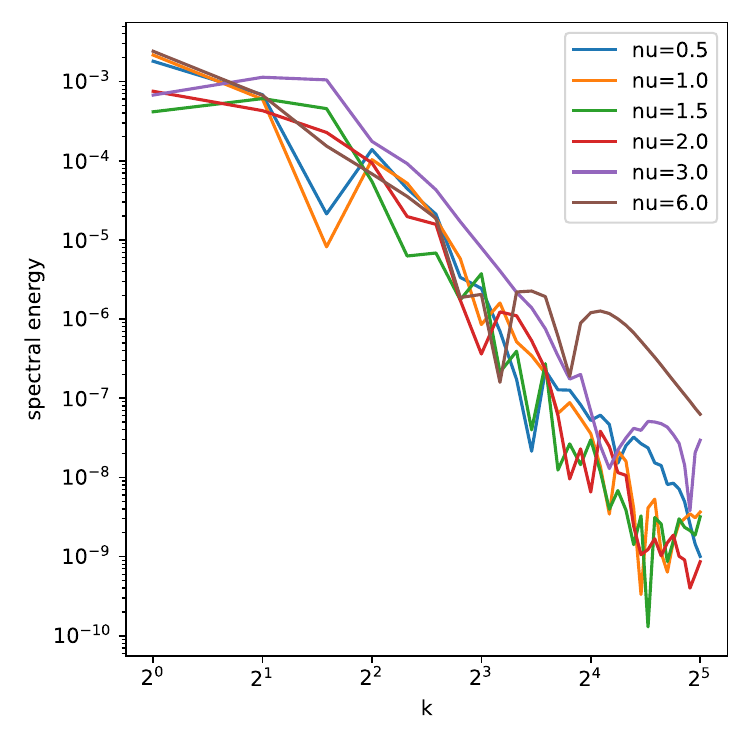}
    \caption{Guided loss trained by different $\nu$.}
    \label{fig:loss_nu}
  \end{subfigure}
  \begin{subfigure}[t]{0.48\textwidth}
    \centering
    \includegraphics[width=\textwidth]{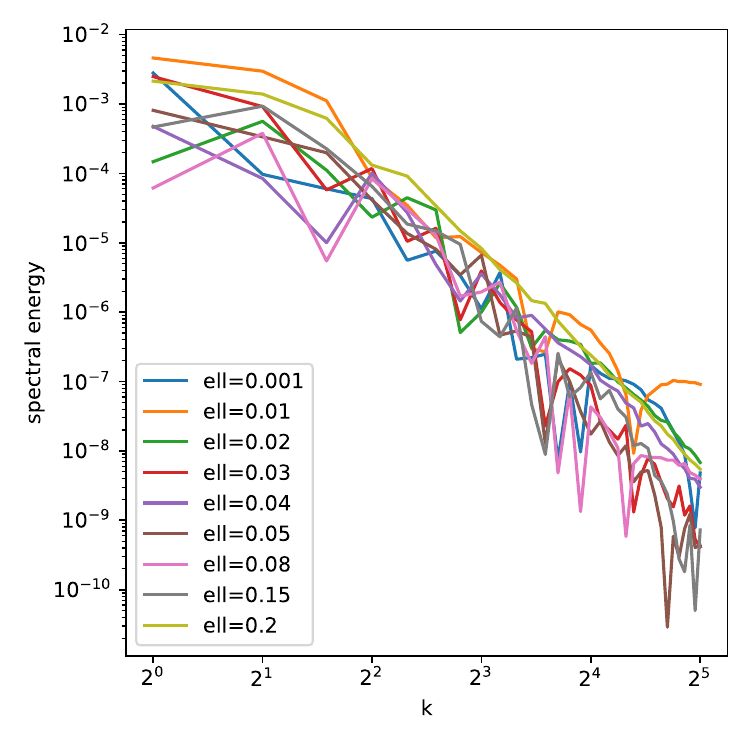}
    \caption{Guided loss trained by different $\ell$}
    \label{fig:loss_ell}
  \end{subfigure}
  \caption{Comparison of guided losses under different noise settings.}
  \label{fig:loss_matern}
\end{figure}

\begin{figure}[!ht]
  \centering
  \begin{subfigure}[t]{0.44\textwidth}
    \centering
    \includegraphics[width=\textwidth]{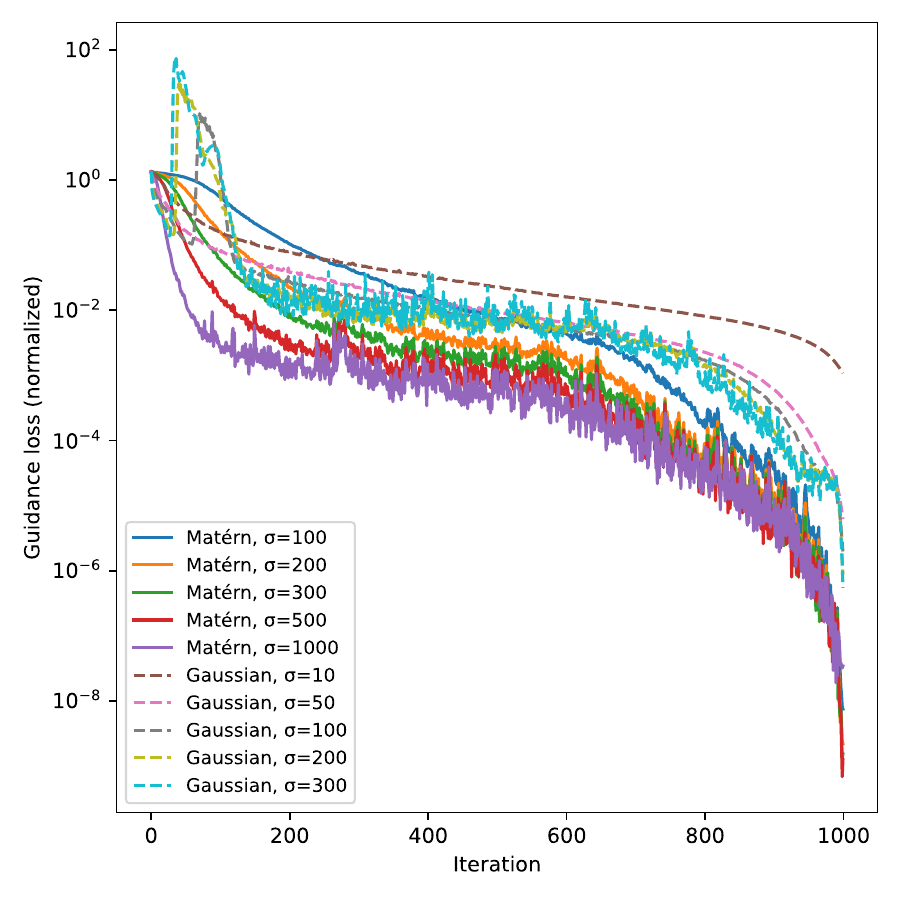}
    \caption{Guided Loss as Iteration Steps. (Matern)}
    \label{fig:matern_loss_guidance}
  \end{subfigure}
  \begin{subfigure}[t]{0.48\textwidth}
    \centering
    \includegraphics[width=\textwidth]{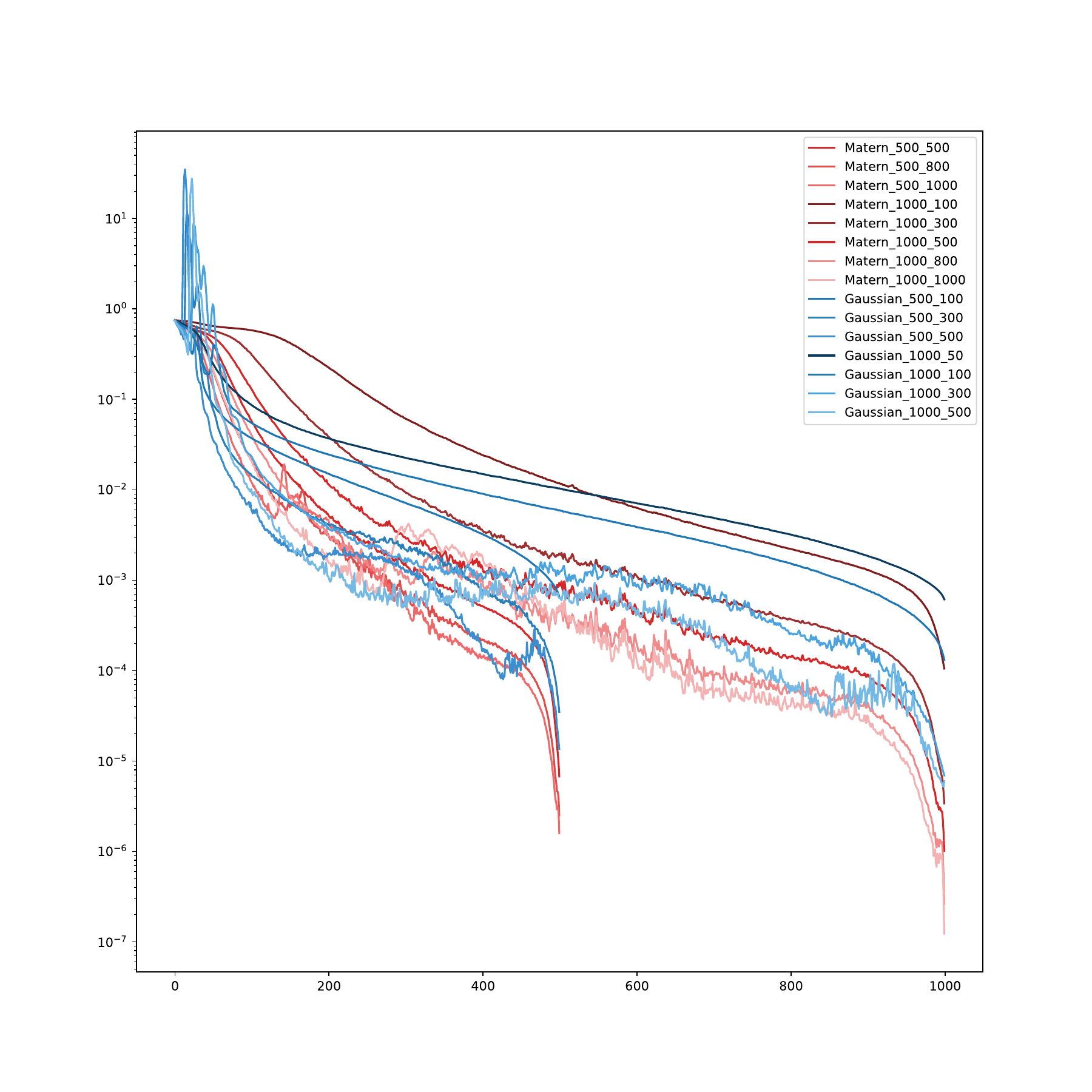}
    \caption{Guided Loss as Iteration Steps. (SNS)}
    \label{fig:sns_guide_loss_b}
  \end{subfigure}
  \caption{Comparison of guided losses under different noise settings.}
  \label{fig:guide_loss}
\end{figure}
\label{apdx:guide_loss}
Guided sampling introduces a second numerical intervention in addition to ODE refinement. If a per-step guidance coefficient is held fixed while the number of steps grows, the cumulative guidance also grows, confounding source effects with intervention strength. We therefore add a fixed-total protocol in which the per-step coefficient is scaled so that the total guidance budget remains constant.

On the \(64^2\) Mat\'ern target with 25\% observations, 1{,}000 generated samples, and total guidance budget 100, Gaussian sources yield lower observation and conditional-mean MSE at both 3 and 100 RK4 steps. The Mat\'ern source is substantially better in spectral lMSE at three steps but not at 100 steps (Table~\ref{tab:fixed_total_guidance}). Thus, source choice changes the trade-off among observation fit, posterior spread, and spectral fidelity; no single source uniformly dominates the guided metrics. The loss trajectories in Figs.~\ref{fig:loss_matern} and \ref{fig:guide_loss} remain useful diagnostics under their original parameterization, but they should not be read as causal isolation of the source.

\begin{table}[H]
\centering
\caption{Fixed-total-guidance control on the Mat\'ern target (lower is better).}
\label{tab:fixed_total_guidance}
\begin{tabular}{llrrr}
\toprule
Steps & Source & Observed MSE & Mean MSE & Spectral lMSE \\
\midrule
3 & Gaussian & \textbf{1.136} & \textbf{0.510} & 0.2665 \\
3 & Mat\'ern & 1.817 & 0.907 & \textbf{0.00259} \\
100 & Gaussian & \textbf{0.977} & \textbf{0.450} & \textbf{0.02620} \\
100 & Mat\'ern & 1.741 & 0.893 & 0.07456 \\
\bottomrule
\end{tabular}
\end{table}

\section{Limitations of This Study}
\label{app:limitations}
The current theory and implementation assume stationary measures on periodic regular grids. Under these assumptions, the spectrum is diagonal in a Fourier basis and Mat\'ern samples are generated efficiently by FFT. On irregular geometries, nonperiodic domains, or nonstationary targets, estimating a reliable spectrum and applying the same sampler are nontrivial. The finite-band selection rule also depends on the sample size and on the band used to estimate \(\widehat\beta\); uncertainty in this estimate is not yet propagated into source selection.

The empirical evidence is also bounded. Most comparisons use one trained checkpoint per configuration, so repeated evaluation seeds measure conditional sampling variability rather than training variability. The anisotropic experiment uses an isotropic source and shows accurate recovery of a stationary directional target, but it does not establish uniform superiority or address spatially varying anisotropy. Cross-resolution tests preserve the represented frequency band but do not reconstruct an unseen spectral tail. Finally, the guidance results depend on the metric and on how guidance strength is normalized across solver steps.

We also emphasize that the scope of this work is sampling from invariant measures of SPDEs, rather than general PDE solution generation or image generation. 
For deterministic PDE prediction, generative models are usually less efficient than neural operator methods when the goal is to approximate the forward solution map, since iterative sampling is not naturally optimized for fast surrogate modeling. 
Therefore, we do not position the proposed method as a replacement for operator-learning approaches in standard PDE forecasting tasks.

Similarly, this work does not focus on image-generation benchmarks. 
Our resolution-consistency analysis is motivated by the function-space nature of SPDEs, where samples are discretizations of underlying random fields and should remain consistent under changes of spatial resolution. 
This consideration is fundamentally different from common image-generation settings, where low- and high-resolution images are usually treated as distinct finite-dimensional data distributions rather than discretizations of the same function-space object.

%% file: checklist.tex
\newpage
\section*{NeurIPS Paper Checklist}

\begin{enumerate}

\item {\bf Claims}
    \item[] Question: Do the main claims made in the abstract and introduction accurately reflect the paper's contributions and scope?
    \item[] Answer:  \answerYes{} 
    \item[] Justification: We have made clear claims about contributions and scopes in the abstract and introduction.

\item {\bf Limitations}
    \item[] Question: Does the paper discuss the limitations of the work performed by the authors?
    \item[] Answer:  \answerYes{} 
    \item[] Justification: We have the separated limitation section in the appendix.

\item {\bf Theory assumptions and proofs}
    \item[] Question: For each theoretical result, does the paper provide the full set of assumptions and a complete (and correct) proof?
    \item[] Answer: \answerYes{} 
    \item[] Justification: We provide detailed theoretical proofs in the main paper and appendix.

    \item {\bf Experimental result reproducibility}
    \item[] Question: Does the paper fully disclose all the information needed to reproduce the main experimental results of the paper to the extent that it affects the main claims and/or conclusions of the paper (regardless of whether the code and data are provided or not)?
    \item[] Answer: \answerYes{} 
    \item[] Justification: We describe the details of our model and training strategy in the paper for reproduction.

\item {\bf Open access to data and code}
    \item[] Question: Does the paper provide open access to the data and code, with sufficient instructions to faithfully reproduce the main experimental results, as described in supplemental material?
    \item[] Answer: \answerYes{} 
    \item[] Justification: The code is provided in supplemental material.

\item {\bf Experimental setting/details}
    \item[] Question: Does the paper specify all the training and test details (e.g., data splits, hyperparameters, how they were chosen, type of optimizer) necessary to understand the results?
    \item[] Answer: \answerYes{} 
    \item[] Justification: We have introduced experiment settings and details, and even add more details in the appendix.
    \item[] Guidelines: We have introduced experiment settings and details, and even add more details in the appendix.
    \begin{itemize}
        \item The answer \answerNA{} means that the paper does not include experiments.
        \item The experimental setting should be presented in the core of the paper to a level of detail that is necessary to appreciate the results and make sense of them.
        \item The full details can be provided either with the code, in appendix, or as supplemental material.
    \end{itemize}

\item {\bf Experiment statistical significance}
    \item[] Question: Does the paper report error bars suitably and correctly defined or other appropriate information about the statistical significance of the experiments?
    \item[] Answer: \answerYes{} 
    \item[] Justification: Our paper is about SPDE, so the error is all based on the mean of samples, and we also consider their spectral distribution.

\item {\bf Experiments compute resources}
    \item[] Question: For each experiment, does the paper provide sufficient information on the computer resources (type of compute workers, memory, time of execution) needed to reproduce the experiments?
    \item[] Answer: \answerYes{} 
    \item[] Justification: We have discussed the information about computer resources.
    
\item {\bf Code of ethics}
    \item[] Question: Does the research conducted in the paper conform, in every respect, with the NeurIPS Code of Ethics \url{https://neurips.cc/public/EthicsGuidelines}?
    \item[] Answer: \answerYes{} 
    \item[] Justification:  We follow the NeurIPS Code of Ethics in every respect.

\item {\bf Broader impacts}
    \item[] Question: Does the paper discuss both potential positive societal impacts and negative societal impacts of the work performed?
    \item[] Answer: \answerYes{} 
    \item[] Justification: Positive impacts include
reducing computational costs for sampling invariant measures of SPDEs and improving the stability of scientific generative modeling. We do not anticipate specific negative societal impacts from this theoretical and computational work.
The method is designed for scientific simulations of stochastic PDEs rather than human-facing decision making or deployment in sensitive domains.
    \begin{itemize}
        \item The answer \answerNA{} means that there is no societal impact of the work performed.
        \item If the authors answer \answerNA{} or \answerNo, they should explain why their work has no societal impact or why the paper does not address societal impact.
        \item Examples of negative societal impacts include potential malicious or unintended uses (e.g., disinformation, generating fake profiles, surveillance), fairness considerations (e.g., deployment of technologies that could make decisions that unfairly impact specific groups), privacy considerations, and security considerations.
        \item The conference expects that many papers will be foundational research and not tied to particular applications, let alone deployments. However, if there is a direct path to any negative applications, the authors should point it out. For example, it is legitimate to point out that an improvement in the quality of generative models could be used to generate Deepfakes for disinformation. On the other hand, it is not needed to point out that a generic algorithm for optimizing neural networks could enable people to train models that generate Deepfakes faster.
        \item The authors should consider possible harms that could arise when the technology is being used as intended and functioning correctly, harms that could arise when the technology is being used as intended but gives incorrect results, and harms following from (intentional or unintentional) misuse of the technology.
        \item If there are negative societal impacts, the authors could also discuss possible mitigation strategies (e.g., gated release of models, providing defenses in addition to attacks, mechanisms for monitoring misuse, mechanisms to monitor how a system learns from feedback over time, improving the efficiency and accessibility of ML).
    \end{itemize}
    
\item {\bf Safeguards}
    \item[] Question: Does the paper describe safeguards that have been put in place for responsible release of data or models that have a high risk for misuse (e.g., pre-trained language models, image generators, or scraped datasets)?
    \item[] Answer: \answerNA{} 
    \item[] Justification: We do not release any data or model that has high risks for misuse.

\item {\bf Licenses for existing assets}
    \item[] Question: Are the creators or original owners of assets (e.g., code, data, models), used in the paper, properly credited and are the license and terms of use explicitly mentioned and properly respected?
    \item[] Answer: \answerYes{} 
    \item[] Justification: We have cited the original paper that produced the code package and dataset.

\item {\bf New assets}
    \item[] Question: Are new assets introduced in the paper well documented and is the documentation provided alongside the assets?
    \item[] Answer: \answerYes{} 
    \item[] Justification: We provide documentation with our code.

\item {\bf Crowdsourcing and research with human subjects}
    \item[] Question: For crowdsourcing experiments and research with human subjects, does the paper include the full text of instructions given to participants and screenshots, if applicable, as well as details about compensation (if any)? 
    \item[] Answer: \answerNA{} 
    \item[] Justification: Our work does not involve crowdsourcing nor research with human subjects.

\item {\bf Institutional review board (IRB) approvals or equivalent for research with human subjects}
    \item[] Question: Does the paper describe potential risks incurred by study participants, whether such risks were disclosed to the subjects, and whether Institutional Review Board (IRB) approvals (or an equivalent approval/review based on the requirements of your country or institution) were obtained?
    \item[] Answer: \answerNA{} 
    \item[] Justification: Our work does not involve crowdsourcing nor research with human subjects.

\item {\bf Declaration of LLM usage}
    \item[] Question: Does the paper describe the usage of LLMs if it is an important, original, or non-standard component of the core methods in this research? Note that if the LLM is used only for writing, editing, or formatting purposes and does \emph{not} impact the core methodology, scientific rigor, or originality of the research, declaration is not required.
    \item[] Answer: \answerNA{} 
    \item[] Justification: The core method development in this research does not involve LLMs as any important, original, or non-standard components.

\end{enumerate}